\documentclass[USletter,10pt,english]{article}
\usepackage[margin=1in]{geometry}
\usepackage{graphicx}      
\usepackage[font=normalsize]{caption}
\usepackage{amsfonts}
\usepackage{amsmath}
\usepackage{amssymb}
\usepackage{bm}
\usepackage{algorithm}
\usepackage{algpseudocode}
\usepackage{lipsum}
\usepackage{color}
\usepackage[breaklinks]{hyperref}
\usepackage{mathrsfs}
\usepackage{psfrag,epsf}
\usepackage{amsthm}
\usepackage{url}

\usepackage{breakurl}
\usepackage{float}
\usepackage{natbib}

\newif\ifarxiv
\arxivtrue
\arxivfalse

\usepackage{amsbsy}
\usepackage{mathrsfs}
\usepackage{mathtools}
\usepackage{setspace,multirow, url,  subfig, bm, titlesec, textcomp, gensymb}

\newtheorem{thm}{Theorem}
\newtheorem{lem}{Lemma}
\newtheorem{defn}{Definition}

\newtheorem{assum}{Assumption}
\newtheorem{cor}{Corollary}

\newtheorem{rem}{Remark}

\setcitestyle{authoryear,open={(},close={)}}

\def \exp {\mathrm{exp}}

\newcommand{\Mnorm}[2]{{\left\vert\kern-0.30ex\left\vert\kern-0.30ex\left\vert #1 
		\right\vert\kern-0.30ex\right\vert\kern-0.30ex\right\vert}}
\newcommand{\Opnorm}[3]{{\left\vert\kern-0.25ex\left\vert\kern-0.25ex\left\vert #1 
		\right\vert\kern-0.25ex\right\vert\kern-0.25ex\right\vert}_{#2 \to #3}}
\newcommand{\norm}[2]{{\left\vert\kern-0.30ex\left\vert #1 
		\right\vert\kern-0.30ex\right\vert}}

\newcommand{\innerproductminconstant}[1]{\psi_0}

\newif\ifarxiv
\arxivtrue
\arxivfalse

\begin{document}
\ifarxiv
\doublespacing
\onecolumn
\else \fi
\sloppy

\title{\bf Thompson Sampling in Partially Observable Contextual Bandits} 
\author{Hongju Park$^1$ and Mohamad Kazem Shirani Faradonbeh$^2$}%
\date{    $^1$Department of Statistics, University of Georgia\\%
    $^2$Department of Mathematics, Southern Methodist University\\[2ex]}
\maketitle
 

\begin{abstract}
		Contextual bandits constitute a classical framework for decision-making under uncertainty. In this setting, the goal is to learn the arms of highest reward subject to contextual information, while the unknown reward parameters of each arm need to be learned by experimenting that specific arm. Accordingly, a fundamental problem is that of balancing exploration (i.e., pulling different arms to learn their parameters), versus exploitation (i.e., pulling the best arms to gain reward). To study this problem, the existing literature mostly considers perfectly observed contexts. However, the setting of partial context observations remains unexplored to date, despite being theoretically more general and practically more versatile. We study bandit policies for learning to select optimal arms based on the data of observations, which are noisy linear functions of the unobserved context vectors. Our theoretical analysis shows that the Thompson sampling policy successfully balances exploration and exploitation. Specifically, we establish the followings: (i) regret bounds that grow poly-logarithmically with time, (ii) square-root consistency of parameter estimation, and (iii) scaling of the regret with other quantities including dimensions and number of arms. Extensive numerical experiments with both real and synthetic data are presented as well, corroborating the efficacy of Thompson sampling. To establish the results, we introduce novel martingale techniques and concentration inequalities to address partially observed dependent random variables generated from unspecified distributions, and also leverage problem-dependent information to sharpen probabilistic bounds for time-varying suboptimality gaps. These techniques pave the road towards studying other decision-making problems with contextual information as well as partial observations.
	\end{abstract}


\section{Introduction}
\label{sec:1}
Contextual bandits have emerged in the recent literature as widely-used decision-making models involving time-varying information. In this setup, a policy takes action after (perfectly or partially) observing the context(s) at each time. The data collected thus far is utilized, aiming to maximize cumulative rewards determined by both the context(s) and unknown parameters. So, any desirable policy needs to manage the delicate trade-off between learning the best (i.e., exploration) and earning the most (i.e., exploitation). For this purpose, Thompson sampling stands-out among various competitive algorithms, thanks to its strong performance as well as computationally favorable implementations. Its main idea is to explore based on samples from a data-driven posterior belief about the unknown parameters. However, comprehensive studies are currently missing for imperfectly observed contexts, and it is adopted as the focus of this work.
	
	Letting the time-varying components of the decision options (e.g., contexts) be observed only partially, is known to be advantageous. More specifically, in various real-world problems including robot control and image processing~\citep{lin2012kalman,nagrath2006control,dougherty2020digital,kang2012lidar,aastrom1965optimal,kaelbling1998planning}, partial, transformed, or noisy signal-observation models have been used traditionally to obtain better performance. On the other hand, overlooking imperfectness of observations can lead to compromised decisions. For example, if disregarding uncertainty in medical profiles of septic patients, clinical decisions end-up with worse consequences~\citep{gottesman2019guidelines}. Accordingly, partial observation models are studied in canonical settings such as linear systems \citep{kargin2023thompson}, bandit monitoring \citep{kirschner2020information,lattimore2022minimax,tsuchiya2023best}, and Markov decision processes \citep{bensoussan2004stochastic,krishnamurthy2009partially}. The above have recently motivated some work on contextual bandit policies with partially observed contexts~\citep{park2021analysis,park2022efficient,park2022worst}. However, a reliable policy that can provably balance exploration and exploitation is not currently available, as will be elaborated shortly, after clarifying the technical setting and reviewing the literature. 
	
	The common bandit setting is the so-called linear one, where the \emph{expected} reward of each arm is the inner product of (adversarial or stochastic) context(s) and reward parameter(s). The latter in stochastic contextual bandits can be either \textit{arm-specific}~\citep{goldenshluger2013linear,bastani2020online}, or \textit{shared} across all arms~\citep{dani2008stochastic,chakraborty2023thompson}. We analyze both settings, with the focus being on the more general and challenging one of the former. For the sake of completeness, the authors also refer to a (non-exhaustive) variety of extant approaches in the realm of contextual bandits. That includes (possibly infinite but bounded) action sets in a Euclidean space \citep{abbasi2011improved,abeille2017linear}, as well as those with adversarial contexts \citep{dani2008stochastic,agarwal2014taming}, together with non-linear or non-parametric reward functions \citep{dumitrascu2018pg,guan2018nonparametric,wanigasekara2019nonparametric}. Notably, all of these references assume fully observed contexts, in contradistinction to this work.

	
	The discourse of efficient policies for contextual bandits has come a long way. Algorithms based on Optimism in the Face of Uncertainty (OFU) held prominent positions thanks to intuitively balanced exploration and exploitation, together with theoretical performance guarantees \citep{auer2002using,dani2008stochastic,abbasi2011improved}. Afterward, Thompson sampling has been recognized as the pioneer, first via excelling empirical performance~\citep{chapelle2011empirical}, and then supplemented with theoretical analyses \citep{agrawal2013thompson,russo2014learning,abeille2017linear}. More recently, it has come to light that Greedy policies can be nearly optimal in some contextual bandits, e.g., those with a shared reward parameter \citep{park2022worst} and those with only two arm-specific parameters \citep{raghavan2023greedy}. In contrast, for contextual bandits with multiple arm-specific reward parameters, it is known that vanilla Greedy algorithms are non-optimal \citep{bastani2021mostly}. That is caused, intuitively, by superior arms dominating some others, leaving them unexplored, and is also illustrated in our experiments at the end of this paper.
	
	Accordingly, the study of theoretical performance guarantees for Thompson sampling has gained much popularity and made significant progress in the recent literature with an emphasis on high-probability instance-dependent regret. First, regret bounds growing as square-root of time were shown for adversarial contextual bandits \citep{agrawal2013thompson,russo2014learning,abeille2017linear}, succeeded by a square-root regret bound for settings with a Euclidean action set \citep{hamidi2020worst} and logarithmic regret bound for stochastic contextual bandits with a shared reward parameter~\citep{chakraborty2023thompson}. In particular, in the latter case (that the rewards of different arms share the unknown parameter), the regret of Thompson sampling can still be logarithmic with time, if the observations are noisy versions of the stochastic context vectors and the same dimension \citep{park2021analysis,park2022efficient}. 
    However, for the arm-specific reward setup of stochastic contexts, the efficiency of the above-mentioned bandit policies remains unanswered. Indeed, the analysis is more challenging in such settings as the policy needs to address the trade-off between exploration and exploitation, unlike the setting with a shared reward parameter. As we expressed before, we study this trade-off, taking advantage of problem-dependent information to get tighter regret bounds.

 
	We analyze the Thompson sampling policy in contextual bandits with partially observable stochastic contexts focusing on the high-probability frequentist instance-dependent regret. Our analysis indicates that the error in estimating the reward parameters decays with square-root of time, and the worst-case regret grows at most as fast as the poly-logarithm of time. Next, the effect of the ambient dimension $d$ is of the order of $\sqrt{d}$ on the estimation error accuracy, while they exacerbate the regret bound as $d^4$. Furthermore, the larger number of arms $N$ has a negative effect on estimation accuracy with the rate $\sqrt{N}$ on average by lowering the growth of the number of each arm being chosen, and subsequently incurs the regret increase with a rate at most $N$. Lastly, a smaller value of instance-dependent suboptimality gap $\kappa$ amplifies regret, with a rate no greater than $\kappa^{-5}$.

	
	For regret analysis in partially observed contextual bandits, it is crucial to examine the partial observability of contexts and both factors of (i) the suboptimality gaps (i.e., the lost reward by pulling non-optimal arms), as well as (ii) probabilities of pulling such suboptimal arms. The existing technical approaches fail to provide useful results, especially when it comes to bounding the latter factor, mainly due to the inter-dependencies of the involved variables. This challenge is addressed in this work via developing novel technical tools, as briefly mentioned below. First, we take into account a problem-dependent instance of partially observed stochastic contexts to sharpen the existing regret upper bounds for adversarial contextual bandits. Based on this instance, we construct the minimum time that guarantees linear growths of the number of selections of each arm. Then, we establish that the probability of pulling suboptimal arms decreases fast as time proceeds with rate $t^{-1/2}$. Along the way, we delicately construct stochastic processes with self-normalized or martingale structures, and employ useful stochastic bounds for them, in order to prove our results on the regret bound. This problem techniques lay the groundwork for the exploration of other decision-making problems with contexts or partial observations. 

	The organization is outlined below. In Section~\ref{sec:2}, we formulate the problem and discuss preliminaries. Then, the Thompson sampling policy for partially observable contextual bandits is presented in Section~\ref{sec:3}. We provide its theoretical performance guarantees in Section \ref{sec:4}, followed by real-data experiments in Section~\ref{sec:5}. The paper is concluded then, delegating further technical discussions, intermediate lemmas, and proofs of the theorems, all to appendices. 
	
	Henceforth, for an integer $i$, $[i]$ represents the set of natural numbers up to $i$; $\{1,2,\dots,i\}$. We use $M^\top$ to refer to the transpose of the matrix $M \in \mathbb{C}^{p \times q}$, and $C(M)$ denotes the column space of $M$. For a vector $v \in \mathbb{C}^d$, we denote the $\ell_2$ norm  by $\|v\| = \left(\sum_{i=1}^d |v_i|^2\right)^{1/2}$ and the weighted $\ell_2$ norm with a positive definite matrix $A$ by $\|v\|_A = \sqrt{v^\top A v}$. Finally, $\lambda_{\min}(\cdot)$ and $\lambda_{\max}(\cdot)$ are the minimum and maximum eigenvalues. 

\section{Problem Formulation} \label{sec:2}
	
	In this section, we express the partially observable contextual bandit problem. A decision-maker aims to maximize their cumulative reward by selecting from $N$ arms, the reward of arm $i \in [N]$ at time $t$ being 
	\begin{eqnarray}
		r_i(t) = x_i(t)^\top \mu_i + \varepsilon_i(t).\label{eq:reward}
	\end{eqnarray}
	Above, $x_i(t)$ is the \textit{unobserved} $d_x$-dimensional stochastic context of arm $i$, independently generated over time and across arms, with $\mathbb{E} \left[ x_i(t) \right]=\mathbf{0}_{d_x}$ and unknown covariance $\mathrm{Cov}(x_i(t))=\Sigma_x$. Further, $\mu_i \in \mathbb{R}^{d_x}$ is the \emph{unknown arm-specific} reward parameter of the $i$-th arm, and $\varepsilon_i(t)$ is the noise in realization of the reward value. We assume that each element of $x_i(t)$ has a sub-Gaussian tail, as defined below, and the reward noise is sub-Gaussian as well. That is, there exists a fixed constant $R_1>0$, that for all real $\lambda$, we have
	\begin{eqnarray}
		\mathbb{E}\left[e^{\lambda \varepsilon_i(t)}\right] \leq \exp \left( \frac{\lambda^2 R^2_1}{2} \right) .\label{eq:r1}
	\end{eqnarray}
	
	The policy observes $y_i(t)$ for $i\in [N]$, which is the following transformed noisy function of the context:
	\begin{eqnarray}
		y_i(t) = A x_i(t) + \xi_i(t),\label{eq:obmodel}
	\end{eqnarray}
	where $A$ is the unknown ${d_y\times d_x}$ sensing matrix, and $\xi_i(t)$ is the sensing (or measurement) noise, its covariance matrix being denoted by $\Sigma_\xi$ and is unknown. We assume that each element of $\xi_i(t)$ is sub-Gaussian as well, as rigorously defined above. At each time $t$, the decision-maker chooses an arm, denoted by $a(t)$, given the history of actions $\{a(\tau)\}_{\tau \in [t-1]}$, rewards $\{r_{a(\tau)}(\tau)\}_{\tau \in [t-1]}$, and past observations $\{y_i(\tau)\}_{\tau \in [t-1], i\in[N]}$, as well as the current ones $\{y_i(t)\}_{i\in[N]}$. Once choosing the arm $a(t)$, the decision-maker gets a reward $r_{a(t)}(t)$ according to \eqref{eq:reward}, whereas rewards of other arms are \emph{not} realized. 
 
 The true reward parameters $\mu_i$ for $i \in [N]$, are unknown to the bandit policy. In addition, the context vectors are not fully observed and $x_i(t)$ is not available either. Thus, we investigate estimation of $x_i(t)^\top \mu_i$ by utilizing the information $y_i(t)$ can provide about $x_i(t)$ according to the observation model in \eqref{eq:obmodel}. First, since the above context estimation needs to be repeated over time for many rounds, it is essential for learning high-reward arms to estimate $x_i(t)$ unbiasedly and with a small variance. Technically, consider predicting $x_i(t)^\top \mu$, for an arbitrary vector $\mu \in \mathbb{R}^{d_x}$, based on $y_i(t)$. The observation model in \eqref{eq:obmodel} dictates to focus on finding $b \in \mathbb{R}^{d_y}$ to minimize the variance of prediction error $\textrm{Var}( x_i(t)^\top \mu-y_i(t)^\top b )$, subject to the unbiasedness $\mathbb{E} \left[ x_i(t)^\top \mu-y_i(t)^\top b \right]=0$.  This linear prediction, called best linear unbiased prediction~\citep{harville1976extension,robinson1991blup}, is invariant of $x_i(t)$, and is given by $b = D^\top \mu$ for $D=(A^\top \Sigma_\xi^{-1} A + \Sigma_x^{-1} )^{-1} A^\top \Sigma_\xi^{-1}.$ Thus, for a policy that has access to $A, \Sigma_\xi, \Sigma_x, \left\{ \mu_i \right\}_{i \in [N]}$, it suffices to plug in $\mu_i$ to obtain the predict the expected reward in \eqref{eq:reward} by $y_i(t)^\top D^\top \mu_i$. This reflects the optimal policy to compete against, as we will discuss in detail shortly. 
	
	However, from the perspective of the bandit policy that $A, \Sigma_\xi, \Sigma_x, \left\{ \mu_i \right\}_{i \in [N]}$ are unknown, context vectors are unobserved, and $a(t)$ needs to be decided merely based on $y_i(t)$, further steps and finer analyses are required. First, let us rewrite the reward in \eqref{eq:reward} as
	\begin{eqnarray}
		r_i(t) = y_i(t)^\top D^\top \mu_i + \zeta_i(t)\label{eq:reward2} ,
	\end{eqnarray}
	where  $\zeta_i(t) = (x_i(t)^\top \mu_i - y_i(t)^\top D^\top \mu_i)+\varepsilon_i(t)$ is a mean-zero noise uncorrelated with the observation; $\mathbb{E}\left[ \zeta_i(t) y_i(t) \right]=0_{d_y}$. Intuitively, \eqref{eq:reward2} provides the reward in terms of $y_i(t)$, and so the noise $\zeta_i(t)$ encapsulates the original noise in the obtained reward as in \eqref{eq:reward}, as well as the error caused by the imperfectness in observing the contexts. Moreover, the transformed parameters \begin{eqnarray}
		\eta_i := D^\top \mu_i \label{eq:etai}
	\end{eqnarray}
    for $i \in [N]$, correspond to the model parameters one can (at best) hope to learn by using the partial observations of the contexts. Specifically, in the case $d_y< d_x$ that is of primary interest from a practical point of view, $\mu_i$ is not learnable based on the data $\left\{ r_{a(\tau)}(\tau) \right\}_{\tau \in [t-1]}, \left\{y_1(\tau), \cdots, y_N(\tau)\right\}_{\tau\in [t]}$ at time $t$. Note that an accurate estimate of $\eta_i$ cannot lead to that of $\mu_i$, due to the fact that a bandit policy does not have access to the matrices constituting $D$. Still, the policy needs to balance the trade-off between pulling different arms to learn the best about them, versus earning the most reward by selecting the best arm at all time steps.     
	
Now, using \eqref{eq:reward2} and \eqref{eq:etai}, we get the reward model
	\begin{eqnarray}
		r_i(t) = y_i(t)^\top \eta_i + \zeta_i(t)\label{eq:newreward}.
	\end{eqnarray}
The above discussions indicate that subject to the fact that the contexts $x_i(t), i \in [N]$ are   observed only through the lens of $y_i(t), i\in [N]$ in \eqref{eq:obmodel}, the optimal policy that fully knows $D$ and $\mu_i$ for $i \in [N]$ and use this knowledge to select the arm of highest expected reward, is 
	$$a^\star(t) = \text{argmax}_{i\in[N]}~~ y_i(t)^\top \eta_i.$$
We refer to $a^\star(t)$ as {an optimal arm} at time $t$. Then, similar to other problems in sequential decision-making, regret is the performance measure. It is indeed the decrease in cumulative reward caused by uncertainties the bandit policy needs to cope with, as compared to the optimal policy. So, at time $T$, the regret of the policy that pulls $a(t) \in [N]$ at round $t$ is 
	$$\mathrm{Regret}(T) = \sum_{t=1}^T \left( y_{a^\star(t)}(t)^\top\eta_{a^\star(t)}- y_{a(t)}(t)^\top \eta_{a(t)} \right).$$ 
	
 \begin{rem}
     As above, we aim to compete against an optimal policy $a^\star(t)$ that knows $A, \Sigma_\xi, \Sigma_x$, as well as $\left\{ \mu_i \right\}_{i \in [N]}$, which \emph{all are unknown} to the bandit policy. However, to have a well-defined problem, full observations of the context vectors $\{x_i(t)\}_{i\in[N]}$ are not available to the optimal policy, as elaborated below. 
 \end{rem}
Indeed, exact knowledge of contexts changes the setting essentially and nullifies the problem, because the regret with respect to such policy, \emph{cannot} grow sublinearly with time. This relies on the fact that even with fully known reward parameter, bandit policies might select sub-optimal arms due to their uncertainty about the contexts. So, since contexts vary with time, such suboptimal pulls persist as we proceed, causing a linear regret (with positive probability). Moreover, the above-mentioned optimal policy aligns with the existing literature of partially observed contextual bandits~\citep{kim2023contextual,jose2024thompson}.  
 
Now, we express the further technicalities that will be used in the upcoming theoretical analyses. Note that the bandit algorithm in Section \ref{sec:3} does not need knowledge of the quantities introduced below. First, we define exhaustive and exclusive events in the observation space that correspond to each arm being optimal. 
	\begin{defn}[Optimality Region]\label{def:astar}
		Concatenate the observations in $y(t)= \left( y_1(t)^\top,\dots,y_N(t)^\top \right)^\top$ and let $A_i^\star \subset \mathbb{R}^{Nd_y}$ be the region in the space of $y(t)$ that makes arm $i$ optimal. That is, as long as $y(t) \in A_i^\star$, it holds that $a^\star(t)=i$. Further, denote the optimality probability of arm $i$ by
		\begin{eqnarray*}
			p_i = \mathbb{P}(y(t)\in A_i^\star) =  \mathbb{P}(a^\star(t) = i).
		\end{eqnarray*}    
	\end{defn}


	 The assumption below states the margin condition and properly modifies a similar assumption in the work of \cite{bastani2021mostly} to the setting of partially observable contextual bandits.
    \begin{assum} [Margin Condition]
    Consider the normalized observation vectors $\dot{y}_i(t) = y_i(t)/\|y(t)\|$ for $i\in [N]$ and the transformed parameters $\{\eta_i\}_{i\in [N]}$ in \eqref{eq:etai}. We assume there is $C > 0$ such that for all $u>0$ and all $i \in[N]$ of positive optimality probability $p_i$ in Definition~\ref{def:astar}, it holds that
		\begin{eqnarray*}
			\max\limits_{j \in [N], j\neq i} ~~\mathbb{P} \left(0<\dot{y}_i(t)^\top \eta_i-\dot{y}_j(t)^\top\eta_j \leq u \Big| y(t) \in A_i^\star \right) \leq C u.
		\end{eqnarray*}     
		\label{ass:mar}    
    \end{assum}
The expression above bounds the conditional probability of $\dot{y}_i(t)^\top \eta_i-\dot{y}_j(t)^\top\eta_j$, which is the suboptimality gap of the $j$-th arm in the case that arm $i$ is optimal. The above assumption states that optimal arms are highly likely to be distinguishable. More precisely, it expresses that the likelihood of suboptimality gaps being smaller than $u$ is proportional to $u$. The above inequality holds, for example, if the sensing noise or the context vectors have bounded probability density functions all over their Euclidean spaces~\citep{faradonbeh2018finite,wong2020lasso}. 


 As a result of Assumption \ref{ass:mar}, for all $i,~j\in [N]$, there exist a subset $A_i^\kappa\subseteq A_i^\star$ and $\kappa>0$ such that
	\begin{eqnarray}
		\mathbb{P}(y(t)\in A_i^\kappa) > \frac{p_i}{2}~~~~~~~\text{and}~~~~~~~~\mathbb{P}(\dot{y}_i(t)^\top \eta_i-\dot{y}_j(t)^\top\eta_j > \kappa|y(t)\in A_i^\kappa)=1.\label{eq:kappa}
	\end{eqnarray}
	Following the previous paragraph, note that $\kappa$ is the minimum value that the suboptimality gap can have with a positive probability (0.5 in this case) given the event $A_i^\star$, and is an instance-dependent constant. Moreover, the role of $\kappa$ in the analysis of Algorithm \ref{algo1} for partially observable stochastic contextual bandits of this work, is intrinsically similar to the role of the well-known \emph{gap} in multi-armed bandits\footnote{Rigorously speaking, it is the difference between the expected reward of the best arm, and that of the second best arm.}~\citep{lattimore2020bandit}. Note that our policy in the next section, does \emph{not} need any information about $\kappa$. 
 
 The assumptions so far are typical ones in the bandit literature. On the other hand, the next assumption is adopted for simplifying expressions in the probabilistic analysis of how the reward values are affected by the information lost in the sensing process (i.e., the imperfectness of context observations).

 
    \begin{assum}[Sub-Guassianity]
    For the context vectors $x_i(t)$ and the corresponding observations $y_i(t)$ in \eqref{eq:obmodel}, there exists a positive constant $R_2>0$ such that for all $\mu \in \mathbb{R}^{d_x}, \eta \in \mathbb{R}^{d_y}$, and $\lambda \in \mathbb{R}$, it holds that
        \begin{eqnarray*}
            \max\limits_{i \in [N]}~~\mathbb{E}\left[ \exp \left( {\lambda \left( x_i(t)^\top \mu - 
            y_i(t)^\top \eta \right)} \right) \Big| y_i(t)\right] \leq \exp \left({\frac{\lambda^2 R_2^2}{2}} \right). 
        \end{eqnarray*}
        \label{ass:subg}
    \end{assum}


    The above expression can be understood as conditional sub-Gaussianity of reward-prediction-error $x_i(t)^\top \mu_i -y_i(t)^\top \eta_i$ given the observation $y_i(t)$. Therefore, its role is similar to the tail properties commonly assumed in finite sample analysis of learning errors \citep{abbasi2011improved,agrawal2013thompson}. This assumption holds for a general class of stochastic measurement errors and contexts, including Gaussian and bounded random vectors, and rules out heavy-tailed distributions and those that the covariance structure of $x_i(t) \big| y_i(t)$ grows with $\| y_i(t)\|$. 

    
	\begin{rem}
       Equivalents of the the presented algorithms and results for perfectly observed contextual bandits can be obtained by simply letting $A=I_{d_x}$ and $\Sigma_\xi \to \mathbf{0}_{d_y\times d_y}$.
	\end{rem}

	\section{Thompson Sampling with Partial Contextual Observations}
	\label{sec:3}

	In this section, we outline a version of the well-known Thompson sampling algorithm that can be implemented using \emph{only} the observation vectors. The idea of Thompson sampling is built on posterior distribution that measures the belief about a parameter based on data. However, Thompson sampling is recognized for its robust performance even in the absence of the exact posterior distribution. This robustness stems from the fact that the primary goal of a decision-making algorithm is to maximize cumulative rewards through the exploration-exploitation trade-off, rather than striving for precise inference. Consequently, Thompson sampling works effectively even if there exist mismatches between actual reward distributions and hypothetical ones in Thompson sampling. 
 
 

	Now, we introduce Thompson sampling for partially observed contextual bandits. For decision-making only with partially observed contexts, a decision-maker approximates the true reward distribution given only $y_i(t)$ to construct a (hypothetical) posterior distribution for Thompson sampling. Thus, based on \eqref{eq:newreward}, which is the true reward distribution given only $y_i(t)$, the decision-maker assumes that the reward of the $i$-th arm at time $t$ is generated as follows:  
	\begin{eqnarray}
		r_i(t) = y_i(t)^\top \eta_i + \psi_i(t), \label{eq:pos}
	\end{eqnarray}
	where $\psi_i(t)$ is a noise with the normal distribution with the mean 0 and variance $v^2$. The decision-maker choose a value of the posterior dispersion parameter $v^2$ subject to $v^2 \geq R^2 = R_1^2+R_2^2$, where $R_1$ and $R_2$ in are introduced in \eqref{eq:r1} and Assumption \ref{ass:subg}, respectively. 
 This ensures that the Thompson sampling algorithm appropriately assesses the magnitude of reward error given an observation, considering both the magnitude of partial observation errors and reward errors for decision-making. In the beginning, the decision-maker starts with the initial value $\widehat{\eta}_i(1) = \mathbf{0}_{d_y}$ and $B_i(1) = I_{d_y}$ for all $i\in [N]$, which are the mean and (unscaled) inverse covariance matrix of a prior distribution of $\eta_i$, respectively. 
	The posterior distribution of $\eta_i$ at time $t$ is given as $\mathcal{N}(\widehat{\eta}_i(t), v^2 B_i(t)^{-1})$, where the closed-form expressions of $\widehat{\eta}_i(t)$ and $B_i(t)$ are given below:
 \begin{eqnarray}
		B_i(t) &=& \sum_{\tau=1}^{t-1} y_i(\tau)y_i(\tau)^\top \mathbb{I}(a(\tau)=i) + I_{d_y},\label{eq:Bic}\\
		\widehat{\eta}_i(t) &=& B_i(t)^{-1} \left( \sum_{\tau =1 }^{t-1} r_i(\tau)y_i(\tau) \mathbb{I}(a(\tau)=i)\right)\label{eq:etahatc}.
\end{eqnarray}
 Then, we sample from the following posterior distribution of the transformed parameters $\eta_i$:
	\begin{eqnarray}
		\widetilde{\eta}_i(t) \sim \mathcal{N}(\widehat{\eta}_i(t),v^2B_i(t)^{-1}),~~i=1,2\dots,N.\label{eq:sample}
	\end{eqnarray}
	
 Accordingly, the decision-maker pulls the arm $a(t)$ such that
	$$a(t) =  \underset{1\leq i\leq N}{\text{argmax}}~ y_i(t)^\top \widetilde{\eta}_i(t),$$ 
	as if the samples generated in \eqref{eq:sample} are the true values.  Then, once the decision-maker gains the reward of the chosen arm $a(t)$, it can update the parameters of posterior distribution $\widehat{\eta}_i(t)$ and $B_i(t)$ based on the recursions below:
    \begin{eqnarray}
    B_i(t+1) &=& B_i(t) + y_i(t) y_i(t)^\top \mathbb{I}(a(t)=i),\label{eq:Bi} \\
    \widehat{\eta}_i(t+1) &=& B_i(t+1)^{-1} \left(B_i(t) \widehat{\eta}_i(t) + y_i(t) r_i(t)\mathbb{I}(a(t)=i)\right).\label{eq:etahat}
    \end{eqnarray}
    Here, the quantities of only the chosen arm $a(t)$ are updated, while those of the other arms remain unaltered.  It is noteworthy that the estimation of optimal arm does not require any information about other parameters such as $A$, $D$, $\Sigma_x$, and $\Sigma_\xi$. The pseudo-code is provided in Algorithm \ref{algo1}. 

\begin{algorithm}[t] 
		\begin{algorithmic}[1]
			\State Set $B_i(1) = I_{d_y}$, $\widehat{\eta}_i(1) = \mathbf{0}_{d_y}$~for~$i = 1,2, \dots, N$
			\For{$t = 1,2, \dots, $}
			\For{$i = 1,2, \dots, N$}
			\State Sample $\widetilde{\eta}_i(t)$ from $\mathcal{N}(\widehat{\eta}_i(t),v^2B_i^{-1}(t))$
			\EndFor
			\State Select arm $a(t) = \text{argmax}_{i\in[N]} y_i(t)^\top \widetilde{\eta}_i(t)$
			\State Gain reward $r_{a(t)}(t) = x_{a(t)}(t)^\top \mu_{a(t)} + \varepsilon_{a(t)}(t)$
			\State Update $B_i(t+1)$ and $\widehat{\eta}_i(t+1)$ by \eqref{eq:Bi} and \eqref{eq:etahat}~for~$i = 1,2, \dots, N$
			\EndFor  
		\end{algorithmic}
		\caption{: Thompson sampling algorithm for partially observable contextual bandits}  
		\label{algo1}
	\end{algorithm}

	\section{Theoretical Performance Analyses}
	\label{sec:4}
	
	In this section, we establish the theoretical results of Algorithm \ref{algo1} for partially observable contextual bandits with arm-specific parameters. The following results provide estimation error bounds of the estimators defined in \eqref{eq:etahatc} and a high probability regret upper bound for Algorithm \ref{algo1}. It is worth noting that the accuracy of parameter estimation and regret growth are closely related because higher estimation accuracy leads to lower regret. Thus, we build the estimation accuracy first and then construct a regret bound based on it. The first theorem presents the estimation error bound, which scales with the rate of the inverse of the square root of $t$. 

	\begin{thm}[Partial Estimation Accuracy]
		Let $\eta_i$ and $\widehat{\eta}_i(t)$ be the transformed true parameter in \eqref{eq:etai} and its estimate in \eqref{eq:etahatc}, respectively. Then, with probability at least $1-\delta$, Algorithm \ref{algo1} guarantees
		\begin{flalign}
\| \widehat{\eta}_i(t) - \eta_{i} \| =  \mathcal{O}\left(\frac{R}{ t^{1/2}}\sqrt{\frac{d_y}{p_i} \log \frac{TNd_y}{\delta}}\right),\nonumber
\end{flalign}
for all arms $i\in [N]$ and at all times $t$ in the range $\tau_i^{(1)} <t\leq T$, where $R = \sqrt{R_1^2 + R_2^2}$ and $\tau_i^{(1)}=\mathcal{O}(p_i^{-2}Nd_y^{3.5}\kappa^{-5}\log^{5}(TNd_y/\delta))$ is the minimum time $t$ the algorithm is run. 
		\label{thm:eta2}
	\end{thm}



 
	The theorem above indicates that the estimation error bound decreases in $t$. This is not trivial in the sense that the estimation accuracy of parameter for the $i$-th arm generally increases with respect to the number of the $i$-th arm selected, $n_i(t)$, instead of the overall 
 horizon $t$. To prove this, we first show the sub-Gaussian property of reward prediction error given an observation for a confidence ellipsoid of the transformed parameters using Assumption \ref{ass:subg}. By putting linear growing eigenvalues of $B_i(t)$ regarding $n_i(t)$ together with the confidence ellipsoid, we have the square-root estimation accuracy with respect to $n_i(t)$. Finally, to get the square-root estimation accuracy with respect to $t$, we prove that $n_i(t)$ scales linearly with $t$ with a high probability. In this process, the minimum time of $t$, $\tau_i^{(1)}$, which is of the order $\mathcal{O}(p_i^{-2}Nd_y^{3.5}\kappa^{-5}\log^{5}(TNd_y/\delta))$, is required. This is primarily due to the gap analysis involving $\kappa$ and the truncation of observations, which is necessary for the applications of concentration inequalities. 
 Moving forward, the following theorem demonstrates that the regret upper bound scales with a rate at most $\log^{5.5} T$ with respect to the time. This result is built on Theorem \ref{thm:eta2} that guarantees the square-root estimation accuracy of each reward parameter.

\paragraph{Proof Outline.} The complete proof is provided in Appendix \ref{sec:pthm:eta2}; Here, we provide a brief sketch of the proof of Theorem \ref{thm:eta2}. First, we find the high-probability truncation of observation norm, which is necessary for the uses of concentration inequalities in the following steps. We denote the observation norm truncation by $L$ such that $\|y_i(t)\| \leq L$ for all $i\in[N]$ and $t\in[T]$. Using martingale techniques, we construct confidence ellipsoids for the norm of self-normalized estimation error such that $\|\widehat{\eta}_i(t) - \eta_i\|_{B_i(t)} = \mathcal{O}\left(R \sqrt{d_y \log \left((1+L^2 T)/\delta\right)}\right)$ with probability at least $1-\delta$ for all $i\in[N]$ and $t\in[T]$. Next, to derive an upper bound for $\|\widehat{\eta}_i(t) - \eta_i\|$ from the above confidence ellipsoid, we analyze the eigenvalues of $B_i(t)$, which is the sum of dependent random matrices. Using a concentration inequality, we show that the minimum eigenvalue of $B_i(t)$ grows linearly with the sample size of the $i$-th arm $n_i(t)$. Subsequently, we get
\begin{eqnarray}
        \|\widehat{\eta}_i(t) - \eta_i\|\leq \mathcal{O}\left(\frac{R}{n_i(t)^{1/2}} \sqrt{d_y \log \left(\frac{1+L^2 T}{\delta}\right)}\right).\label{eq:err}
\end{eqnarray}
Now, we consider the event where the suboptimality gaps of all suboptimal arms are equal or larger than $\kappa$ given the $i$-th arm optimal, $A_i^\kappa$.  Given this event, we can transform the contextual bandit problem with changing suboptimality gaps to that of bandits with a fixed suboptimality gap $\kappa$. Then, we show that the probability of the $i$-th arm not being chosen given $A_i^\kappa$ defined in \eqref{eq:kappa}, $\mathbb{P}(a^\star(t) \neq i|A_i^\kappa)$, decreases with a rate $n_i(t)^{-1/2}$. Using this, we construct high-probability bounds for $n_i(t)$ such that $n_i(t) \geq (p_i t)/4$. Finally, by plugging $(p_i t)/4$ into the $n_i(t)$ in \eqref{eq:err}, we get the result of Theorem \ref{thm:eta2}. \hfill$\square$


	\begin{thm}
		The regret of Algorithm \ref{algo1} satisfies the following with probability at least $1-\delta$:
		\begin{eqnarray*}
			\mathrm{Regret}(T) =  \mathcal{O}\left(\frac{vNd_y^{4} }{ {p^+_{\min}}^2\kappa^5  } \log^{5.5}\left(\frac{TNd_y}{\delta  }\right)\right),
		\end{eqnarray*}
		where $p_{\min}^+ = \min_{i\in[N]:p_i>0}p_i$.
		\label{thm:reg2}
	\end{thm}

    
	This theorem demonstrates that the regret scales at most $\log^{5.5} T$ with time, and with the rate ${p^+_{\min}}^{-2}$ with respect to the optimal probabilities. In addition, the term $N$ is caused by the use of the inclusion-exclusion formula to find the bound of the sum of probabilities that the optimal arms are not chosen over time. Next, $d_y^4$ is incurred by the repeated application of the observation norm truncation $L$, which has the order of $\mathcal{O}(\sqrt{d_y\log TNd_y/\delta)})$. Furthermore, the regret bound increases at rate $\kappa^{-5}$, stemming from the minimum time $\tau_M$. Lastly, a larger value of posterior dispersion parameter $v$ exacerbates the regret linearly.

	\paragraph{Proof Outline.}  The complete proof is provided in Appendix \ref{sec:pthm:reg2}; For this proof outline, we focus on the effects of $T$, $N$, $L$ and $\tau_M$.  First, we intuitively show that the regret roughly grows with $\log^2 T$ over time
 based on the fact that: i) regret is the sum of expected reward gaps, which is incurred when a suboptimal arm is chosen, 
 ii) the expected reward gap is $\mathcal{O}(t^{-1/2}\log (TNd_y/\delta))$ and the probability of choosing a suboptimal arm decreases at rate $t^{-1/2}$, resulting in their product decreasing at rate $t^{-1}$ over time, and iii) the sum of product terms diminishing with $t^{-1}$ is of the order $\mathcal{O}\left(\log^2 T\right)$. 
 But, because the order of expected reward gap $\mathcal{O}\left(t^{-1/2}\log (TNd_y/\delta)\right)$ is achieved, only when $t$ is greater than the minimum time $\tau_M =\mathcal{O}\left(NL^7\log^{1.5}(TNd_y/\delta)\right)$, the regret bound is looser than $\mathcal{O}(\log^2 T)$. Based on this intuition, we split the regret as follows:
\begin{eqnarray}
    \mathrm{Regret}(T) = \sum_{t=1 }^{\lfloor \tau_M \rfloor} \mathrm{gap}(t)\mathbb{I}(a^\star(t)\neq a(t)) + \sum_{t=\lceil \tau_M\rceil}^T \mathrm{gap}(t)\mathbb{I}(a^\star(t)\neq a(t)), \label{eq:regdec}
\end{eqnarray}  

where $\mathrm{gap}(t)=y_{a^\star(t)}(t)^\top \eta_{a^\star(t)}(t) - y_{a(t)}(t)^\top \eta_{a(t)}(t)$ represents the expected decrease in reward. Because the first term in \eqref{eq:regdec} is $\mathcal{O}(L \tau_M)$, it suffices to find the order of the second term in \eqref{eq:regdec} to find an upper bound for regret. 

 $\mathrm{gap}(t)$ decreases with rate $t^{-1/2}$, because $\mathrm{gap}(t) \leq \max_{i\in[N]} 2L\|\widetilde{\eta}_i(t)-\eta_i\| = \mathcal{O}\left(Lt^{-1/2}\right)$ for $t> \tau_M$. Accordingly, we aim to show a high-probability bound for $\sum_{t=\lceil \tau_M\rceil}^T t^{-1/2} \mathbb{I}(a^\star(t)\neq a(t))$. By showing $\mathbb{E}[\mathbb{I}(a^\star(t)\neq a(t))|\{a(\tau)\}_{\tau\in[t-1]}] = \mathcal{O}(N t^{-1/2})$, we get a concentration result of $\sum_{t=\lceil \tau_M\rceil}^T \mathrm{gap}(t) \mathbb{I}(a^\star(t)\neq a(t))=\mathcal{O}\left(LN\log T\right)$. 
 Because, the order of first term in \eqref{eq:regdec} dominates that of the second term, the regret has the order $\mathcal{O}(L\tau_M)$, which corresponds to the result of Theorem \ref{thm:eta2}, considering $N,~d_y,~T$ and $\delta$ with $L = \mathcal{O}(\sqrt{d_y\log (TNd_y/\delta)})$ and $\tau_M =\mathcal{O}\left(NL^7\log^{1.5}(TNd_y/\delta)\right)$.
\hfill$\square$


\begin{cor}
    If the observations are assumed to be generated from a distribution with a support bounded by $L$ such that $\|y_i(t)\| \leq L$, the regret of Algorithm \ref{algo1} is of the order
$$\mathrm{Regret}(T) =\mathcal{O}\left(vNd_y{p^+_{\min}}^{-2}\kappa^{-5} \log^2\left(TNd_y/\delta\right)\right).$$
\end{cor}

 The truncation of observations $L$ with the order of $\mathcal{O}(\sqrt{d_y\log (TNd_y/\delta)})$ leads to additional $d_y^3\log^{3.5} (TNd_y/\delta)$ by increasing the minimum time. By letting $L$ be a constant, which does not depend on any factors, we get the result of the corollary above from Theorem \ref{thm:reg2}.

\begin{rem}
Contextual bandits with a shared reward parameter are obtained by setting $\mu_i = \mu_j$ for all $i, j \in [N]$, building on the setup described thus far. Similarly, $x_i(t) = x_j(t)$ for all $i,~j\in[N]$, suffices for obtaining the shared context setting. All other quantities too, can be defined in the same manner.
\end{rem}

\begin{cor}
    For partially observed contextual bandits with the shared parameter, where the following equalities hold: 
$$n_i(t) = t, \eta_i = \eta_\centerdot,  \widehat{\eta}_i(t) = \widehat{\eta}_\centerdot(t), \text{ and } B_i(t) = B_\centerdot(t),$$ 
    the regret of Algorithm \ref{algo1} is of the order
$$\mathrm{Regret}(T) =\mathcal{O}\left(vNd_y^{2.5}\log^{3.5}\left(\frac{TNd_y}{\delta}\right)\right).$$
\end{cor}


	
	The above results are unprecedented even for fully observed contextual bandits as well as partially observed ones with the arm-specific parameter setup to the best of our knowledge. Especially, a high probability poly-logarithmic regret bound for Thompson sampling with respect to the time horizon has not been shown for stochastic contextual bandits, even though the previously available regret bound for Thompson sampling has a square-root order with respect to time for the adversarially chosen contexts \citep{agrawal2013thompson}. A logarithmic regret upper bound for stochastic contextual bandits is shown to be achieved for the arm-specific parameter setup by the greedy first algorithm that takes a greedy action if a criterion is met and explores otherwise 
 \citep{bastani2021mostly}. However, the above regret bound is valuable in that the greedy first algorithm needs another standard algorithm such as Thompson sampling and OFU-type algorithms for exploration.
	
	
	
	For the shared parameter setup, Thompson sampling for contextual bandits with stochastic contexts has been more widely studied compared to the setup involving arm-specific parameters \citep{abeille2017linear,chakraborty2023thompson,ghosh2022breaking}. In the case where stochastic contexts have a positive covariance matrix, it becomes possible to naturally explore the entire parameter space, eliminating the necessity for explicit exploration strategies. As a result, Thompson sampling directly ensures square root accuracy in estimating the reward parameter, leading to the achievement of a poly-logarithmic upper bound on regret.
	
\section{Numerical Experiments}
	\label{sec:5}

	\begin{figure*}[h]
		\centering
		\psfrag{dx=10~~~~}{\scriptsize$d_x=10$}
		\psfrag{dx=20~~~~}{\scriptsize$d_x=20$}
		\psfrag{dx=40~~~~}{\scriptsize$d_x=40$}
		\psfrag{dx=80~~~~}{\scriptsize$d_x=80$}
		\psfrag{~~~~dy=10}{\scriptsize$d_y=10$}
		\psfrag{~~~~dy=20}{\scriptsize$d_y=20$}
		\psfrag{~~~~dy=40}{\scriptsize$d_y=40$}
		\psfrag{~~~~dy=80}{\scriptsize$d_y=80$}
		\psfrag{time}{\scriptsize time}
		\psfrag{Regret~~~~~~~~~~}{\scriptsize Normalized Regret}
  \includegraphics[width=0.8\textwidth]{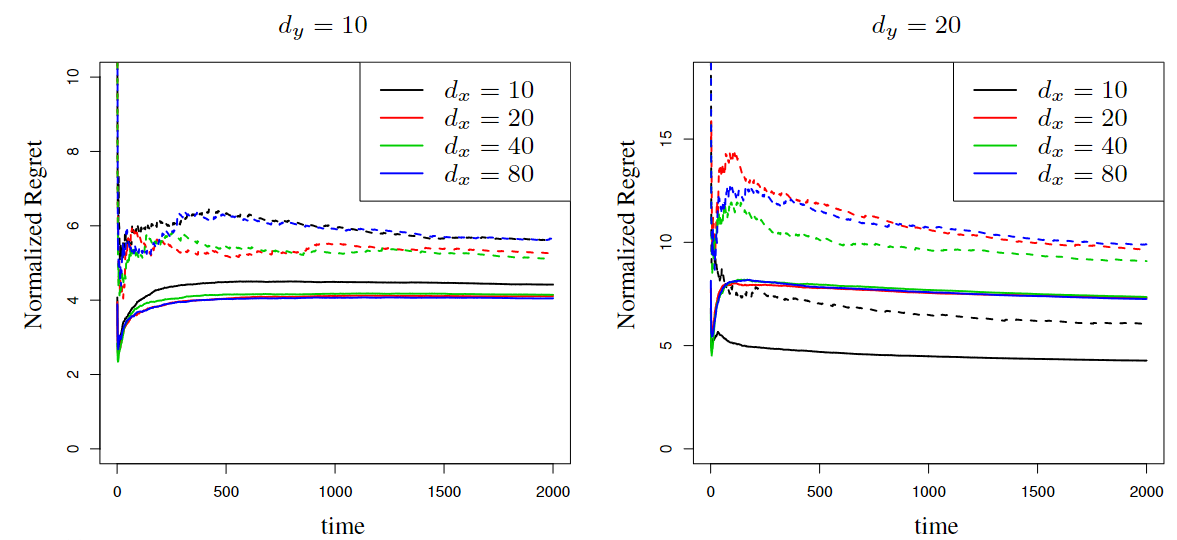}
  \includegraphics[width=0.8\textwidth]{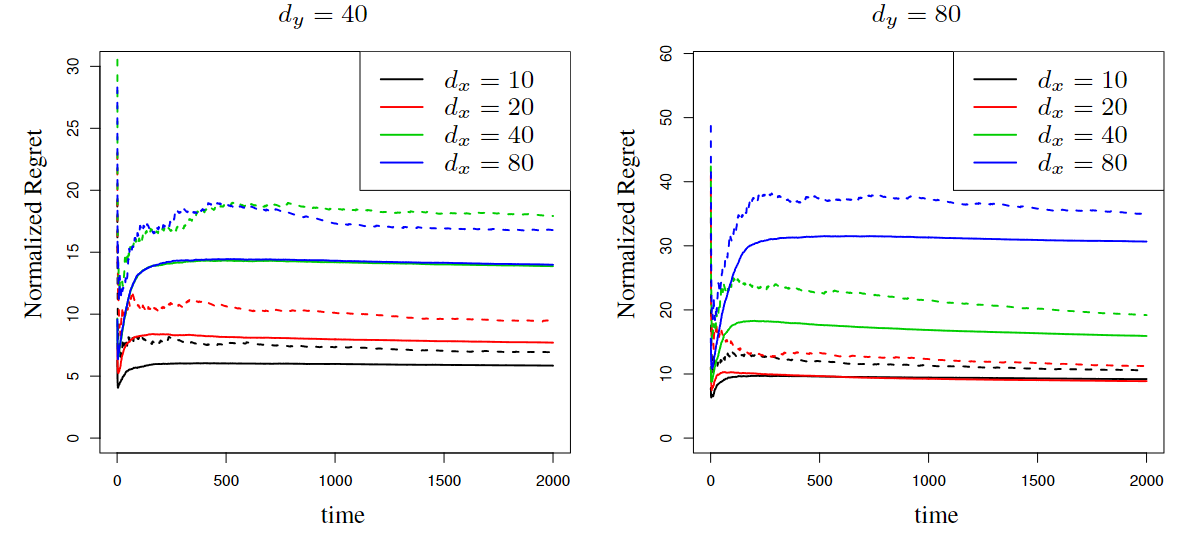}
		\normalsize 
		\caption{Plots of $\mathrm{Regret}(t)/(\log t)^2$ over time for the different dimensions of context at $N = 5$ and  $d_y=10,20,40,80$. The solid and dashed lines represent the average-case and worst-case regret curves, respectively.} 
	\label{fig:1}
\end{figure*}

\begin{figure*}[h]
	\centering
	\psfrag{time}{\scriptsize time}
	\psfrag{Error~~~~~~~~~~}{\scriptsize Normalized Error}
	\psfrag{Estimation~~~~~~}{\scriptsize Normalized Regret}
	\psfrag{Estimation1~~~~~~}{\scriptsize $d_x=10,~d_y=20$}
	\psfrag{Estimation2~~~~~~}{\scriptsize $d_x=20,~d_y=20$}
	\psfrag{Estimation6~~~~~~}{\scriptsize $d_x=40,~d_y=20$}
	\psfrag{1~~~~~~}{\scriptsize arm 1}
	\psfrag{2~~~~~~}{\scriptsize arm 2}
	\psfrag{3~~~~~~}{\scriptsize arm 3}
	\psfrag{4~~~~~~}{\scriptsize arm 4}
	\psfrag{5~~~~~~}{\scriptsize arm 5}
 \includegraphics[width=1\textwidth]{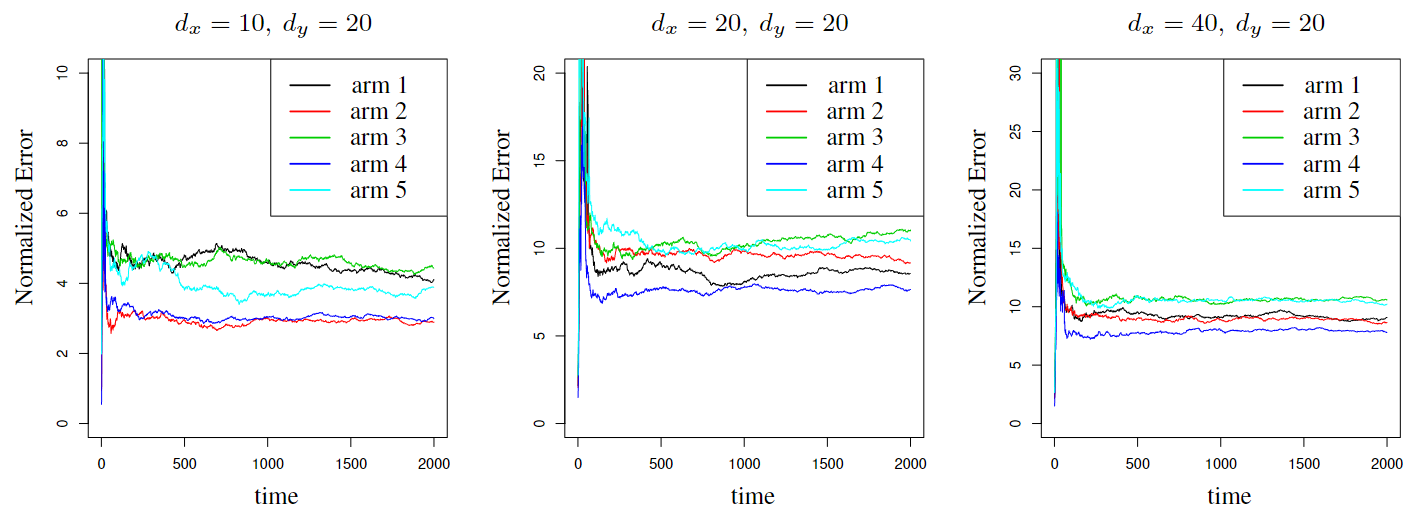}
	\normalsize 
	\caption{Plots of normalized estimation errors $\sqrt{t}\|\widehat{\eta}_i(t)-\eta_i\| $ of Algorithm \ref{algo1} over time for partially observable stochastic contextual bandits with five arm-specific parameters and dimensions of observations and contexts $d_y=20$, $d_x=10,~20,~40$.} 
\label{fig:2}
\end{figure*}

\begin{figure*}[h]
\centering
\psfrag{t}{\scriptsize time}
\psfrag{gr~~~~~~~~~~}{\tiny Greedy}
\psfrag{ts~~~~~~~~~~~~~~~~~~~~~~~~~~~}{\tiny Thompson Sampling}
\psfrag{regret}{\scriptsize Regret}
\psfrag{N=5~~~}{\scriptsize $N=5$}
\psfrag{N=10~~~}{\scriptsize $N=10$}
\psfrag{N=20~~~}{\scriptsize $N=20$}
\psfrag{N=30~~~}{\scriptsize $N=30$}
\includegraphics[width=1\textwidth]{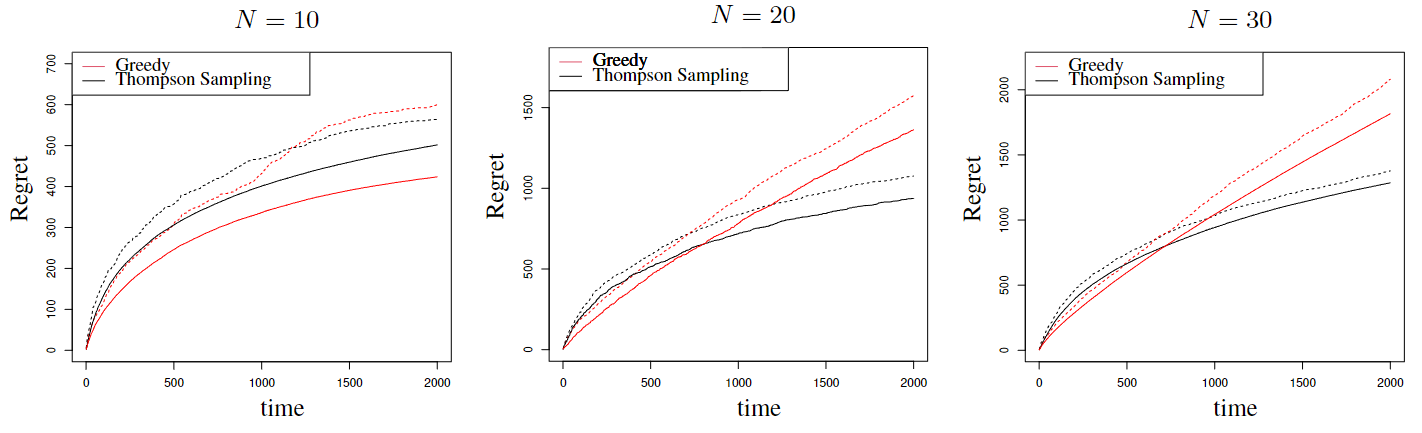}
\normalsize 
\caption{Plots of regrets over time with the different number of arms $N = 10,~20,~30$ for Thomson sampling versus the Greedy algorithm. The solid and dashed lines represent the average-case and worst-case regret curves, respectively.}
\label{fig:3}
\end{figure*}

\paragraph{Simulation Experiments:}
\label{ssec:se}
In this sub-section, we numerically show the results in Section \ref{sec:4} with synthetic data. First, to explore the relationships between the regret and dimension of observations and contexts, we simulate various scenarios for the model with arm-specific parameters with $N=5$ arms and different dimensions of the observations $d_y=10,~20,~40,~80$ and context dimension $d_x=10,~20,~40,~80$. Each case is repeated $50$ times and the average and worst quantities amongst all $50$ scenarios are reported. Figure \ref{fig:1} illustrates regret normalized by $(\log t)^2$, which is the regret growth that the minimum time effect is removed. Second, Figure \ref{fig:2} showcases the average estimation errors of the estimates in \eqref{eq:etahatc} for five different arm-specific parameters defined in \eqref{eq:etai}, changing dimensions of observations and contexts. These errors are normalized by $t^{-1/2}$ based on Theorem \ref{thm:eta2}. Since the error decreases with a rate $t^{-1/2}$, the normalized errors for all the arms are flattened over time. This demonstrates that the square-root accuracy estimations of $\{\eta_i\}_{i=1}^N$ are available regardless of whether the dimension of observations is greater or less than that of contexts.

Moving on, Figure \ref{fig:3} provides insights into the average and worst-case regrets of Thompson sampling compared to the Greedy algorithm, with variations in the number of arms ($N=10, 20, 30$). It is worth noting that the Greedy algorithm is considered optimal for the model with a shared parameter, but the worst-case regret of it exhibits linear growth in the model with arm-specific parameters. The worst-case linear regret growth of the greedy algorithm can occur when some arms, which are totally dominated by other arms, are missing in potential action because of no explicit exploration scheme. In Figure \ref{fig:3}, the plots represent the average and worst-case regrets of the models with arm-specific parameters, showing that the greedy algorithm has greater worst-case regret for the model with arm-specific parameters, especially for the case with a large number of arms.

\begin{figure*}[h]
\centering
\psfrag{time}{\scriptsize time}
\psfrag{mc~~~~~~~~~~~~~~~~~~~~~~~~~~~~~~}{\tiny Average Correct Decision Rate}
\psfrag{oc~~~~~~~~~~}{\tiny Regression Oracle}
\psfrag{ts~~~~~~~~~~~~~~~~~~~~}{\tiny Thompson Sampling}
\psfrag{egg~~~}{\scriptsize EGG}
\psfrag{t}{\scriptsize time}
\psfrag{emv~~~~~~~~~~~~~~~}{\scriptsize Eye Movement}    
\includegraphics[width=0.48\textwidth]{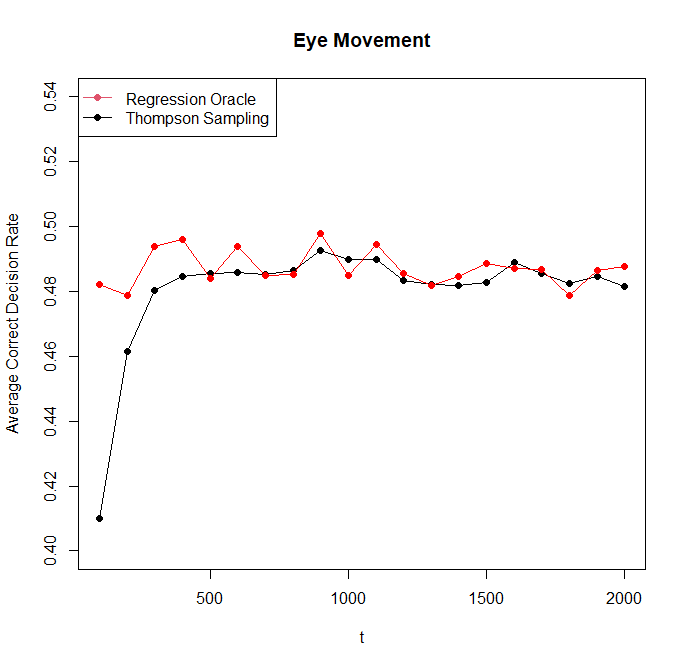}
\includegraphics[width=0.48\textwidth]{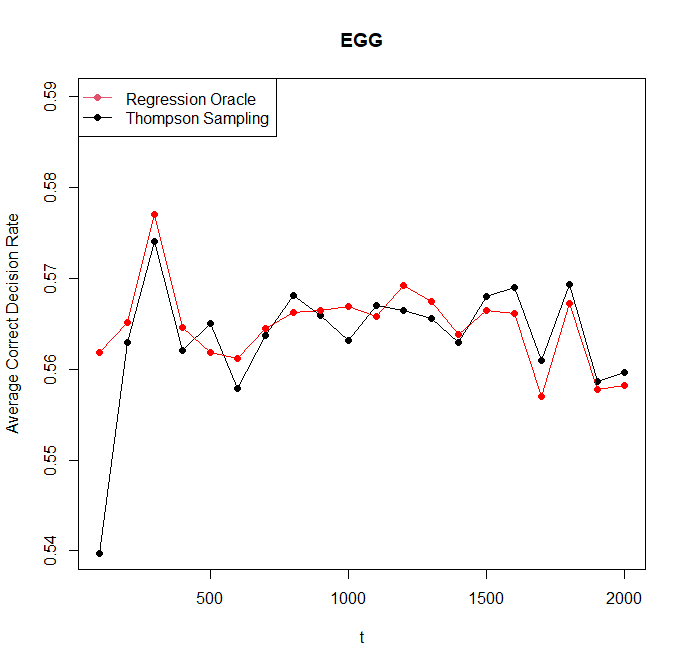}
\normalsize 
\caption{Plots of average correct decision rates of the regression oracle and Thompson sampling for Eye movement (left) and EGG dataset (right).}
\label{fig:real}
\end{figure*}

\paragraph{Real Data Experiments:}
\label{ssec:re}

In this sub-section, we assess the performance of the proposed algorithm using two healthcare datasets: Eye Movement and EGG\footnote{The datasets are publicly available at: \url{https://www.openml.org/}}. These two datasets are presented in previous studies by \cite{bastani2020online,bietti2021contextual} using contextual bandits with arm-specific parameters and shared context. These datasets involve classification tasks based on patient information. The Eye Movement and EGG data sets are comprised of 26 and 14-dimensional (shared) contexts, respectively, with the corresponding patient class. 
Also, the number of class for Eye Movement and EGG datasets are 3 and 2, respectively, and each category of patient class is considered an arm in the perspective of the bandit problem. We analyze these datasets under the logistic linear regression assumption, where reward is assumed to be generated based on \eqref{eq:reward} and 
\begin{eqnarray}
    \log \frac{\mathbb{P}(l(t) = i)}{1-\mathbb{P}(l(t)=i)} = x(t)^\top \mu_i = \mathbb{E}[r_i(t)],\label{eq:reward22}
\end{eqnarray}
where $l(t)$ is the true label of the patient at time $t$. Because the datasets do not have rewards based on this setup, we generated rewards based on \eqref{eq:reward} and \eqref{eq:reward22} with artificial noises.

For evaluation, we generate 100 scenarios for each dataset. We calculate the average correct decision rate defined as $t^{-1} \sum_{\tau=1}^t \mathbb{I}(a(\tau) = l(\tau))$. We compare the suggested algorithm against the regression oracle with the estimates trained on the entire data in hindsight, which are not updated over time. We artificially create observations of the patients' contexts based on the structure given in \eqref{eq:obmodel} with a sensing matrix $A$ consisting of 0 and 1 only. We reduce the dimension of the patient contexts from 26 to 13 for the Eye movement dataset and from 14 to 10 for the EGG dataset. Figure \ref{fig:real} displays the average correct decision rates of the regression oracle and Thompson sampling for the two real datasets. We evaluate the mean correct decision rates over every 100 patients and then average them across 100 scenarios. Accordingly, each dot represents the sample mean of 10,000 results. For both data sets, the correct decision rate of Thompson sampling converges to that of the regression oracle over time. More results of real data experiments are provided in Appendix \ref{sec:real2}.

\section{Concluding Remarks and Future Work}
\label{sec:6}

We studied Thompson sampling for partially observable stochastic contextual bandits under relaxed assumptions with a particular focus on the arm-specific parameter setup. Indeed, the suggested model is versatile, encompassing a wide range of possible observation structures and offering estimation methods suitable for stochastic contexts. Further, we showed that Thompson sampling guarantees the square-root consistency of parameter estimation for reward parameters. Finally, we proved regret bounds for Thompson sampling with a poly-logarithmic rate for the most common two cases of parameter setups. Our techniques for the analysis hold for other analogous reinforcement learning problems such as a Markov Decision Process thanks to the inclusive assumptions and comprehensive approaches.

A topic of prospective research involves proposing and examining algorithms designed for partially observable contextual bandits, where arms are clustered. Additionally, there is an opportunity to explore the introduction of non-linear structures into both the observation and reward models. Lastly, investigating this framework in the presence of an adversary presents a fascinating challenge for future investigations.

\newpage
\addcontentsline{toc}{part}{Appendices}
\bibliography{mybib1}
\bibliographystyle{abbrvnat}

\newpage
\tableofcontents
\newpage

\appendix

\section*{Organization of Appendices}
The appendices are organized as follows. First, we provide experiments with real datasets, which are not shown in Section \ref{sec:4}  due to space constraints. Second, Appendix \ref{sec:sps} describes the shared parameter setup and Thompson sampling algorithm for it. Next,  Appendix \ref{sec:rgf} presents the theoretical results for the general model, with comprehensive proofs. Furthermore, Appendix \ref{sec:rss} provides the worst-case regret upper bounds for the model with a shared parameter, accompanied by its detailed proof. Lastly,  Appendix \ref{sec:pthm:eta2} and Appendix \ref{sec:pthm:reg2} present the complete proofs of Theorem \ref{thm:eta2} and Theorem \ref{thm:reg2}, respectively.

\section{Real Data Experiments}

\label{sec:real2}

\begin{figure*}[h]
\centering
\psfrag{time}{\scriptsize time}
\psfrag{mc~~~~~~~~~~~~~~~~~~~~~~~~~~~~~~}{\tiny Average Correct Decision Rate}
\psfrag{oc~~~~~~~~~~}{\tiny Regression Oracle}
\psfrag{ts~~~~~~~~~~~~~~~~~~~~}{\tiny Thompson Sampling}
\psfrag{egg~~~}{\scriptsize EGG}
\psfrag{t}{\scriptsize time}
\psfrag{emv~~~~~~~~~~~~~~~}{\scriptsize Eye Movement}
\includegraphics[width=0.4\textwidth]{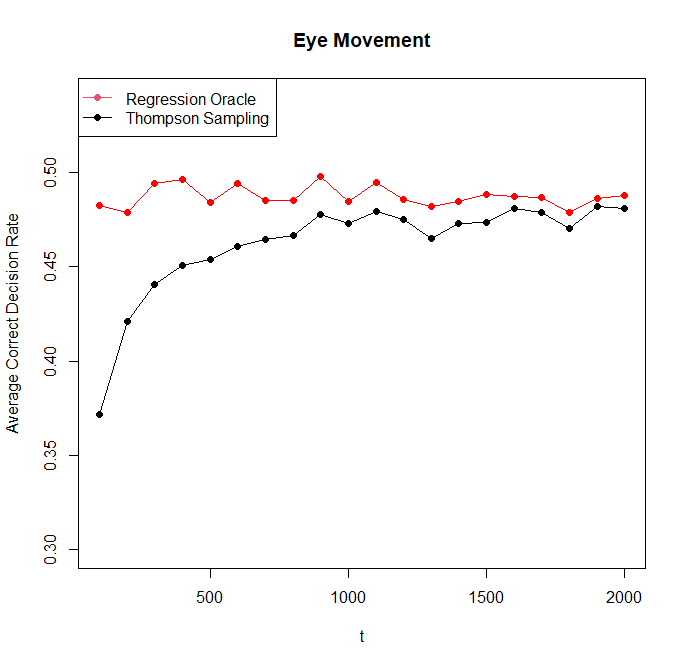}
\includegraphics[width=0.4\textwidth]{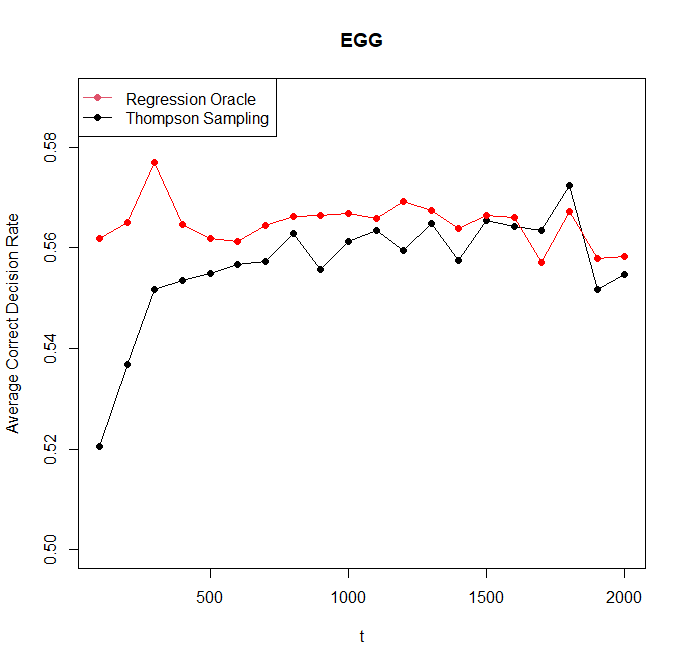}
\includegraphics[width=0.4\textwidth]{eye1.png}
\includegraphics[width=0.4\textwidth]{egg1.png}
\normalsize 
\caption{Plots of average correct decision rates of the regression oracle and Thompson sampling for Eye Movement (top left) and EGG dataset (top right) under the simple linear regression setup and Eye Movement (bottom left) and EGG dataset (bottom right) under the logistic linear regression setup.}
\label{fig:real1}
\end{figure*}

\begin{figure*}[h]
\centering
\psfrag{time}{\scriptsize time}
\psfrag{mc~~~~~~~~~~~~~~~~~~~~~~~~~~~~~~}{\tiny Average Correct Decision Rate}
\psfrag{oc~~~~~~~~~~}{\tiny Regression Oracle}
\psfrag{ts~~~~~~~~~~~~~~~~~~~~}{\tiny Thompson Sampling}
\psfrag{egg~~~}{\scriptsize EGG}
\psfrag{t}{\scriptsize time}
\psfrag{emv~~~~~~~~~~~~~~~}{\scriptsize Eye Movement}
\includegraphics[width=0.4\textwidth]{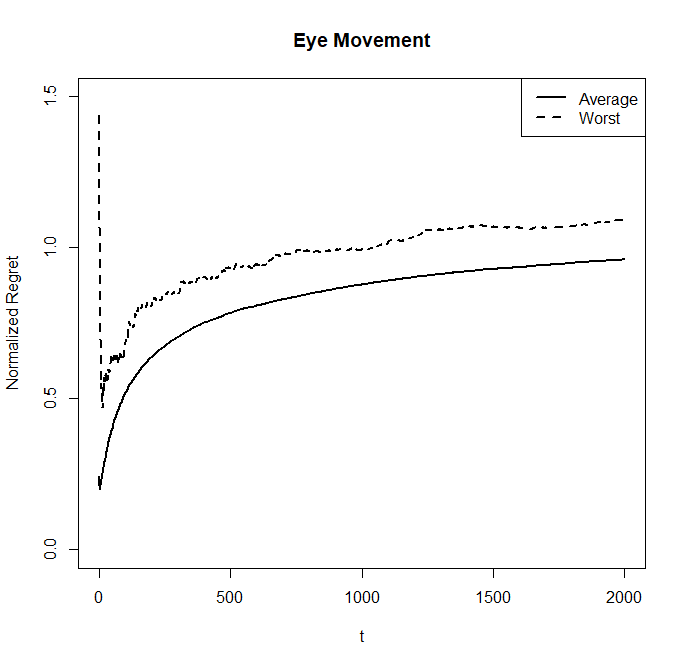}
\includegraphics[width=0.4\textwidth]{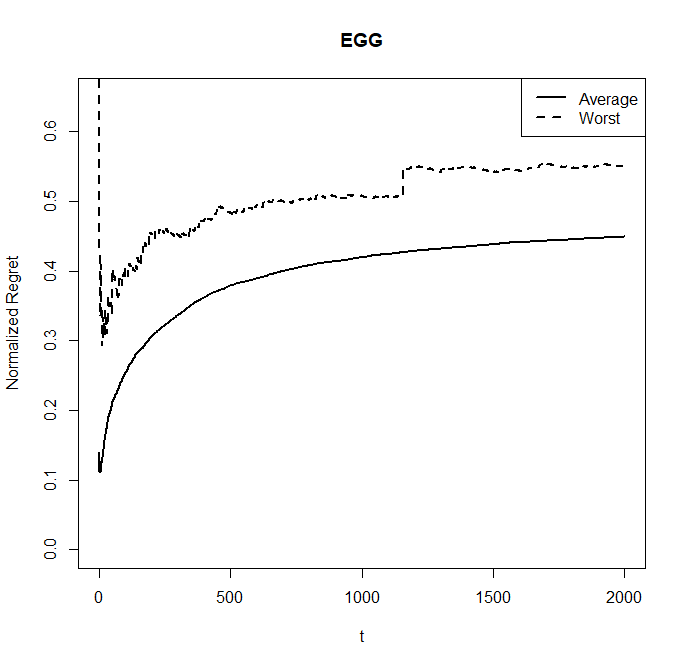}
\includegraphics[width=0.4\textwidth]{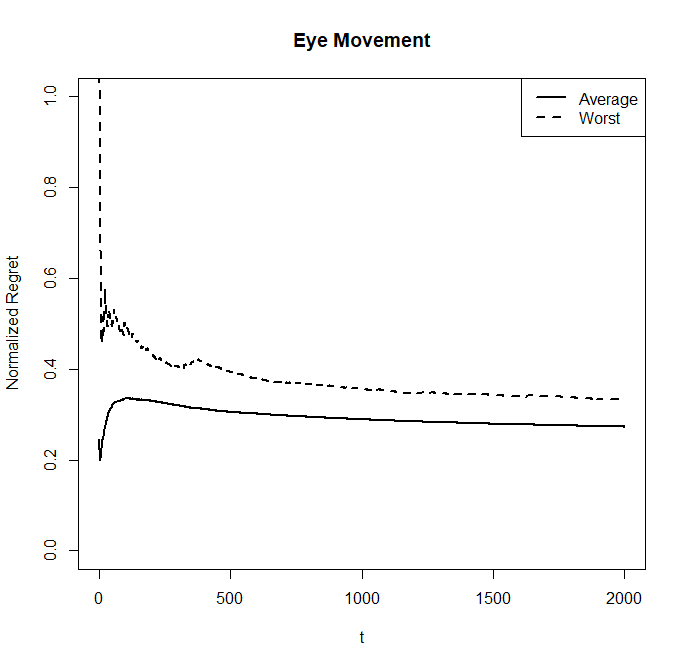}
\includegraphics[width=0.4\textwidth]{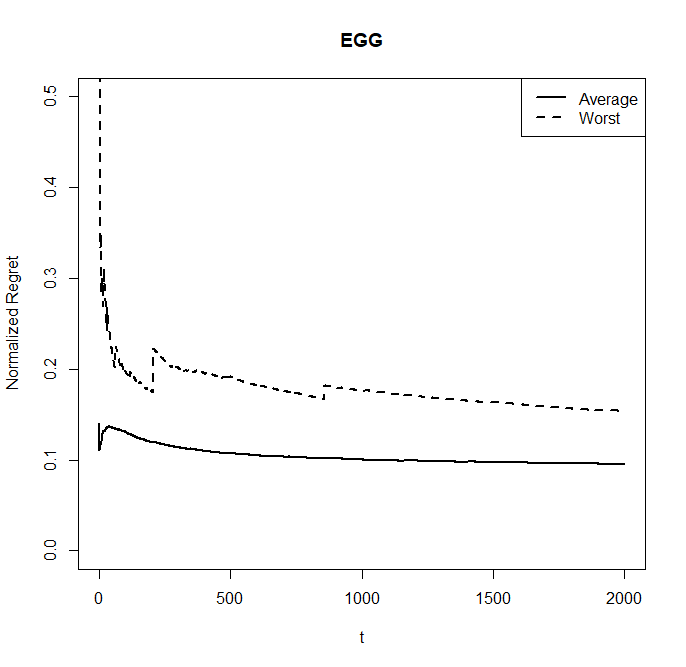}
\normalsize 
\caption{Plots of normalized regret of the regression oracle and Thompson sampling for Eye Movement (top left) and EGG dataset (top right) under the simple linear regression setup and Eye Movement (bottom left) and EGG dataset (bottom right) under the logistic linear regression setup.}
\label{fig:real2}
\end{figure*}

In this section, we analyze two realdata sets used in Section \ref{sec:4} under two different assumptions. The first assumption is a simple linear regression, where a decision-maker gets a reward of 1 for successful classification and 0 otherwise.  The reward is assumed to be generated as follows:
\begin{eqnarray*}
    r_i(t) = x(t)^\top \mu_i + \varepsilon_i(t) = \mathbb{I}(l(t)=i),
\end{eqnarray*}
where $l(t)$ is the true label of the patient randomly chosen at time $t$ and $x(t)^\top \mu_i$ represents $\mathbb{P}(l(t)=i|x(t))$. 
This assumption is prone to a reward model misspecification since the expected value $\mathbb{P}(l(t)=i|x(t))$ is constrained to be between 0 and 1. Next, the second assumption is a logistic linear regression, introduced in Section \ref{sec:4}.

For evaluation, we generate 100 scenarios for each dataset. We calculate the (average and worst case) regret as well as the average correct decision rate introduced in Section \ref{sec:4} for Thompson sampling versus regression oracle. We consider the estimates of regression oracle as the truth for regret evaluation for Thompson sampling. The observations are generated in the same manner introduced in Section \ref{sec:4}.

Figure \ref{fig:real1} displays the average correct decision rates of the regression oracle and Thompson sampling for the two real datasets under the two assumptions. We evaluate the mean correct decision rates over every 100 patients and then average them across 100 scenarios. Accordingly, each dot represents a sample mean of 10,000 results. For both data sets, the correct decision rate of Thompson sampling converges to that of the regression oracle over time. In addition, Figure \ref{fig:real2} illustrates the average and worst cases of normalized regret of Thompson sampling against the regression oracle. Regret grows slightly faster than $\log^2 t$ for the simple linear model, but seems to scale with at most $\log^2 t$ in the logistic regression model. The over growth of regret in the first model can be caused by a potential model misspecification \citep{foster2020adapting}. 

\newpage
\section{Shared Parameter Setup}
\label{sec:sps}
In Section \ref{sec:2}, we describe the arm-specific parameter setup. In this section, we introduce the shared parameter setup, where the reward is generated
\begin{eqnarray*}
    r_i(t) = x_i(t)^\top \mu_\centerdot + \varepsilon_i(t).
\end{eqnarray*}
In this setup, the parameter $\mu_\centerdot$ is shared across all arms. Accordingly, the parameter can be learned regardless of the actions taken. Thus, the transformed parameter can be written as 
$$\eta_\centerdot = D^\top \mu_\centerdot.$$ 
The optimal arm is 
\begin{eqnarray*}
    a^\star(t) = \text{argmax}_{i\in[N]} ~y_i(t)^\top \eta_\centerdot
\end{eqnarray*}
and subsequently regret is
\begin{eqnarray*}
\mathrm{Regret}(T) =\sum_{t=1}^T (y_{a^\star(t)}(t) -y_{a(t)}(t))^\top \eta_\centerdot.
\end{eqnarray*}

The Thompson sampling algorithm for the shared parameter setup is basically the same as Algorithm \ref{algo1}, but simpler than it in that the estimate of transformed parameters and (unscaled) inverse covariance matrices are the same for all arms. Thus, we use the notation $\widehat{\eta}_\centerdot(t)$ and $B_\centerdot(t)$ for the estimate and (unscaled) inverse covariance, respectively. The update procedure for estimators is
\begin{eqnarray}
B_\centerdot(t+1) &=& B_\centerdot(t) + y_{a(t)}(t) y_{a(t)}(t)^\top,\label{eq:Bi2} \\
\widehat{\eta}_\centerdot(t+1) &=& B_\centerdot(t+1)^{-1} \left(B_\centerdot(t) \widehat{\eta}_\centerdot(t) + y_{a(t)}(t) r_{a(t)}(t)\right),\label{eq:etahat2}
\end{eqnarray}
which is similar to \eqref{eq:Bi} and \eqref{eq:etahat}, but there is not the indicator function $\mathbb{I}(a(t)=i)$. This implies that a decision-maker can learn the entire reward parameter regardless of the chosen arm in the shared parameter setup, while it learns an arm-specific parameter only given the arm chosen in the arm-specific parameter setup. The pseudo-code for the Thompson Sampling algorithm is described in Algorithm \ref{algo2}.

\begin{algorithm}[t] 
	\begin{algorithmic}[1]
		\State Set $B_\centerdot(1) = I_{d_y}$, $\widehat{\eta}_\centerdot(1) = \mathbf{0}_{d_y}$~for~$i = 1,2, \dots, N$
		\For{$t = 1,2, \dots, T$}
        \For{$i = 1,2, \dots, N$}
        \State Sample $\widetilde{\eta}_i(t)$ from $\mathcal{N}(\widehat{\eta}_\centerdot(t),v^2B_\centerdot(t)^{-1})$
        \EndFor
        \State Select arm $a(t) = \text{argmax}_{i\in[N]} y_i(t)^\top \widetilde{\eta}_i(t)$ 
        \State Gain reward $r_{a(t)}(t) = x_{a(t)}(t)^\top \mu_\centerdot + \varepsilon_{a(t)}(t)$
        \State Update $B_\centerdot(t+1)$ and $\widehat{\eta}_\centerdot(t+1)$ by \eqref{eq:Bi2} and \eqref{eq:etahat2}
		\EndFor  
	\end{algorithmic}
\caption{: Thompson sampling algorithm for partially observable contextual bandits with a shared parameter}  
\label{algo2}
\end{algorithm}

\newpage
\section{Auxiliary Lemmas}
\label{sec:rgf}


In this section, we present auxiliary lemmas that serve as building blocks for the main results in Section \ref{sec:4}. To begin, Lemma \ref{lem:tr} establishes a truncation bound for the following steps of proofs. Next, Lemma \ref{lem:subg}, supported by Lemma \ref{lem:dti}, guarantees the sub-Gaussian tail property for the reward prediction error given an observation. Additionally, Lemma \ref{lem:eig} demonstrates that the minimum eigenvalue of $B_i(t)$ grows linearly with the sample size $n_i(t)$ with a high-probability. Furthermore, Lemma \ref{lem:azu} is Azuma's inequality. Lastly, Lemma \ref{lem:eta0} and \ref{lem:eta1} provide upper bounds for the estimation error and sample bias, respectively.

We first show the general results which are valid for both arm-specific and shared parameter setups. To proceed, we state an assumption about the parameter space, which is commonly adapted in the antecedent literature \citep{dani2008stochastic,goldenshluger2013linear,bastani2021mostly,kargin2023thompson}. 

\begin{assum}[Parameter Bounds]
		For the transformed reward parameters $\{\eta_i\}_{i\in[N]}$,  there exists a positive constant $c_\eta$ such that
		$\|\eta_i\| \leq c_\eta$, for all $i=1,\dots,N$. 
		\label{ass:para}
\end{assum}
	
 Note that according to the above assumption, a similar bound also holds for the parameters $\{\mu_i\}_{i\in[N]}$ so that their norms are bounded by a positive constant $c_\mu$. This assumption expresses that the unknown reward parameters live in an unknown bounded region. Intuitively, this enables us to control the effect of parameter sizes on regret growth.

Since each element of a context and observation noise is sub-Gaussian and the sum of two sub-Gaussian random variables is sub-Gaussian as well, based on \eqref{eq:obmodel}, a positive number $c_y$ exists such that
\begin{eqnarray}
    \mathbb{E}\left[e^{\lambda y_{ij}(t)}\right] \leq \exp\left(\frac{\lambda^2c_y^2}{2}\right),\label{eq:cy}
\end{eqnarray}
for all real $\lambda > 0$, where $y_{ij}(t)$ is the $j-$th element of $y_i(t)$. Next, we find a high-probability upper bound for the norm of observations for the following steps. To find the high probability bound for $\|y_i(t)\|$ for a confidence level $\delta> 0$, we define $W_T$ such that

\begin{eqnarray}
W_T = \left\{\underset{\{i\in[N],\tau \in [T]\}}{\max} ||y_i(\tau)||_{\infty} \leq  v_T(\delta)\right\} \label{eq:WT},
\end{eqnarray} 
where $v_T(\delta) = c_y\sqrt{2\log (2TNd_y/\delta)}$. In the next lemma, we show that the event $W_T$ happens with probability at least $1-\delta$.

\begin{lem}
For the event $W_T$ defined in \eqref{eq:WT}, we have $\mathbb{P}(W_T) \geq 1 -\delta$.\label{lem:tr}
\end{lem}
\begin{proof}
By \eqref{eq:cy} and the properties of sub-Gaussian random variables, 
\begin{eqnarray*}
    \mathbb{P}\left(|y_{ij}(t)| \geq \varepsilon \right) \leq 2\cdot e^{-\frac{\epsilon^2}{2c_y^2}},
\end{eqnarray*}
is satisfied for given $i,j\in [N]$. Accordingly, we have

$$\mathbb{P}\left( \|y_i(t)\|_{\infty} \geq \varepsilon \right) \leq 2 d_y \cdot e^{-\frac{\varepsilon^2}{2c_y^2}}.$$ 

By taking the union of the events over time, we get
\begin{eqnarray*}
\mathbb{P}\left( \max_{i\in[N],\tau \in [T]}\|y_i(t)\|_\infty \geq \varepsilon \right) \leq 2TNd_y  \cdot e^{-\frac{ \varepsilon^2}{2c_y^2}}
\end{eqnarray*}

By plugging $c_y(2\log (2TNd_y/\delta))^{1/2}$ in $\varepsilon$, we have
\begin{eqnarray*}
\mathbb{P}\left( \max_{i\in[N],\tau \in [T]}\|y_i(t)\|_\infty \geq  c_y( 2 \log (2TNd_y/\delta))^{1/2} \right) \leq 2TNd_y  \cdot \exp\left(-\frac{ 2c_y^2\log (2TNd_y/\delta)}{2c_y^2}\right)=\delta.
\end{eqnarray*}
Thus,
\begin{eqnarray*}
\mathbb{P} (W_T) \geq 1 - \mathbb{P}\left( \max_{i\in[N],\tau \in [T]}\|y_i(t)\|\geq v_T(\delta) \right) \geq 1-\delta.
\end{eqnarray*}
\end{proof}

Then, by Lemma \ref{lem:tr}, we have
\begin{eqnarray}
    \|y_i(t)\| \leq \sqrt{d_y} v_T(\delta):=L=\mathcal{O}\left(\sqrt{d_y \log(TNd_y/\delta)}\right)\label{eq:ybnd},
\end{eqnarray}
for all $1\leq i \leq N$ and $1\leq t \leq T$ with probability at least $1-\delta$.

 The next lemma presents that reward prediction errors given observations have the sub-Gaussian property when observations and rewards have sub-Gaussian distributions, and thereby, a confidence ellipsoid is constructed for the estimator in \eqref{eq:etahatc}. This result is built on Theorem 1 in \cite{abbasi2011improved} with proper modifications.

\begin{lem} Let $w_t = r_{a(t)}(t) - y_{a(t)}(t)^\top \eta_{a(t)}$ and  $\mathscr{F}_{t-1}=\sigma\{\{y(\tau)\}_{\tau=1}^t,\{a(\tau)\}_{\tau=1}^t,\{r_{a(\tau)}(\tau)\}_{\tau=1}^{t-1}\}$. Then, $w_t$ is $\mathscr{F}_{t-1}$-measurable and conditionally $R$-sub-Gaussian for some $R>0$ such that
\begin{eqnarray*}
    \mathbb{E}[e^{\nu w_t}|\mathscr{F}_{t-1}]\leq \exp\left(\frac{\nu^2 R^2 }{2}\right).
\end{eqnarray*}
In addition, for any $\delta > 0$, with probability at least $1-\delta$, we have
    \begin{eqnarray*}
        \|\widehat{\eta}_i(t) - \eta_i\|_{B_i(t)} = \left\|\sum_{\tau=1}^{t-1}  y_i(\tau)w_\tau\mathbb{I}(a(\tau)=i) \right\|_{B_i(t)^{-1}}
        \leq R \sqrt{d_y \log \left(\frac{1+L^2 n_i(t)}{\delta}\right)} +c_\eta.
    \end{eqnarray*}
\label{lem:subg}
\end{lem}
This lemma provides the sub-Gaussianity for the reward prediction error $w_t$ given $y_i(t)$, and shows a self-normalized bound for a vector-valued martingale $
\sum_{\tau=1}^{t-1} y_i(\tau)w_\tau\mathbb{I}(a(\tau)=i)$. 
The reward estimation error $w_t$ can be decomposed into two parts. The one is the reward error $\varepsilon_i(t)$ given \eqref{eq:reward} due to the randomness of rewards. This error is created even if the context $x_i(t)$ is known. The other is the reward mean prediction error $x_i(t)^\top \mu_i-y_i(t) \eta_i$ caused by unknown contexts. The first step of the proof for this lemma is to show the sub-Gaussian property of $w_t$ based on the decomposition. Next, using the sub-Gaussian property of reward prediction errors, we construct a confidence ellipsoid for the transformed reward estimator in \eqref{eq:etahat} with some martingale techniques.

\begin{proof}
To show the sub-Gaussianity of $w_t$ given the observation $y(t)$, 
we use the following decomposition of $ r_i(t) -  y_i(t)^\top D^\top \mu_i$:
\begin{eqnarray}
    r_i(t) -  y_i(t)^\top D^\top \mu_i = (r_i(t) - x_i(t)^\top \mu_i) + (x_i(t)^\top \mu_i - y_i(t)^\top D^\top \mu_i). \label{eq:dec}
\end{eqnarray}
The first and second terms on the RHS in \eqref{eq:dec} are $R_1$ and $R_2$-sub-Gaussian by \eqref{eq:r1} and Assumption \ref{ass:subg}, respectively. Because $r_i(t) - x_i(t)^\top \mu_i$ in the RHS of \eqref{eq:dec} is independent from others, we have
$$\mathbb{E}[e^{\nu(r_i(t) -  y_i(t)^\top D^\top \mu_i)}|y(t)] = \mathbb{E}[e^{\nu\varepsilon_i(t)}] \mathbb{E}[e^{\nu(x_i(t)^\top \mu_i- y_i(t)^\top D^\top \mu_i)}|y(t)] \leq \exp\left(-\frac{\nu^2 R_1^2}{2}\right)\exp\left(-\frac{\nu^2 R_2^2}{2}\right).$$
Thus, we have 
\begin{eqnarray}
    \mathbb{E}[e^{\nu(r_i(t) -  y_i(t)^\top D^\top \mu_i)}|y(t)] \leq \exp\left(-\frac{\nu^2 R^2}{2}\right),\label{eq:r}
\end{eqnarray}
where $R^2=R_1^2+R_2^2$. Now, we construct a confidence ellipsoid of the transformed reward parameter based on the sub-Gaussian property of the reward prediction error. 

\begin{lem}
Let 
\begin{eqnarray*}
    D_{it}^{\eta} = \exp\left(  \frac{(r_{a(t)}(t) - y_{a(t)}(t)^\top \eta_{a(t)})y_{a(t)}(t)^\top \eta_{a(t)} }{R} -\frac{1}{2} (y_{a(t)}(t)^\top \eta_{a(t)})^2 \right)^{\mathbb{I}(a(t)=i)},
\end{eqnarray*} 
$M_{it}^{\eta} = \prod_{\tau=1}^t D_{i\tau}^{\eta}$ and $t^\star$ be a stopping time. Then, $\mathbb{E}[M_{it^\star}^{\eta}] \leq 1$.
\label{lem:dti}
\end{lem}
\begin{proof}
First, we take the expected value of $D_{it}^{\eta}$ conditioned on $\mathscr{F}_{t-1}$ and arrange it as follows:
\begin{eqnarray*}
    &&\mathbb{E}[D_{it}^{\eta}|\mathscr{F}_{t-1}] \\
    &=& \mathbb{E}\left[\left.\exp\left(  \frac{(r_{a(t)}(t) - y_{a(t)}(t)^\top \eta_{a(t)})y_{a(t)}(t)^\top \eta_{a(t)} }{R} -\frac{1}{2} (y_{a(t)}(t)^\top \eta_{a(t)})^2 \right)^{\mathbb{I}(a(t)=i)}\right|y(t),a(t)\right]\\\
    &=&\mathbb{E}\left[\left.\exp\left(  \frac{\zeta_{a(t)}(t)y_{a(t)}(t)^\top \eta_{a(t)}}{R} \right)^{\mathbb{I}(a(t)=i)}\right|y(t),a(t)\right]\exp\left( -\frac{1}{2} (y_{a(t)}(t)^\top \eta_{a(t)})^2 \right)^{\mathbb{I}(a(t)=i)}.
\end{eqnarray*}

Then, by \eqref{eq:r}, we have
\begin{eqnarray*}
    &&\mathbb{E}\left[\left.\exp\left(  \frac{\zeta_{a(t)}(t)y_{a(t)}(t)^\top \eta_{a(t)}}{R} \right)^{\mathbb{I}(a(t)=i)}\right|y(t),a(t)\right]\exp\left( -\frac{1}{2} (y_{a(t)}(t)^\top \eta_{a(t)})^2 \right)^{\mathbb{I}(a(t)=i)}\\
    &\leq& \left(\exp\left(\frac{1}{2} (y_{a(t)}(t)^\top \eta_{a(t)})^2 \right)\exp\left( -\frac{1}{2} (y_{a(t)}(t)^\top \eta_{a(t)})^2 \right)\right)^{\mathbb{I}(a(t)=i)}= 1.
\end{eqnarray*}
Thus, we have
\begin{eqnarray*}
\mathbb{E}[M_{it}^{\eta}|\mathscr{F}_{t-1}] = \mathbb{E}[M_{i1}^{\eta} D_{i2}^{\eta} \cdots D_{i(t-1)}^{\eta_i} D_{it}^{\eta}|\mathscr{F}_{t-1}] = D_{1}^{\eta}\cdots D_{i(t-1)}^{\eta} \mathbb{E}[D_{it}^{\eta}|\mathscr{F}_{t-1}] \leq M_{i(t-1)}^{\eta},
\end{eqnarray*}
showing that $\{M_{i\tau}^\eta\}_{\tau=1}^\infty$ is a supermartingale and accordingly
\begin{eqnarray*}
    \mathbb{E}[M_{it}^{\eta}] = \mathbb{E}[\mathbb{E}[M_{it}^{\eta}|\mathscr{F}_{t-1}]] \leq \mathbb{E}[M_{i(t-1)}^{\eta}] \leq \cdots \leq  \mathbb{E}[\mathbb{E}[D_{i1}^{\eta}|\mathscr{F}_1]]  \leq 1.
\end{eqnarray*}

Next, we examine the quantity $M_{it^\star}^{\eta}$. Since $M_{it}^{\eta}$ is a nonnegative supermartingale, by Doob's martingale convergence theorems \citep{doob1953stochastic}, $M_{it}^{\eta}$ converges to a random variable, which is denoted by $M_i^\eta$. Let $Q_{it}^\eta = M_{i\min(t,t^\star)}^\eta$ be the stopping time version of $\{M_{it}^\eta\}_t$. Then, by Fatou's Lemma \citep{rudin1976principles}, we have
\begin{eqnarray*}
    \mathbb{E}[M_{it^\star}^\eta] = \mathbb{E}[\text{liminf}_{t\rightarrow \infty} Q_{it}^\eta] \leq \text{liminf}_{t\rightarrow \infty} \mathbb{E}[Q_{it}^\eta] \leq 1.
\end{eqnarray*}

\end{proof}

Now, we continue the proof of Lemma \ref{lem:subg}.
Let $\phi_{\eta_i}$ be the probability density function of multivariate Gaussian distribution of $\eta_i$ with the mean $\mathbf{0}_{d_y}$ and the covariance matrix $v^2I_{d_y}$. By Lemma 9 in \cite{abbasi2011improved}, we have
\begin{eqnarray}
    \mathbb{P}_{\phi_{\eta_i}}\left(\left\|S_{it^\star} \right\|_{B_i(t^\star)^{-1}}^2  > 2R^2 \log \left(\frac{\text{det}(B_i(t^\star))^{1/2} }{\delta}\right)\right) \leq  \delta,\label{eq:stau}
\end{eqnarray}
where 
$\mathbb{P}_{\phi_{\eta_i}}$ denotes the probability measure associated with $\phi_{\eta_i}$ representing the distribution of $\eta_i$  and $S_{it}= \sum_{\tau=1}^{t-1}  y_{a(\tau)}(\tau) w_\tau \mathbb{I}(a(\tau)=i)$. Lemma \ref{lem:dti} and \eqref{eq:stau} are sufficient conditions for the use of Theorem 1 in \cite{abbasi2011improved}, thus we get
\begin{eqnarray}
    \mathbb{P}_{\phi_{\eta_i}}\left( \exists t^\star <\infty~s.t.~\left\|S_{it^\star} \right\|_{B_i(t^\star)^{-1}}^2 > 2R^2 \log \left(\frac{\text{det}(B_i(t^\star))^{1/2} }{\delta}\right)\right) \leq \delta.\label{eq:stopping}
\end{eqnarray}
By Lemma 10 in \cite{abbasi2011improved}, we have
\begin{eqnarray*}
    \text{det}(B_i(t)) \leq (1+ n_i(t)L^2/d_y)^{d_y},
\end{eqnarray*}
and subsequently, we have
\begin{eqnarray*}
    2 \log \left(\frac{\text{det}(B_i(t))^{1/2} }{\delta}\right) \leq d_y \log \left(\frac{1+L^2 n_i(t)}{\delta}\right).
\end{eqnarray*}
Thus, with probability at least $1-\delta$, we get
\begin{eqnarray*}
    \left\|S_{it} \right\|_{B_i(t)^{-1}}^2 < R \sqrt{d_y \log \left(\frac{1+L^2 n_i(t)}{\delta}\right)} + c_\eta,
\end{eqnarray*}
for all $t>0$. 
Because $S_{it}$ can be written as

\begin{eqnarray*}
    S_{it} = \sum_{\tau=1:a(\tau)=i}^{t-1} y_{a(\tau)}(\tau)(r_{a(\tau)}(\tau) - y_{a(\tau)}(\tau)^\top \eta_{a(\tau)})\mathbb{I}(a(\tau) = i)
    = B_i(t) (\widehat{\eta}_i(t) - \eta_i),
\end{eqnarray*}
we have
\begin{eqnarray*}
    \|\widehat{\eta}_i(t) - \eta_i\|_{B_i(t)} = \left\|S_{it}\right\|_{B_i(t)^{-1}}.
\end{eqnarray*}

Therefore, with probability of at least $1-\delta$, for all $t>0$, we have
\begin{eqnarray*}
        \|\widehat{\eta}_i(t) - \eta_i\|_{B_i(t)}
        \leq R \sqrt{d_y \log \left(\frac{1+L^2 n_i(t)}{\delta}\right)} + c_\eta,
\end{eqnarray*}
which is a similar result to Theorem 2 in \cite{abbasi2011improved}.
\end{proof}

Lemma \ref{lem:subg}, together with Lemma \ref{lem:eig} and \ref{lem:eta0}, provides theoretical foundations for the square-root estimation accuracy, which is showcased in Theorem \ref{thm:eta2}. The next lemma guarantees the linear growth of eigenvalues of covariance matrices $\{B_i(t)\}_{i\in [N]}$ defined in \eqref{eq:Bic} with respect to the number of samples of each arm.

\begin{lem}
Let $n_i(t)$ be the count of $i$-th arm chosen up to the time $t$. For $B_i(t)$ in \eqref{eq:Bic},  with probability at least  $1-\delta$, if $\nu_{(1)} \leq n_i(t) \leq T$, we have
\begin{eqnarray*}
\lambda_{\max}\left(B_i(t)^{-1}  \right) \leq \frac{2}{\lambda_m }n_i(t)^{-1},
\end{eqnarray*}
where $\nu_{(1)} = 8L^4\log (TN/\delta)/\lambda_m^2$.
\label{lem:eig}
\end{lem}

\begin{proof}
We investigate the (unscaled) inverse covariance matrix $B_i(t)$, whose eigenvalues are closely related to estimation accuracy. It is worth noting that the matrix $B_i(t)$ is the sum of mutually dependent rank 1 matrices. Due to the dependence of the matrices, classical techniques for independent random variables cannot be applied to them. To address this issue, we construct a martingale sequence and use the next lemma (Azuma's inequality), which provides a high probability bound for a sum of martingale sequences. 

\begin{lem}
(Azuma's Inequality) Consider the sequence $\{X_t\}_{1\leq t\leq T}$ random variables adapted to some filtration $\{\mathscr{G}_t\}_{1\leq t\leq T}$, such that $\mathbb{E}[X_t|\mathscr{G}_{t-1}] = 0$. Assume that there is a deterministic sequence $\{c_t\}_{1\leq t\leq T}$ that satisfies $X_t^2 \leq c_t^2$ , almost surely. Let $\sigma^2 = \sum_{t=1}^T c_t^2$. Then, for all $\varepsilon\geq 0$, it holds that
\begin{flalign}
\mathbb{P}\left(\sum_{t=1}^T X_t \geq \varepsilon\right) \leq  e^{-\varepsilon^2/2\sigma^2}.\nonumber
\end{flalign}
\label{lem:azu}
\end{lem}

The proof of Lemma \ref{lem:azu} is provided in the work of \cite{azuma1967weighted}. We use the above lemma and construct a martingale via its difference sequence. Then, we establish a lower bound for the smallest eigenvalue of $B_i(t)$, which we show is crucial in the analysis of the worst-case estimation error. Let the sigma-field generated by the contexts and chosen arms up to time $t$ be
$$\mathscr{G}_{t-1} = \sigma\{\{x_i(\tau)\}_{\tau\in[t],i\in[N]},\{a(\tau)\}_{\tau \in [t]}\}.$$
Consider $V_t^i=y_{a(t)}(t)y_{a(t)}(t)^{\top}\mathbb{I}(a(t)=i)$ in order to study the behavior of $B_i(t)$. Since
\begin{eqnarray*}
\mathrm{Var}(y_i(t)|\mathscr{G}_{t-1}) &=& \mathbb{E}[y_i(t)y_i(t)^\top|\mathscr{G}_{t-1}] - \mathbb{E}[y_i(t)|\mathscr{G}_{t-1}]\mathbb{E}[y_i(t)|\mathscr{G}_{t-1}]^\top\\
&=& \mathbb{E}[V_t|\mathscr{G}_{t-1}] -  Ax_i(t)x_i(t)^\top A^\top,
\end{eqnarray*}
we have 
\begin{eqnarray}
    \mathbb{E}[  V_t^i |\mathscr{G}_{t-1}] =  \left(\mathrm{Var}(y_i(t)|\mathscr{G}_{t-1})  +  Ax_i(t)x_i(t)^\top A^\top\right)\mathbb{I}(a(t)=i) \succeq  \Sigma_\xi \mathbb{I}(a(t)=i) \succeq \lambda_m I_{d_y} \mathbb{I}(a(t)=i),\label{eq:vt}
\end{eqnarray}
where $M_1 \succeq M_2$ for square matrices $M_1$ and $M_1$ represents that $M_1-M_2$ is a semi-positive definite matrix and $\lambda_m=\lambda_{\min} (\Sigma_\xi)$, i.e., for all $t>0$ and $\|z\|=1$, it holds that
\begin{eqnarray}
z^\top \left(\sum_{\tau=1}^{t-1} \mathbb{E}[V_\tau^i|\mathscr{G}_{\tau-1} ]  \right) z  \geq \lambda_m n_i(t).\label{eq:mineig}
\end{eqnarray}
Now, we focus on a high probability lower bound for the smallest eigenvalue of $B_i(t)$. To proceed, define the martingale difference $X_t^i$ and martingale $Y_t^i$ such that
\begin{eqnarray}
X_t^i &=& V_t^i - \mathbb{E}[V_t|\mathscr{G}_{t-1}]\label{eq:xti},\\
Y_t^i &=& \sum_{\tau=1}^t \left(V_\tau^i - \mathbb{E}[V_\tau|\mathscr{G}_{\tau-1}] \right).\label{eq:yti}
\end{eqnarray}
Then, $X_t^i= Y_t^i-Y_{t-1}^i$ and $\mathbb{E}\left[ X_t^i | \mathscr{G}_{t-1} \right] = 0$. Thus, $z^\top X_t^i z$ is a martingale difference sequence. Here, we are interested in the minimum eigenvalue of $\sum_{\tau=1}^{t-1} V_{\tau}^i$. Because $(z^\top X_t^i z)^2 \leq \|y_i(t)\|^4\mathbb{I}(a(t)=i) \leq L^4\mathbb{I}(a(t)=i)$ and thereby $\sum_{\tau=1}^{t-1} \left(z^\top X_{\tau}^i z\right)^2\leq n_i(t)L^4$, using Lemma \ref{lem:azu}, we get the following inequality
\begin{eqnarray*}
\mathbb{P}\left( z^\top \left(\sum_{\tau=1}^{t-1} X_{\tau}^i\right) z \leq  \varepsilon \right)
\leq \exp\left(-\frac{\varepsilon^2}{2n_i(t)L^4}\right),
\end{eqnarray*}
for $\varepsilon \leq 0$. By plugging $n_i(t)\varepsilon$ into $\varepsilon$ above, we have
\begin{eqnarray}
\mathbb{P}\left( z^\top \left(\sum_{\tau=1}^{t-1} X_{\tau}^i\right) z \leq  n_i(t) \varepsilon \right)
\leq \exp\left(-\frac{n_i(t) \varepsilon^2}{2L^4}\right)\label{eq:vineq0}
\end{eqnarray}
for $\varepsilon \leq 0$. Because
$$z^\top \left(\sum_{\tau=1}^{t-1}\left(V_\tau^i - \mathbb{E}[V_\tau^i|\mathscr{G}_{\tau-1}] \right) \right) z\leq z^\top \left(\sum_{\tau=1}^{t-1}\left(V_\tau^i - \lambda_m I_{d_y}\mathbb{I}(a(\tau)=i) \right) \right) z$$
based on \eqref{eq:vt}, we have the following inequality
\begin{eqnarray}
    \mathbb{P}\left( z^\top \left(\sum_{\tau=1}^{t-1}\left(V_\tau^i - \mathbb{E}[V_\tau^i|\mathscr{G}_{\tau-1}] \right) \right) z 
    \leq  n_i(t) \varepsilon \right)
    \geq \mathbb{P}\left( z^\top \left(\sum_{\tau=1}^{t-1}\left(V_\tau^i - \lambda_m I_{d_y}\mathbb{I}(a(\tau)=i) \right) \right) z \leq  n_i(t) \varepsilon \right).\label{eq:vineq}
\end{eqnarray}

Putting \eqref{eq:vineq0} and \eqref{eq:vineq} together,
we obtain
\begin{flalign}
\mathbb{P}\left( z^\top \left( \sum_{\tau=1}^{t-1}V_\tau^i  \right)z 
\leq   n_i(t) (\lambda_m + \varepsilon) \right)
\leq \exp\left(-\frac{n_i(t)\varepsilon^2}{2L^4}\right),
\end{flalign}
where $-\lambda_m \leq \varepsilon \leq 0$ is arbitrary. Because $z^\top B_i(t) z \geq z^\top \left( \sum_{\tau=1}^{t-1}V_\tau^i  \right)z$ based on $B_i(t)= I_{d_y}+\sum_{\tau=1}^{t-1}V_\tau^i $, we have
\begin{flalign}
\mathbb{P}\left( z^\top B_i(t) z 
\leq   n_i(t) (\lambda_m + \varepsilon) \right)  \leq  \exp\left(-\frac{n_i(t)\varepsilon^2}{2 L^4 }\right),\label{eq:eig}
\end{flalign}
for $-\lambda_m\leq \varepsilon \leq 0$. By putting $\exp\left(-n_i(t)\varepsilon^2/(2 L^4 )\right)=\delta/(TN)$, \eqref{eq:eig} can be written as
\begin{eqnarray}
z^\top B_i(t) z  \geq n_i(t) \left(\lambda_m - \sqrt{\frac{2L^4}{n_i(t)}\log\frac{TN}{\delta}} \right),\label{eq:eig2}
\end{eqnarray}
for any $z\in \mathbb{R}^{d_y}$ such that $\|z\|=1$ and all $1 \leq t \leq T$ with probability at least $1-\delta$. That is, we have
\begin{eqnarray*}
n_i(t) \left(\lambda_m - \sqrt{\frac{2L^4}{n_i(t)}\log\frac{TN}{\delta}} \right) \leq \lambda_{\min}(B_i(t)),
\end{eqnarray*}
because the inequality \eqref{eq:eig2} is achieved for any $z\in \mathbb{R}^{d_y}$. If $n_i(t) \geq \nu_{(1)} := 8L^4\log (TN/\delta)/\lambda_m^2 = \mathcal{O}(L^4 \log (TN/\delta))$, we have
\begin{eqnarray*}
\lambda_{\max}\left(   B_i(t)^{-1}  \right) \leq \frac{2}{\lambda_m }n_i(t)^{-1}.
\end{eqnarray*}
\end{proof}

In the lemma above, the minimum sampling size $\nu_{(1)}$ is required to guarantee the linear growth of the eigenvalues of $B_i(t)$ based on \eqref{eq:eig2}. The next lemma shows that the estimate in \eqref{eq:etahatc} has the square-root estimation accuracy regarding $n_i(t)$.

\begin{lem}
Let $\widehat{\eta}_i(t)$ be the estimate in \eqref{eq:etahatc}. Then, if $\nu_{(1)} < n_i(t) \leq T$, with probability at least $1-\delta$, for all $i\in [N]$, we have
\begin{flalign}
\| \widehat{\eta}_i(t) - \eta_{i} \| \leq  \sqrt{\frac{2 }{\lambda_m }} \left(  R\sqrt{d_y \log \left(\frac{1+TL^2}{\delta}\right)}+c_\eta\right) n_i(t)^{-1/2}
 .\nonumber
\end{flalign}
\label{lem:eta0}
\end{lem}

\begin{proof}
    First, it is given that 
    \begin{eqnarray*}
        \|\widehat{\eta}_i(t) - \eta_i\|_{B_i(t)} = \|B_i(t)^{1/2}(\widehat{\eta}_i(t) - \eta_i)\| \leq R \sqrt{d_y \log \left(\frac{1+TL^2}{\delta}\right)} + c_\eta
    \end{eqnarray*}
    by Lemma \ref{lem:subg}. Then, because $\sqrt{\lambda_m n_i(t)/2}\leq\lambda_{\min}(B_i(t)^{1/2})$ for $n_i(t)\geq \nu_{(1)}$  by Lemma \ref{lem:eig}, we have 
    \begin{eqnarray*}
        \sqrt{\frac{\lambda_m n_i(t)}{2}} \|\widehat{\eta}_i(t) - \eta_i\|\leq \lambda_{\min}(B_i(t)^{1/2}) \|\widehat{\eta}_i(t) - \eta_i\|\leq \|B_i(t)^{1/2}(\widehat{\eta}_i(t) - \eta_i)\|.
    \end{eqnarray*}
    Therefore, putting the two inequalities above together, we have
    \begin{flalign}
\| \widehat{\eta}_i(t) - \eta_{i} \| \leq  \sqrt{\frac{2 }{\lambda_m }} \left(  R\sqrt{d_y \log \left(\frac{1+TL^2}{\delta}\right)}+c_\eta\right) n_i(t)^{-1/2}
 ,\nonumber
\end{flalign}
if $n_i(t)\geq \nu_{(1)}$.
\end{proof}

The next lemma provides an upper bound for the norm of sample bias, $\widetilde{\eta}_i(t)-\eta_i$, which is represented as the sum of the degree of exploration $\widetilde{\eta}_i(t)-\widehat{\eta}_i(t)$ and estimation error $\widehat{\eta}_i(t)-\eta_i$. This lemma is used to find the bound for the contribution of sample bias to the regret growth. This lemma is built on the linear growth of eigenvalues of $B_i(t)$ along with the confidence ellipsoid of the estimates, $\{\widehat{\eta}_i(t)\}_{i\in[N]}$, in Lemma \ref{lem:subg}.

\begin{lem}
Consider $\widetilde{\eta}_i(t)$, a sample of the $i$-th arm in \eqref{eq:sample}. Then, if $\nu_{(1)} < n_i(t) \leq T$, with probability at least $1-\delta$, for all $i\in [N]$, we have
\begin{flalign}
\| \widetilde{\eta}_i(t) - \eta_{i} \| \leq  \sqrt{\frac{2 }{\lambda_m }} \left( v\sqrt{2d_y \log \frac{2TN}{\delta}} + R\sqrt{d_y \log \left(\frac{1+TL^2}{\delta}\right)}+c_\eta\right) n_i(t)^{-1/2}
 .\nonumber
\end{flalign}
\label{lem:eta1}
\end{lem}

\begin{proof}
First, we consider the distribution of $\widetilde{\eta}_i(t)-\widehat{\eta}_i(t)$. Note that we sample $\widetilde{\eta}_i(t)$ from $\mathcal{N} \left(\widehat{\eta}_i(t),v^2 B_i(t)^{-1}\right)$. Using $\mathbb{P}\left(\| \widetilde{\eta}_i(t) - \widehat{\eta}_i(t)\| > \epsilon |B_i(t) \right) \leq \mathbb{P}\left(\sqrt{d_y}Z  > \epsilon |B_i(t) \right)$ for $Z|B_i(t) \sim \mathcal{N} \left(0,v^2\lambda_{\max}(B_i(t)^{-1})\right)$,
    we have 
    \begin{eqnarray*}
        \mathbb{P}\left(\| \widetilde{\eta}_i(t) - \widehat{\eta}_i(t)\| > \epsilon |B_i(t) \right) <2\cdot \exp\left(-\frac{\epsilon^2}{2d_y v^2\lambda_{\max}(B_i(t)^{-1})} \right).
    \end{eqnarray*}
    
By putting $2\cdot \exp\left(-\epsilon^2/(2v^2\lambda_{\max}(B_i(t)^{-1}))\right)=\delta/(TN)$, we have
\begin{eqnarray*}
    \| \widetilde{\eta}_i(t) - \widehat{\eta}_i(t)\| < v\sqrt{2d_y \lambda_{\max}(B_i(t)^{-1}) \log \frac{2TN}{\delta}}.
\end{eqnarray*}

If $n_i(t)>\nu_{(1)}$, by Lemma \ref{lem:eig}, we have $\lambda_{\max}(B_i(t)^{-1}) \leq \sqrt{2/(\lambda_m n_i(t))}$ and subsequently

\begin{eqnarray*}
    \| \widetilde{\eta}_i(t) - \widehat{\eta}_i(t)\| <  v\sqrt{\frac{2 }{\lambda_m }} \sqrt{2d_y \log \frac{2TN}{\delta}} n_i(t)^{-1/2}.
\end{eqnarray*}

Therefore, by putting the above inequality together with Lemma \ref{lem:eta0}, for $\nu_{(1)}<n_i(t)\leq T$, we have

\begin{flalign*}
\| \widetilde{\eta}_i(t) - \eta_{i} \| \leq  \sqrt{\frac{2 }{\lambda_m }}   \left( v\sqrt{2d_y \log \frac{2TN}{\delta}} + R\sqrt{d_y \log \left(\frac{1+TL^2}{\delta}\right)}+c_\eta\right) n_i(t)^{-1/2}
 .\nonumber
\end{flalign*}    
\end{proof}


\newpage
\section{Results for the shared parameter setup}

\label{sec:rss}
In this section, we present the theoretical result of the model with a shared parameter described in Appendix \ref{sec:sps}. For this setup, we have $n_i(t) = t$ for all $i \in [N]$, which means that a decision-maker can learn the shared parameter regardless of the chosen arm. The next theorem provides a high probability regret upper bound for Thompson sampling for partially observable contextual bandits with a shared parameter. 

\begin{thm}
Assume that Algorithm \ref{algo1} is used in partially observable contextual bandits with a shared parameter. Then, the following regret bound holds with probability at least $1-\delta$:
\begin{eqnarray*}
    \mathrm{Regret}(T) =  \mathcal{O}\left(vNd_y^{2.5}\log^{3.5}\left(\frac{TNd_y}{\delta}\right)\right).
\end{eqnarray*}
\label{thm:reg1}
\end{thm}

The regret bound scales at most $\log^{3.5} T$ with respect to the time horizon and linearly with $N$. $\sqrt{d_y\log (TNd_y/\delta)}$ and $d_y^2 \log^3 (TNd_y/\delta)$ are incurred by the estimation errors and the minimum time, respectively. Lastly, $N$ is resulted by the use of the inclusion-exclusion formula to find the bound for the sum of probabilities that the optimal arms are not chosen over time.

Note that a high probability logarithmic (with respect to time) upper bound for regret for the greedy algorithm under the normality assumption has been found for the model with a shared parameter by \cite{park2022worst}. As compared to the setting in \cite{park2022worst}, the result above is constructed based on less strict assumptions, in which contexts, observation noise, and reward noise have sub-Gaussian distributions for observation noise, contexts, and reward noise.

\begin{proof} 
To begin with the analysis, as mentioned in Appendix \ref{sec:sps}, we have the following equalities in the shared parameter setting: 
$$n_i(t) = t, \eta_i = \eta_\centerdot,  \widehat{\eta}_i(t) = \widehat{\eta}_\centerdot(t), \text{ and } B_i(t) = B_\centerdot(t),$$ 
for all $i\in [N]$. So, we decompose the regret as follows:
\begin{flalign*}
    &\mathrm{Regret}(T)
    = \sum_{t=1}^T (y_{a^\star(t)}(t) - y_{a(t)}(t) )^\top \eta_\centerdot \\
    &\leq \sum_{t=1}^{\lfloor\nu_{(1)}\rfloor}2c_\eta L + \sum_{t=\lceil\nu_{(1)}\rceil}^T \left( (y_{a^\star(t)}(t) - y_{a(t)}(t) )^\top \eta_\centerdot +  (y_{a(t)}(t) -y_{a^\star(t)}(t))^\top \widetilde{\eta}_{a(t)}(t) \right) \mathbb{I}(a^\star(t)\neq a(t)).
\end{flalign*}

Now, because $\|y_i(t)\| \leq L$ for all $i\in [N]$ and $t\in[T]$, the above regret bound leads to
\begin{equation*}
    \mathrm{Regret}(T) \leq 2L \left(c_\eta\nu_{(1)} + \sum_{t=\lceil\nu_{(1)}\rceil}^T  \|\widetilde{\eta}_\star(t)-\eta_{\star}\|\mathbb{I}(a^\star(t)\neq a(t))\right).
\end{equation*}

By Lemma \ref{lem:eta1}, if $t>\nu_{(1)}$, with probability at least $1-\delta$, we have
$$\|\widetilde{\eta}_\star(t)-\eta_\centerdot\| \leq  g(\delta) t^{-1/2},$$
where 
\begin{eqnarray*}
g(\delta) &=& 2 \left( v\sqrt{2d_y \log \frac{2TN}{\delta}} + R\sqrt{d_y \log \left(\frac{1+TL^2/\lambda}{\delta}\right)}+c_\eta\right)\\
&=&\mathcal{O}\left(\sqrt{d_y \log(TNd_y/\delta)}\right).    
\end{eqnarray*}
Now, we use Azuma's inequality to find a high probability upper bound for $\sum_{t=\lceil\nu_{(1)}\rceil}^T t^{-1/2} \mathbb{I}(a^\star(t) \neq a(t))$. For that purpose, consider the martingale sequence $\sum_{\tau=1}^t (\tau^{-1/2} \mathbb{I}(a^\star(\tau) \neq a(\tau)) - \tau^{-1/2} \mathbb{P}(a^\star(\tau) \neq a(\tau)))$ with respect to a filtration $\{\sigma\{\varnothing \}\}_{\tau=1}^{t-1}$, where $\varnothing$ is the empty set. Since $t^{-1/2} \mathbb{I}(a^\star(t) \neq a(t))\leq t^{-1/2}$ and $\sum_{t=\lceil\nu_{(1)}\rceil}^T 2t^{-1} \leq 4\log T$ (assuming $\lceil\nu_{(1)}\rceil\geq 2$), we have 
\begin{flalign}
\mathbb{P}\left( \sum_{t=\lceil\nu_{(1)}\rceil}^T \frac{1}{\sqrt{t}} \mathbb{I}(a^\star(t) \neq a(t)) - \sum_{t=\lceil\nu_{(1)}\rceil}^T \frac{1}{\sqrt{t}} \mathbb{P}(a^\star(t) \neq a(t)) > \varepsilon \right) \leq \exp \left(-\frac{\varepsilon^2}{4\log T}\right).\nonumber
\end{flalign}

By putting $\delta = \exp\left(-\varepsilon^2/(4\log T)\right)$, with probability at least $1-\delta$, we have
\begin{eqnarray}
\sum_{t=\lceil\nu_{(1)}\rceil}^T \frac{1}{\sqrt{t}} \mathbb{I}(a^\star(t) \neq a(t)) \leq \sqrt{4\log T \log \delta^{-1} }+ \sum_{t=\lceil\nu_{(1)}\rceil}^T\frac{1}{\sqrt{t}} \mathbb{P}(a^\star(t) \neq a(t)).\label{eq:indi}
\end{eqnarray}


Now, we proceed towards establishing an upper bound for the second term on the right side in \eqref{eq:indi}. Denote $A_{it}^\star = \{y(t) \in A_i^\star\}$, where $A_i^\star$ is defined in Definition \ref{def:astar}. By using the fact that
\begin{flalign}
&\{y_i(t)^\top \widetilde{\eta}_i(t) < y_j(t)^\top \widetilde{\eta}_j(t)\} \nonumber\\
&\subset \left\{y_j(t)^\top(  \widetilde{\eta}_j(t) - \eta_\centerdot)  >  \frac{1}{2}(y_i(t)-y_j(t))^\top \eta_\centerdot \right\}\bigcup\left\{y_i(t)^\top  (\widetilde{\eta}_i(t) - \eta_\centerdot)  < -   \frac{1}{2}((y_i(t)-y_j(t))^\top \eta_\centerdot) \right\},\label{eq:abc2}
\end{flalign}
we get
\begin{eqnarray}
    &&\mathbb{P}(y_j(t)^\top \widetilde{\eta}_j(t)-y_i(t)^\top 
 \widetilde{\eta}_i(t)>0 |G_{t-1}^\star,A_{it}^\star)\nonumber\\
    &\leq&\mathbb{P}(y_j(t)^\top (\widetilde{\eta}_j(t)-\widehat{\eta}_\centerdot(t)) > - y_j(t)^\top(\widehat{\eta}_\centerdot(t)-\eta_\centerdot) +  0.5(y_i(t)-y_j(t))^\top \eta_\centerdot  |G_{t-1}^\star,A_{it}^\star)\nonumber\\
    &+& \mathbb{P}(y_i(t)^\top (\widetilde{\eta}_i(t)-\widehat{\eta}_\centerdot(t)) > - y_i(t)^\top(\widehat{\eta}_\centerdot(t)-\eta_\centerdot) +   0.5(y_i(t)-y_j(t))^\top \eta_\centerdot |G_{t-1}^\star,A_{it}^\star).\nonumber
\end{eqnarray}

By Lemma \ref{lem:eta0}, with probability of at least $1-\delta$, we have
\begin{eqnarray*}
    y_i(t)^\top (\widehat{\eta}_\centerdot(t)-\eta_\centerdot) \leq \frac{h(\delta,T)\|y(t)\|}{t^{1/2}},
\end{eqnarray*}
for all $\nu_{(1)} < t \leq T$ and $i\in [N]$, where 

\begin{eqnarray}
    h(\delta,T) =  \sqrt{\frac{2 }{\lambda_m  }}  \left( R\sqrt{d_y \log \left(\frac{1+TL^2}{\delta}\right)}+c_\eta\right) = \mathcal{O}\left(R\sqrt{d_y\log (TNd_y/\delta)}\right).\label{eq:h}
\end{eqnarray}

Accordingly, we have
\begin{eqnarray}
    &&\mathbb{P}\left( y_i(t)^\top \widetilde{\eta}_j(t)-y_j(t)^\top\widetilde{\eta}_i(t) >0 \Big| A_{it}^\star \right) \nonumber\\
    &\leq& \mathbb{P}\left( y_i(t)^\top (\widetilde{\eta}_i(t)-\widehat{\eta}_\centerdot(t)) > - h(\delta,T)\|y(t)\|t^{-1/2} + 0.5(y_i(t)-y_j(t))^\top \eta_\centerdot  \Big| A_{it}^\star \right) \nonumber\\
    &+& \mathbb{P}\left( y_j(t)^\top(\widetilde{\eta}_j(t)-\widehat{\eta}_\centerdot(t)) >  - h(\delta,T)\|y(t)\|t^{-1/2} +  0.5(y_i(t)-y_j(t))^\top \eta_\centerdot   \Big| A_{it}^\star \right).~~\label{eq:sumprob}
\end{eqnarray}   
Now, let $E_{ijt} = \{h(\delta,T)\|y(t)\|t^{-1/2} <  0.25(y_i(t)-y_j(t))^\top \eta_\centerdot \}$. Then, we can decompose the first term on the RHS in \eqref{eq:sumprob} as follows:
\begin{flalign}
    & \mathbb{P}\left(\left.y_i(t)^\top (\widetilde{\eta}_i(t)-\widehat{\eta}_\centerdot(t)) > - \frac{h(\delta,T)\|y(t)\|}{t^{1/2}} +  (y_i(t)-y_j(t))^\top \eta_\centerdot   \right|A_{it}^\star\right) \nonumber\\
    &= \mathbb{P}\left(\left.y_i(t)^\top (\widetilde{\eta}_i(t)-\widehat{\eta}_\centerdot(t)) > -\frac{h(\delta,T)\|y(t)\|}{t^{1/2}} +  0.5(y_i(t)-y_j(t))^\top \eta_\centerdot   \right|E_{ijt},A_{it}^\star\right)\mathbb{P}(E_{ijt}|A_{it}^\star) \nonumber\\
    &+ \mathbb{P}\left(\left.y_i(t)^\top (\widetilde{\eta}_i(t)-\widehat{\eta}_\centerdot(t)) > - \frac{h(\delta,T)\|y(t)\|}{t^{1/2}} +  0.5(y_i(t)-y_j(t))^\top \eta_\centerdot   \right| E_{ijt}^c,A_{it}^\star\right)\mathbb{P}(E_{ijt}^c|A_{it}^\star).\label{eq:probrhs}
\end{flalign}

We aim to show that the above probability is $\mathcal{O}(t^{-0.5})$ by showing each term in the RHS of \eqref{eq:probrhs} is $\mathcal{O}(t^{-0.5})$. By Assumption \ref{ass:mar}, if $t > \nu_{(1)}$, we have
\begin{eqnarray}
    \mathbb{P}(E_{ijt}^c |A_{it}^\star) =\mathbb{P}\left(\left.4h(\delta,T)t^{-1/2} > (y_i(t)-y_j(t))^\top \eta_\centerdot/\|y(t)\| \right|A_{it}^\star\right)
    \leq \frac{4h(\delta,T)C}{t^{1/2}}.\label{eq:rhs2}
\end{eqnarray}
Thus, we showed that the second term in \eqref{eq:probrhs} is $\mathcal{O}(t^{-0.5})$.  Now, we aim to show that the first term in \eqref{eq:probrhs} is $\mathcal{O}(t^{-0.5})$. Note that
\begin{eqnarray}
    &&\mathbb{P}\left(\left.\dot{y}_i(t)^\top (\widetilde{\eta}_i(t)-\widehat{\eta}_\centerdot(t)) > -\frac{h(\delta,T)}{t^{1/2}} +  0.5(\dot{y}_i(t)-\dot{y}_j(t))^\top \eta_\centerdot   \right|E_{ijt},A_{it}^\star\right)\nonumber\\
    &\leq& \mathbb{P}\left(\left.\dot{y}_i(t)^\top (\widetilde{\eta}_i(t)-\widehat{\eta}_\centerdot(t)) >   0.25(\dot{y}_i(t)-\dot{y}_j(t))^\top \eta_\centerdot   \right|A_{it}^\star\right),\label{eq:rhs1}
\end{eqnarray}
where $\dot{y}_i(t) =y_i(t)/\|y(t)\|$. Now it suffices to show that the first term in the RHS of the inequality above is $\mathcal{O}(t^{-1/2})$. Using $\dot{y}_i(t)^\top (\widetilde{\eta}_i(t)-\widehat{\eta}_\centerdot(t)) \sim \mathcal{N}(0,v^2 \dot{y}_i(t)^\top B_\centerdot(t)^{-1}\dot{y}_i(t))$ given $y(t)$ and $\mathscr{G}^\star_{t-1}$ and $\lambda_{\min}(B_\centerdot(t)^{-1}) \leq 2/(\lambda_m t)$ by Lemma \ref{lem:eig}, we have
\begin{eqnarray*}
    \mathbb{P}(\dot{y}_i(t)^\top (\widetilde{\eta}_i(t)-\widehat{\eta}_\centerdot(t)) >   0.25(\dot{y}_i(t)-\dot{y}_j(t))^\top \eta_\centerdot  |y(t),A_{it}^\star) \leq \exp\left(-\frac{t \lambda_m ((\dot{y}_i(t)-\dot{y}_j(t))^\top \eta_\centerdot)^2 }{64v^2 }\right).
\end{eqnarray*}
Thus, the first term on the RHS of the above inequality can be written as
\begin{eqnarray*}
    &&\mathbb{P}(\dot{y}_i(t)^\top (\widetilde{\eta}_i(t)-\widehat{\eta}_\centerdot(t)) >   0.25(\dot{y}_i(t)-\dot{y}_j(t))^\top \eta_\centerdot  |A_{it}^\star) \\
    &=& \mathbb{E}[\mathbb{P}(\dot{y}_i(t)^\top (\widetilde{\eta}_i(t)-\widehat{\eta}_\centerdot(t)) >   0.25(\dot{y}_i(t)-\dot{y}_j(t))^\top \eta_\centerdot  |y(t),A_{it}^\star)|A_{it}^\star]\\
    &\leq& \mathbb{E}\left[\left.\exp\left(-\frac{t \lambda_m ((\dot{y}_i(t)-\dot{y}_j(t))^\top \eta_\centerdot)^2 }{64v^2 }\right) \right|A_{it}^\star\right].
\end{eqnarray*}
By integration by part, we have
\begin{eqnarray*}
    &&\mathbb{E}\left[\left.\exp\left(-\frac{t \lambda_m ((\dot{y}_i(t)-\dot{y}_j(t))^\top \eta_\centerdot)^2 }{64v^2 }\right) \right|A_{it}^\star\right] \\
    &=&\int_{0}^\infty \frac{2 t \lambda_m u}{64 v^2  } \exp\left(-\frac{t \lambda_mu^2}{64 v^2  }\right) \mathbb{P}((\dot{y}_i(t)-\dot{y}_j(t))^\top \eta_\centerdot < u|A_{it}^\star)  du.
\end{eqnarray*}

Since $\mathbb{P}((\dot{y}_i(t)-\dot{y}_j(t))^\top \eta_\centerdot < u|A_{it}^\star) \leq Cu$ for $C>0$ based on Assumption \ref{ass:mar}, the term above can be written as
    
    \begin{eqnarray*}
    &&\mathbb{E}\left[\left.\exp\left(-\frac{t \lambda_m ((\dot{y}_i(t)-\dot{y}_j(t))^\top \eta_\centerdot)^2 }{64v^2 }\right) \right|A_{it}^\star\right] \\
    &\leq& \frac{2\sqrt{\pi}}{\sqrt{64 v^2 /(\lambda_m t)} } \int_{0}^\infty \frac{2 t \lambda_m  u}{\sqrt{64\pi v^2 /(\lambda_m t)} } \exp\left(-\frac{u^2}{64 v^2 /(t \lambda_m) }\right) Cu du = 8vC\sqrt{\frac{\pi}{\lambda_m t}},
\end{eqnarray*}
where we used the following result about one-sided Gaussian integrals $$\int_0^{\infty}x^2 \frac{1}{\sqrt{2\pi\sigma^2}}e^{-\frac{x^2}{2\sigma^2}}dx =\sigma^2/2.$$

Thus, combining \eqref{eq:rhs2} and \eqref{eq:rhs1}, we get
\begin{eqnarray}
    &&\mathbb{P}(\dot{y}_i(t)^\top (\widetilde{\eta}_i(t)-\widehat{\eta}_\centerdot(t)) > - \dot{y}_i(t)^\top(\widehat{\eta}_\centerdot(t)-\eta_\centerdot) +  0.5(\dot{y}_i(t)-\dot{y}_j(t))^\top \eta_\centerdot |A_{it}^\star) \nonumber\\
    &\leq& Ct^{-1/2}\left(8v\sqrt{ \frac{\pi}{\lambda_m}} + 4h(\delta,T)\right).\label{eq:prob1}
\end{eqnarray}

Similarly, we have
\begin{eqnarray}
    &&\mathbb{P}(\dot{y}_j(t)^\top (\widetilde{\eta}_j(t)-\widehat{\eta}_\centerdot(t)) > - \dot{y}_j(t)^\top(\widehat{\eta}_\centerdot(t)-\eta_\centerdot) +  0.5(\dot{y}_i(t)-\dot{y}_j(t))^\top \eta_\centerdot  |A_{it}^\star) \nonumber\\
    &\leq& Ct^{-1/2} \left(8v\sqrt{ \frac{\pi}{\lambda_m }} + 4h(\delta,T)\right).\label{eq:prob2}
\end{eqnarray}

Using \eqref{eq:sumprob}, we have
\begin{eqnarray}
    &&\mathbb{P}(\dot{y}_j(t)^\top \widetilde{\eta}_j(t)-\dot{y}_i(t)^\top \widetilde{\eta}_i(t))>0 |A_{it}^\star)\nonumber\\ 
    &\leq& L Ct^{-1/2} \left(v\left(8\sqrt{ \frac{\pi}{\lambda_m}}+ 8\sqrt{ \frac{\pi}{\lambda_m }}\right) + 4h(\delta,T)+4h(\delta,T)\right). \label{eq:probij}
\end{eqnarray}

By summing the probabilities in \eqref{eq:probij} over $i,j \in [N]$, we have
\begin{eqnarray}
    \sum_{i=1}^N \sum_{j=1}^N \mathbb{P}(\dot{y}_j(t)^\top \widetilde{\eta}_j(t)-\dot{y}_i(t)^\top \widetilde{\eta}_i(t) > 0| A_{it}^\star)\mathbb{P}(A_{it}^\star)
    &\leq&  \frac{2C}{\sqrt{t}}  \sum_{i=1}^N \sum_{j=1}^N \mathbb{P}(A_{it}^\star) \left(8v\sqrt{ \frac{\pi}{\lambda_m}} + 4h(\delta,T)\right)\nonumber\\
   &\leq&  \frac{2c_M(\delta,T)CN}{\sqrt{t}},\label{eq:probijsum}
\end{eqnarray}
where 
\begin{eqnarray}
    c_M(\delta,T) = 8v\sqrt{ \pi/\lambda_m}  + 4h(\delta,T) = \mathcal{O}\left(v\sqrt{d_y \log (Td_y/\delta)}\right).\label{eq:cm}
\end{eqnarray}

Note that by using the inclusion-exclusion formula, we can bound the probability of pulling a suboptimal arm as follows
\begin{eqnarray}
    \mathbb{P}(a^\star(t) \neq a(t))\leq \sum_{i=1}^N \sum_{j=1}^N \mathbb{P}(y(t)^\top (\widetilde{\eta}_j(t)-\widetilde{\eta}_i(t)) > 0| A_{it}^\star)\mathbb{P}(A_{it}^\star).\label{eq:probijsum2}
\end{eqnarray}
Putting \eqref{eq:probijsum}, \eqref{eq:probijsum2}, and the minimal sample size $\nu_{(1)}$ together, we obtain the following inequality
\begin{eqnarray*}
    \sum_{t=\lceil\nu_{(1)}\rceil}^T \frac{1}{\sqrt{t}}\mathbb{P}(a^\star(t) \neq a(t)) 
 &\leq& 2c_M(\delta,T)CN \log T.
\end{eqnarray*}
Then, according to \eqref{eq:indi}, with probability at least $1-\delta$, it holds that $$\sum_{t=\lceil\nu_{(1)}\rceil}^T \frac{1}{\sqrt{t}} \mathbb{I}(a^\star(t) \neq a(t))\leq \sqrt{4\log T\log \delta^{-1}}  + 2c_M(\delta,T)CN \log T. $$

Therefore, using $\nu_{(1)} = 8L^4 \log(T/\delta)/\lambda_m^2$ and $L = c_y \sqrt{2 d_y\log (TNd_y/\delta)}$, with probability at least $1-\delta$, the regret bound below holds true 
\begin{eqnarray*}
    \mathrm{Regret}(T) 
    &\leq&  2c_\eta L \nu_{(1)} + + L  g(\delta) \left( \sqrt{4\log T\log \delta^{-1}}  + 2c_M(\delta,T)CN \log T\right)\\
    &=& \mathcal{O}\left(vNd_y^{2.5}\log^{3.5}\left(\frac{TNd_y}{\delta}\right)\right).
\end{eqnarray*}
\end{proof}

\newpage

\newpage
\section{Proof of Theorem \ref{thm:eta2}}
\label{sec:pthm:eta2}


Before starting the proof, we remind the constants described in the statement in Theorem \ref{thm:eta2}. $L$ is the bound for the $\ell_2$-norm of observations. $p_i$ is the probability of optimality of the $i$-th arm, as defined in Definition \ref{def:astar}.  $\kappa$ is the minimum value of suboptimality gap with a positive probability (0.5) defined in \eqref{eq:kappa}. First, we show that the number of selections of each arm scales linearly with a high probability. We utilize the inequality below to find a high probability upper bound for $n_i(t)$.

\begin{eqnarray*}
    n_i(t) \geq \sum_{\tau=1}^t \mathbb{I}(a(\tau)=i,A_{i\tau}^\kappa).
\end{eqnarray*}
We construct a martingale sequence $\mathbb{I}(a(t)=i,A_{it}^\kappa) - \mathbb{P}(a(\tau)=i,A_{it}^\kappa|G_{t-1}^\star)$ with respect to a filtration $\{G_t^\star\}_{t=1}^{\infty}$, where $G_t^\star = \sigma\{\{a(\tau)\}_{\tau=1}^t\}$.
By Azuma's inequality, we have
\begin{eqnarray}
    \sum_{\tau=1}^t \mathbb{I}(a(t)=i,A_{it}^\kappa) \geq - \sqrt{2t\log \delta^{-1}} +  \sum_{\tau=1}^t \mathbb{P}(a(\tau)=i||G_{\tau-1}^\star,A_{i\tau}^\kappa)\mathbb{P}(A_{i\tau}^\kappa).\label{eq:sumi}
\end{eqnarray}

Since $\mathbb{P}(a(t)=i|G_{t-1}^\star, A_{it}^\kappa)$ can be written as $\mathbb{P}(a(t)=i|G_{t-1}^\star, A_{it}^\kappa) =  1- \sum_{j \neq i}\mathbb{P}(a(t)=j|G_{t-1}^\star, A_{it}^\kappa)$, we focus on an upper bound for $\sum_{\tau=1}^t\sum_{j \neq i}\mathbb{P}(a(\tau)=j|G_{\tau-1}^\star,A_{i\tau}^\kappa)$. To proceed, we rewrite the probability $\mathbb{P}(a(t)=j|G_{t-1}^\star, A_{it}^\kappa)$ as follows:
\begin{flalign}
    &\mathbb{P}(a(t)=j|G_{t-1}^\star, A_{it}^\kappa) \nonumber\\
    &= \mathbb{P}(a(t)=j,E_{jt}^1,E_{jt}^2|G_{t-1}^\star, A_{it}^\kappa) +  \mathbb{P}(a(t)=j,(E_{jt}^1)^c,E_{jt}^2|G_{t-1}^\star, A_{it}^\kappa)+\mathbb{P}(a(t)=j,(E_{jt}^2)^c|G_{t-1}^\star, A_{it}^\kappa),\label{eq:decomp}
\end{flalign}

where $E_{jt}^1 = \{y_j(t)^\top\widetilde{\eta}_j(t) < y_j(t)^\top\eta_j + 0.5( y_i(t)^\top \eta_i-y_j(t)^\top \eta_j)\}$ and $E_{jt}^2 = \{y_j(t)^\top\widehat{\eta}_j(t) \leq y_j(t)^\top\eta_j+0.5( y_i(t)^\top \eta_i-y_j(t)^\top \eta_j)\}$. Based on the decomposition above, we will show the upper bound for $\sum_{\tau=1}^t\mathbb{P}(a(\tau)=j|A_{i\tau}^\kappa,F_{\tau-1}^\star) $ by establishing upper bounds of the above three terms in Lemmas \ref{lem:term1}, \ref{lem:term2}, and \ref{lem:term3}. Moving forward, we will find an upper bound for the first term in \eqref{eq:decomp}.

\begin{lem}
    For all $1\leq t\leq T$ and instantiations of $F_{t-1}^\star = \sigma\{\{y(\tau)\}_{\tau=1}^t,\{a(\tau)\}_{\tau=1}^{t-1},\{r_{a(\tau)}(\tau)\}_{\tau=1}^{t-1}\}$, we have
    \begin{eqnarray*}
        \mathbb{P}(a(t)=j,E_{jt}^1,E_{jt}^2|A_{it}^\kappa,F_{t-1}^\star) \leq \frac{1-p_{ijt}}{p_{ijt}} \mathbb{P}(a(t)=i,E_{jt}^1,E_{jt}^2|A_{it}^\kappa,F_{t-1}^\star),
    \end{eqnarray*}
where $p_{ijt} = \mathbb{P}(y_i(t)^\top \widetilde{\eta}_i(t)>0.5( y_j(t)^\top \eta_j+y_i(t)^\top \eta_i)|A_{it}^\kappa,F_{t-1}^\star)$.
\label{lem:inq}
\end{lem}

\begin{proof}
We consider upper and lower bounds of the probabilities $\mathbb{P}\left(a(t)=j | A_{it}^\kappa,E_{jt}^1, F_{t-1}^\star\right)$ and $\mathbb{P}\left(a(t)=i | A_{it}^\kappa, E_{jt}^1, F_{t-1}^\star\right)$, respectively. First, we aim to find an upper bound for $\mathbb{P}\left(a(t)=j | A_{it}^\kappa,E_{jt}^1, F_{t-1}^\star\right)$. Given $E_{jt}^1$, if arm $j$ is selected, $y_k(t)^\top \widetilde{\eta}_k(t) \leq 0.5( y_j(t)^\top \eta_j+y_i(t)^\top \eta_i)$ for all $k$ including $j$. Using this fact, we get
\begin{eqnarray*}
\mathbb{P}\left(a(t)=j | A_{it}^\kappa,E_{jt}^1, F_{t-1}^\star\right) 
\leq  \mathbb{P}\left(y_k(t)^\top \widetilde{\eta}_k(t) \leq 0.5( y_j(t)^\top \eta_j+y_i(t)^\top \eta_i), \forall k | A_{it}^\kappa,E_{jt}^1, F_{t-1}^\star\right).
\end{eqnarray*}

Since the sample of each arm is generated independently given $F_{t-1}^\star$, the term on the RHS above can be written as
\begin{flalign}
&\mathbb{P}\left(y_k(t)^\top \widetilde{\eta}_k(t) \leq 0.5( y_j(t)^\top \eta_j+y_i(t)^\top \eta_i), \forall k \neq i | A_{it}^\kappa,E_{jt}^1, F_{t-1}^\star\right)\nonumber\\
=~&\left(1-p_{ijt}\right) \cdot \mathbb{P}\left(y_k(t)^\top \widetilde{\eta}_k(t) \leq 0.5( y_j(t)^\top \eta_j+y_i(t)^\top \eta_i), \forall k \neq i | A_{it}^\kappa, E_{jt}^1, F_{t-1}^\star\right).\label{eq:pinq1}
\end{flalign}

Similarly, we have an upper bound for $\mathbb{P}\left(a(t)=i | A_{it}^\kappa, E_{jt}^1, F_{t-1}^\star\right)$ as follows.
\begin{flalign}
&\mathbb{P}\left(a(t)=i | A_{it}^\kappa, E_{jt}^1, F_{t-1}^\star\right)\nonumber\\ 
\geq~&\mathbb{P}\left(y_i(t)^\top \widetilde{\eta}_i(t)  >0.5( y_j(t)^\top \eta_j+y_i(t)^\top \eta_i) \geq y_k(t)^\top \widetilde{\eta}_k(t) , \forall k \neq i | A_{it}^\kappa, E_{jt}^1, F_{t-1}^\star\right) \nonumber\\
=~&  p_{ijt} \cdot \mathbb{P}\left(y_j(t)^\top \widetilde{\eta}_k(t)  \leq 0.5( y_j(t)^\top \eta_j+y_i(t)^\top \eta_i), \forall k \neq i| A_{it}^\kappa, E_{jt}^1, F_{t-1}^\star\right)\label{eq:pinq2}.
\end{flalign}

Putting the two inequalities \eqref{eq:pinq1} and \eqref{eq:pinq2} together, we have
\begin{eqnarray*}
    \mathbb{P}\left(a(t)=j | A_{it}^\kappa, E_{jt}^1, F_{t-1}^\star\right) \leq \frac{1-p_{ijt}}{p_{ijt}} \mathbb{P}\left(a(t)=i | A_{it}^\kappa, E_{jt}^1, F_{t-1}^\star\right).
\end{eqnarray*}
Since whether $E_{jt}^2$ is true is determined by $F_{t-1}^\star$, we get
\begin{eqnarray*}
        \mathbb{P}(a(t)=j,E_{jt}^1,E_{jt}^2|A_{it}^\kappa,F_{t-1}^\star) \leq \frac{1-p_{ijt}}{p_{ijt}} \mathbb{P}(a(t)=i,E_{jt}^1,E_{jt}^2|A_{it}^\kappa,F_{t-1}^\star).
\end{eqnarray*}
\end{proof}

By Lemma \ref{lem:inq}, we have

\begin{eqnarray*}
    \sum_{t=1}^T \mathbb{P}(a(t)=j,E_{jt}^1,E_{jt}^2|G_{t-1}^\star,A_{it}^\kappa) &=& \sum_{t=1}^T \mathbb{E}[\mathbb{P}(a(t)=j,E_{jt}^1,E_{jt}^2|A_{it}^\kappa,F_{t-1}^\star)|G_{t-1}^\star,A_{it}^\kappa]\\
    &\leq&\sum_{t=1}^T \mathbb{E}\left[\left.\frac{1-p_{ijt}}{p_{ijt}}\mathbb{P}(a(t)=i,E_{jt}^1,E_{jt}^2 |A_{it}^\kappa,F_{t-1}^\star)\right|G_{t-1}^\star,A_{it}^\kappa\right].
\end{eqnarray*}
By simple calculation, we get
\begin{eqnarray*}
    &&\sum_{t=1}^T \mathbb{E}\left[\left.\frac{1-p_{ijt}}{p_{ijt}}\mathbb{P}(a(t)=i,E_{jt}^1,E_{jt}^2 |A_{it}^\kappa,F_{t-1}^\star)\right|G_{t-1}^\star,A_{it}^\kappa\right]\\
    &=&\sum_{t=1}^T \mathbb{E}\left[\left.\mathbb{E}\left[\left.\frac{1-p_{ijt}}{p_{ijt}}\mathbb{I}(a(t)=i,E_{jt}^1,E_{jt}^2)\right|A_{it}^\kappa,F_{t-1}^\star\right]\right|G_{t-1}^\star,A_{it}^\kappa\right]\\
    &=&\sum_{t=1}^T \mathbb{E}\left[\left.\frac{1-p_{ijt}}{p_{ijt}}\mathbb{I}(a(t)=i,E_{jt}^1,E_{jt}^2)\right|G_{t-1}^\star,A_{it}^\kappa\right]\\
    &\leq& \sum_{t=1}^T\mathbb{E}\left[\left.\frac{1-p_{ijt}}{p_{ijt}}\right|G_{t-1}^\star,A_{it}^\kappa\right].
\end{eqnarray*}

Thus, we have
\begin{eqnarray}
    \sum_{t=1}^T \mathbb{P}(a(t)=j,E_{jt}^1,E_{jt}^2|G_{t-1}^\star,A_{it}^\kappa) \leq \sum_{t=1}^T\mathbb{E}\left[\left.\frac{1-p_{ijt}}{p_{ijt}}\right|G_{t-1}^\star,A_{it}^\kappa\right].\label{eq:1pp}
\end{eqnarray}

The next lemma provides an lower and upper bound for probabilities about normal distribution, which will be used to find a lower bound for $p_{ijt}$.

\begin{lem}
   For a Gaussian distributed random variable $Z$ with mean $m$ and variance $\sigma^2$, for any $z \geq 1$,
$$
\frac{1}{2 \sqrt{\pi} z} e^{-z^2 / 2} \leq \mathbb{P}(|Z-m|>z \sigma) \leq \frac{1}{\sqrt{\pi} z} e^{-z^2 / 2} .
$$\label{lem:nineq}
\end{lem}

The lemma above is Lemma 5 in \cite{agrawal2013thompson}, which can be derived from Formula 7.1.13 in \cite{abramowitz1948handbook}. The next lemma suggests an upper bound for the term on RHS in \eqref{eq:1pp}. 

\begin{lem}
For $p_{ijt}$ defined in Lemma \ref{lem:inq} and $A_{it}^\kappa$ in \eqref{eq:kappa}, we have \begin{eqnarray*}
      \mathbb{E}\left[\left.\frac{1-p_{ijt}}{p_{ijt}}\right|G_{t-1}^\star,A_{it}^\kappa\right] \leq \frac{2\sqrt{\pi}}{v}\left( \frac{16v^4}{\kappa^3} + \frac{1}{2} c_{\eta}\sqrt{1+L^2 n_i(t)} \right) + 63.
\end{eqnarray*}
\end{lem}

\begin{proof}
First, we rewrite $p_{ijt}$ as follows:
\begin{eqnarray*}
    p_{ijt} = \mathbb{P}(y_i(t)^\top (\widetilde{\eta}_i(t) -\widehat{\eta}_i(t))> y_i(t)^\top(\eta_i-\widehat{\eta}_i(t))+0.5( y_j(t)^\top \eta_j-y_i(t)^\top \eta_i) |A_{it}^\kappa,F_{t-1}^\star).   
\end{eqnarray*}


Let $\theta_{it} = y_i(t)^\top (\eta_i-\widehat{\eta}_i(t))$. Note that $y_i(t)^\top(\widetilde{\eta}_i(t) -\widehat{\eta}_i(t)) \sim \mathcal{N}(0,v^2 y_i(t)^\top B_i(t)^{-1}y_i(t))$ given $y_i(t)$ and $B_i(t)$. Let $\Phi_{it}(\cdot)$ be the CDF of the normal distribution with mean 0 and variance $v^2 y_i(t)^\top B_i(t)^{-1}y_i(t)$. For ease of presentation, let $\kappa_{ijt} =  -0.5(y_j(t)^\top \eta_j-y_i(t)^\top \eta_i)$ and $\sigma_{it}^2 = y_i(t)^\top B_i(t)^{-1}y_i(t)$. Then, we write 
$\mathbb{P}(y_i(t)^\top (\widetilde{\eta}_i(t) -\widehat{\eta}_i(t)) > \theta_{it}-\kappa_{ijt}|A_{it}^\kappa,F_{t-1}^\star) = \Phi_{it}(-\theta+\kappa_{ijt})$. Since $\Phi_{it}(-\theta+\kappa_{ijt}) \geq \left(v\sigma_{it}/2\sqrt{\pi}(\theta-\kappa_{ijt})\right) \exp\left(-(\theta-\kappa_{ijt})^2/(2v^2\sigma_{it}^2)\right)$ for $\theta > \kappa_{ijt}+v\sigma_{it}$ by Lemma \ref{lem:nineq}, we have
\begin{eqnarray*}
    &&p_{ijt}^{-1} = \Phi_{it}(-\theta+\kappa_{ijt})^{-1} \\
    &\leq& \mathbb{I}(\kappa_{ijt}+v\sigma_{it} < \theta)\frac{2\sqrt{\pi}(\theta-\kappa_{ijt})}{v\sigma_{it}} \exp\left(\frac{(\theta-\kappa_{ijt})^2}{2v^2\sigma_{it}^2}\right) + \mathbb{I}(\kappa_{ijt}+v\sigma_{it} \geq \theta)\Phi_{it}(v\sigma_{it})^{-1}.
\end{eqnarray*}

Thus, we get

\begin{eqnarray}
    \mathbb{E}[ p_{ijt}^{-1}|A_{it}^\kappa,F_{t-1}^\star] \leq \int_{\kappa_{ijt}+v\sigma_{it}}^\infty \frac{2\sqrt{\pi}}{v\sigma_{it}}\left(1+\frac{(\theta-\kappa_{ijt})^2}{v^2\sigma_{it}^2}\right)\exp\left(\frac{(\theta-\kappa_{ijt})^2}{2v^2\sigma_{it}^2}\right)S_{it}(\theta)d\theta + \Phi(-1)^{-1},\label{eq:phiinv}
\end{eqnarray}

where $\Phi(\cdot)$ is the CDF of standard normal distribution and $S_{it}$ is the survival function of $\theta_{it}$. Note that $$y_i(t)^\top (\widehat{\eta}_i(t)-\eta_i)=y_i(t)^\top B_i(t)^{-1}\sum_{\tau=1:a(\tau)=i}^t y_a(\tau)(\tau)(r_a(\tau)(t)-y_{a(\tau)}(\tau)^\top \eta_{a(\tau)})$$ and
\begin{eqnarray*}
&&\mathrm{Var}(y_i(t)^\top (\widehat{\eta}_i(t)-\eta_i)|\{y(\tau)\}_{1\leq \tau\leq t},\{a(\tau)\}_{1\leq \tau\leq t-1}) \\
&=& y_i(t)^\top B_i(t)^{-1}y_i(t) \mathrm{Var}(r_{a(t)}(t)-y_{a(t)}(t)^\top \eta_{a(t)}|\{y(\tau)\}_{1\leq \tau\leq t},\{a(\tau)\}_{1\leq \tau\leq t-1})\\
&\leq&   y_i(t)^\top B_i(t)^{-1}y_i(t) R^2.    
\end{eqnarray*}

Since $r_a(\tau)(t)-y_{a(\tau)}(\tau)^\top \eta_{a(\tau)}$  is R-sub-Gaussian by Lemma \ref{lem:subg} and $v^2 \geq R^2$, we have
\begin{flalign}
    S_{\theta_{it}}(\theta) = \mathbb{P}(y_i(t)^\top (\widehat{\eta}_i(t)-\eta_i) >\theta |\{y(\tau)\}_{1\leq \tau\leq t},\{a(\tau)\}_{1\leq \tau\leq t-1}) \leq \exp\left(-\frac{\theta^2}{ 2y_i(t)^\top B_i(t)^{-1}y_i(t) v^2}\right),\label{eq:cdfhat}
\end{flalign}
for $\theta > 0$. Then, we have

\begin{eqnarray*}
    &&\int_{\kappa_{ijt}+v\sigma_{it}}^\infty \frac{2\sqrt{\pi}}{v\sigma_{it}}\left(1+\frac{(\theta-\kappa_{ijt})^2}{v^2\sigma_{it}^2}\right)\exp\left(\frac{(\theta-\kappa_{ijt})^2}{2v^2\sigma_{it}^2}\right)S_{it}(\theta)d\theta\\
    &\leq&  \int_{\kappa_{ijt}+v\sigma_{it}}^\infty \frac{2\sqrt{\pi}}{v\sigma_{it}}\left(1+\frac{(\theta-\kappa_{ijt})^2}{v^2\sigma_{it}^2}\right)\exp\left(\frac{(\theta-\kappa_{ijt})^2}{2v^2\sigma_{it}^2}\right) \exp\left(-\frac{\theta^2}{  2v^2\sigma_{it}^2 }\right) d\theta\\
    &=&  \int_{\kappa_{ijt}+v\sigma_{it}}^\infty \frac{2\sqrt{\pi}}{v\sigma_{it}}\left(1+\frac{(\theta-\kappa_{ijt})^2}{v^2\sigma_{it}^2}\right)\exp\left(\frac{-2\kappa_{ijt}\theta+\kappa_{ijt}^2}{2v^2\sigma_{it}^2}\right)  d\theta
\end{eqnarray*}

Using the first and second moments of the shifted exponential distribution with the scale parameter $v^2\sigma_{it}^2/\kappa_{ijt}$ and location parameter $\kappa_{ijt}/2$, we have

\begin{eqnarray}
    &&\frac{2\sqrt{\pi}}{v\sigma_{it}}\int_{\kappa_{ijt}+v\sigma_{it}}^\infty \left(1+\frac{(\theta-\kappa_{ijt})^2}{v^2\sigma_{it}^2}\right)\exp\left(\frac{-2\kappa_{ijt}(\theta - (\kappa_{ijt}/2))}{2v^2\sigma_{it}^2}\right)d\theta\nonumber\\
    &\leq& \frac{2\sqrt{\pi}v\sigma_{it}}{\kappa_{ijt}}\int_{\kappa_{ijt}/2}^\infty \left(\frac{(\theta-\kappa_{ijt}/2)^2-\kappa_{ijt}(\theta-\kappa_{ijt}/2)+\kappa_{ijt}^2/4+v^2\sigma_{it}^2}{v^2\sigma_{it}^2}\right) \frac{1}{v^2\sigma_{it}^2/\kappa_{ijt}} \exp\left(\frac{-(\theta - (\kappa_{ijt}/2))}{v^2\sigma_{it}^2/\kappa_{ijt}}\right)d\theta\nonumber\\
    &=&\frac{2\sqrt{\pi}}{\kappa_{ijt}v\sigma_{it}} (2(v^2\sigma_{it}^2/\kappa_{ijt})^2 - \kappa_{ijt}(v^2\sigma_{it}^2/\kappa_{ijt}) + \kappa_{ijt}^2/4 + v^2\sigma_{it}^2).\label{eq:terms}
\end{eqnarray}

Note that $\sigma_{it}^2 = y_i(t)^\top B_i(t)^{-1}y_i(t)$ and $\kappa_{ijt} = 0.5(y_i(t)^\top \eta_i - y_j(t)^\top \eta_j)$ given $A_{it}^\kappa$. The term $\sigma_{it}^3/\kappa_{ijt}^3$ can be bounded as follows:

\begin{eqnarray}
    \frac{\sigma_{it}^3}{\kappa_{ijt}^3} = \frac{(y_i(t)^\top B_i(t)^{-1}y_i(t))^{3/2}/\|y(t)\|^3}{(0.5(y_i(t)^\top \eta_i - y_j(t)^\top \eta_j))^3/\|y(t)\|^3} \leq \frac{8}{\kappa^3},\label{eq:term1}
\end{eqnarray}

because $\lambda_{\max}(B_i(t)^{-1}) \leq 1$ and $(y_i(t)^\top \eta_i - y_j(t)^\top \eta_j)/\|y(t)\| \geq \kappa$ given $A_{it}^\kappa$ by Assumption \ref{ass:mar}. In addition, because $1/(1+L^2n_i(t)) \leq \lambda_{\min}(B_i(t)^{-1})$ for all $t$ and $\|\eta_i\|\leq c_\eta$ for all $i$ by Assumption \ref{ass:para}, we have

\begin{eqnarray}
    \frac{\kappa_{ijt}}{\sigma_{it}} = \frac{(0.5(y_i(t)^\top \eta_i - y_j(t)^\top \eta_j))/\|y(t)\|}{(y_i(t)^\top B_i(t)^{-1}y_i(t))^{0.5}/\|y(t)\|} \leq 2c_{\eta}\sqrt{1+L^2 n_i(t)}.\label{eq:term2}
\end{eqnarray}

Then, putting \eqref{eq:terms} together with \eqref{eq:term1} and \eqref{eq:term2}, we have
\begin{eqnarray*}
    &&\frac{2\sqrt{\pi}}{\kappa_{ijt}v\sigma_{it}} (2(v^2\sigma_{it}^2/\kappa_{ijt})^2 - \kappa_{ijt}(v^2\sigma_{it}^2/\kappa_{ijt}) + \kappa_{ijt}^2/4 + v^2\sigma_{it}^2) \\
    &=& \frac{2\sqrt{\pi}}{v} \left(2v^4 \frac{\sigma_{it}^3}{\kappa_{ijt}^3}  + \frac{\kappa_{ijt}}{4\sigma_{it}}\right) \leq \frac{2\sqrt{\pi}}{v}\left( \frac{16v^4}{\kappa^3} + \frac{1}{2} c_{\eta}\sqrt{1+L^2 n_i(t)} \right).
\end{eqnarray*}

By \eqref{eq:phiinv} and the intermediate result above, we get

\begin{eqnarray}
    \mathbb{E}\left[\left. \frac{1-p_{ijt}}{p_{ijt}}\right|G_{t-1}^\star,A_{it}^\kappa\right] \leq \frac{2\sqrt{\pi}}{v}\left( \frac{16v^4}{\kappa^3} + \frac{1}{2} c_{\eta}\sqrt{1+L^2 n_i(t)} \right) + 63,\label{eq:phiinv2}
\end{eqnarray}
because $\Phi(-1) >  1/8$.
\end{proof}

\begin{lem}
   For the events $E_{jt}^1$ and $E_{jt}^2$ defined in \eqref{eq:decomp} and $A_{it}^\kappa$ in \eqref{eq:kappa}, we have  \begin{eqnarray*}
       &&\sum_{t=1}^T \mathbb{P}(a(t)=j,E_{jt}^1,E_{jt}^2|G_{t-1}^\star,A_{it}^\kappa)  \leq \nu_{(2)}\left ((2\sqrt{\pi})/v\left( 16v^4/\kappa^3 + 0.5 c_{\eta}\sqrt{1+L^2 \nu_{(2)}} \right) + 63\right) + 4,
   \end{eqnarray*}
 where $\nu_{(2)}= \max(\nu_{(1)},64v^2/(\kappa^2 \lambda_m) \log T)$.
   \label{lem:term1}
\end{lem}

\begin{proof}
We rewrite $p_{ijt}$ to decompose the components of it into independent terms as follows:
\begin{eqnarray*}
p_{ijt} &=& \mathbb{P}(y_i(t)^\top \widetilde{\eta}_i(t)> 0.5( y_j(t)^\top \eta_j+y_i(t)^\top \eta_i)|A_{it}^\kappa,F_{t-1}^\star)\\
&=& \mathbb{P}(y_i(t)^\top (\widetilde{\eta}_i(t) - \widehat{\eta}_i(t) )> - y_i(t)^\top (\widehat{\eta}_i(t) - \eta_i) - 0.5( y_i(t)^\top \eta_i-y_j(t)^\top \eta_j) |A_{it}^\kappa,F_{t-1}^\star).
\end{eqnarray*}
Let $A_{it}^\eta = \{|y_i(t)^\top(\widehat{\eta}_i(t) - \eta_i)|<(1/4)( y_i(t)^\top \eta_i-y_j(t)^\top \eta_j)\}$, which represents an event where the estimator of transformed parameter is close to the parameter. Then, we have
\begin{flalign*}
    &\mathbb{P}\left(\left.y_i(t)^\top (\widetilde{\eta}_i(t) - \widehat{\eta}_i(t) )> - y_i(t)^\top (\widehat{\eta}_i(t) - \eta_i) -  \frac{1}{2}( y_i(t)^\top \eta_i-y_j(t)^\top \eta_j)\right|A_{it}^\kappa,F_{t-1}^\star,A_{it}^\eta\right)\mathbb{P}(A_{it}^\eta |A_{it}^\kappa,F_{t-1}^\star)\\
    &\geq \mathbb{P}(y_i(t)^\top (\widetilde{\eta}_i(t) - \widehat{\eta}_i(t) )> -  (1/4)( y_i(t)^\top \eta_i-y_j(t)^\top \eta_j) |A_{it}^\eta,A_{it}^\kappa,F_{t-1}^\star)\mathbb{P}(A_{it}^\eta |A_{it}^\kappa,F_{t-1}^\star)\\
    &\geq \left(1 - \exp\left(-\frac{((1/4)( y_i(t)^\top \eta_i-y_j(t)^\top \eta_j))^2}{2v^2\sigma_{it}^2}\right)\right)\mathbb{P}(A_{it}^\eta |A_{it}^\kappa,F_{t-1}^\star),
\end{flalign*}
using $y_i(t)^\top (\widetilde{\eta}_i(t) - \widehat{\eta}_i(t) ) \sim \mathcal{N}(0,v^2\sigma_{it}^2)$. By Lemma \ref{lem:eig}, if $n_i(t) > \nu_{(1)}$, we have

\begin{eqnarray}
\sigma_{it}^2/\|y(t)\|^2 = y_i(t)^\top B_i(t)^{-1} y_i(t)/\|y(t)\|^2 \leq \frac{2}{\lambda_m}n_i(t)^{-1}\label{eq:sigmaity}
\end{eqnarray}

In addition, if $n_i(t) > \nu_{(2)}:= \max(\nu_{(1)},64v^2/(\kappa^2 \lambda_m) \log T)= \mathcal{O}\left(\kappa^{-2}L^4\log(TN/\delta)\right)$, we have
\begin{eqnarray}
    \frac{((1/4)( y_i(t)^\top \eta_i-y_j(t)^\top \eta_j)/\|y(t)\|)^2}{2v^2\sigma_{it}^2/\|y(t)\|^2} \geq \frac{\lambda_m n_i(t) \kappa^2}{64v^2 } \geq \log T,\label{eq:T}
\end{eqnarray}
and thereby we have $\exp\left(-((1/4)( y_i(t)^\top \eta_i-y_j(t)^\top \eta_j))^2/(2v^2\sigma_{it}^2)\right) \leq T^{-1}$. Accordingly, if $n_i(t) > \nu_{(2)}$, we get

\begin{eqnarray*}
    \left(1 - \exp\left(-\frac{((1/4)( y_i(t)^\top \eta_i-y_j(t)^\top \eta_j))^2}{2v^2\sigma_{it}^2}\right)\right)\mathbb{P}(A_{it}^\eta)\geq \left(1 - \frac{1}{T}\right)\mathbb{P}(A_{it}^\eta |A_{it}^\kappa,F_{t-1}^\star).
\end{eqnarray*}

Thus, we get 

\begin{eqnarray*}
    \mathbb{E}\left[\left.\frac{1}{p_{ijt}}\right|G_{t-1}^\star,A_{it}^\kappa\right] - 1 \leq \frac{1}{\left(1 - \frac{1}{T}\right)\mathbb{P}(A_{it}^\eta |A_{it}^\kappa,F_{t-1}^\star)} - 1,
\end{eqnarray*}
for $n_i(t) > \nu_{(2)}$. Note that
\begin{eqnarray*}
    \mathbb{P}(A_{it}^\eta |A_{it}^\kappa,F_{t-1}^\star) &=& \mathbb{P}(|y_i(t)^\top(\widehat{\eta}_i(t) - \eta_i)|<(1/4)( y_i(t)^\top \eta_i-y_j(t)^\top \eta_j) |F_{t-1}^\star)\\
    &\leq& 1-\exp\left(-\frac{((1/4)\kappa)^2}{2v^2\sigma_{it}^2}\right).
\end{eqnarray*}
Since $\mathbb{P}(A_{it}^\eta |A_{it}^\kappa,F_{t-1}^\star) > 1-T^{-1}$ for $n_i(t) >  \nu_{(2)}$ by \eqref{eq:T}, we have

\begin{eqnarray}
     \mathbb{E}\left[\left.\frac{1-p_{ijt}}{p_{ijt}}\right|G_{t-1}^\star,A_{it}^\kappa\right] \leq \frac{1}{\left(1 - \frac{1}{T}\right)^2} -1 \leq \frac{4}{T}.\label{eq:pinv2}
\end{eqnarray}

Thus, putting \eqref{eq:phiinv2} and \eqref{eq:pinv2} together, we have

\begin{eqnarray*}
    &&\sum_{t=1}^T \mathbb{P}(a(t)=j,E_{jt}^1,E_{jt}^2|G_{t-1}^\star,A_{it}^\kappa)\\ 
    &\leq&  \sum_{t:n_i(t)\leq\nu_{(2)} } \mathbb{E}\left[\left.\frac{1-p_{ijt}}{p_{ijt}}\mathbb{I}(a(t)=i)\right|A_{it}^\kappa,F_{t-1}^\star\right] + \sum_{t:n_i(t)>\nu_{(2)}} \mathbb{E}\left[\left.\frac{1-p_{ijt}}{p_{ijt}}\mathbb{I}(a(t)=i)\right|A_{it}^\kappa,F_{t-1}^\star\right]\\
    &\leq&  \nu_{(2)}\left (\frac{2\sqrt{\pi}}{v}\left( \frac{16v^4}{\kappa^3} + \frac{1}{2} c_{\eta}\sqrt{1+L^2 \nu_{(2)}} \right) + 63\right) + 4.
\end{eqnarray*}
\end{proof}

We showed an upper bound for the first term in \eqref{eq:decomp}. Now, we aim to establish an upper bound for the second term in \eqref{eq:decomp}.

\begin{lem}
For the events $E_{jt}^1$ and $E_{jt}^2$ defined in \eqref{eq:decomp} and $A_{it}^\kappa$ in \eqref{eq:kappa}, for $t\in[T]$ with probability at least $1-\delta$, we have
$$\sum_{\tau=1}^t \mathbb{P}\left(\left.a(\tau)=j,(E_{j\tau}^1)^c,E_{j\tau}^2\right|G_{t-1}^\star,A_{it}^\kappa\right)\leq \nu_{(2)} + 2.$$
\label{lem:term2}
\end{lem}
\begin{proof}
To start, we decompose the summation of $\mathbb{P}(a(\tau)=j,(E_{j\tau}^1)^c,E_{j\tau}^2|G_{\tau-1}^\star,A_{i\tau}^\kappa)$ into two based on the sample size $n_j(t)$ as follows:

\begin{eqnarray}
&&\sum_{\tau=1}^t \mathbb{P}(a(\tau)=j,(E_{j\tau}^1)^c,E_{j\tau}^2|G_{\tau-1}^\star,A_{i\tau}^\kappa)\\
&=&\sum_{\tau=1}^t\mathbb{E}[\mathbb{I}(a(\tau)=j,n_j(\tau) <\nu_{(2)}, (E_{j\tau}^1)^c,E_{j\tau}^2) + \mathbb{I}(a(\tau)=j,n_j(\tau) \geq \nu_{(2)}, (E_{j\tau}^1)^c,E_{j\tau}^2)|G_{\tau-1}^\star,A_{i\tau}^\kappa
]\nonumber\\   
&\leq&\nu_{(2)} + \sum_{\tau:n_j(\tau)\geq\lceil\nu_{(2)}\rceil}^t\mathbb{E}[\mathbb{E}[\mathbb{I}(a(\tau)=j,n_j(\tau) \geq \nu_{(2)},(E_{j\tau}^1)^c,E_{j\tau}^2)|A_{i\tau}^\kappa,F_{\tau-1}^\star]|G_{\tau-1}^\star,A_{i\tau}^\kappa].\label{eq:12311}
\end{eqnarray}

Now, we investigate the case with $n_j(t) \geq \nu_{(2)}$. We consider $\mathbb{P}\left(\left.(E_{jt}^1)^c\right|n_j(t)\geq \nu_{(2)} ,A_{it}^\kappa,F_{t-1}^\star\right)$ to find an upper bound for the second term in \eqref{eq:12311}. To do so, we rewrite

\begin{eqnarray*}
&&\mathbb{P}\left(\left.(E_{jt}^1)^c\right|n_j(t)\geq \nu_{(2)} ,A_{it}^\kappa,F_{t-1}^\star\right) \\
& =&\mathbb{P}\left(\left.(E_{jt}^1)^c,(E_{jt}^2)^c\right|n_j(t)\geq \nu_{(2)},A_{it}^\kappa,F_{t-1}^\star\right)+\mathbb{P}\left(\left.(E_{jt}^1)^c,E_{jt}^2\right|n_j(t) \geq \nu_{(2)},A_{it}^\kappa,F_{t-1}^\star\right) \\
& \leq& \mathbb{P}\left(\left.(E_{jt}^2)^c\right|F_{t-1}^\star,n_j(t)\geq \nu_{(2)} ,A_{it}^\kappa\right)+\mathbb{P}\left(\left.(E_{jt}^1)^c, E_{jt}^2\right|F_{t-1}^\star,n_j(t)\geq \nu_{(2)},A_{it}^\kappa\right).
\end{eqnarray*}

By \eqref{eq:cdfhat}, \eqref{eq:sigmaity}, and \eqref{eq:T}, if $n_j(t) \geq \nu_{(2)}$, we have

\begin{eqnarray}
    &&\mathbb{P}\left( (\left.E_{jt}^2)^c,n_j(t) \geq \nu_{(2)}\right|F_{t-1}^\star,A_{it}^\kappa\right)\nonumber \\
    &\leq& \mathbb{P}(y_j(t)^\top \widehat{\eta}_j(t) > y_j(t)^\top \eta_j +(1/4)( y_i(t)^\top \eta_i-y_j(t)^\top \eta_j)|F_{t-1}^\star,n_j(t) \geq \nu_{(2)} ,A_{it}^\kappa)\nonumber\\
    &\leq& \exp\left(-\frac{n_j(t)\lambda_m\kappa^2}{64v^2}\right) \leq \frac{1}{T}.\label{eq:probe2c}
\end{eqnarray}

Similarly, we have
\begin{eqnarray}
    &&\mathbb{P}((E_{jt}^1)^c,E_{jt}^2 |F_{t-1}^\star,n_j(t)\geq \nu_{(2)},A_{it}^\kappa)\nonumber\\
    &=&\mathbb{P}(y_j(t)^\top \widetilde{\eta}_j(t) > 0.5( y_j(t)^\top \eta_j+y_i(t)^\top \eta_i),E_{jt}^2 |F_{t-1}^\star,n_j(t)\geq \nu_{(2)},A_{it}^\kappa)\nonumber\\ 
    &\leq&\mathbb{P}(y_j(t)^\top( \widetilde{\eta}_j(t)-\widehat{\eta}_j(t)) > (1/4)( y_i(t)^\top \eta_i-y_j(t)^\top \eta_j)|F_{t-1}^\star,n_j(t)\geq \nu_{(2)}, A_{it}^\kappa)\nonumber\\ 
    &\leq& \exp\left(-\frac{n_j(t)\lambda_m\kappa^2}{64v^2}\right)\leq \frac{1}{T}.\label{eq:probe1e2c}
\end{eqnarray}

Putting \eqref{eq:probe2c} and \eqref{eq:probe1e2c} together, we have
\begin{eqnarray}
&&\mathbb{P}\left(\left.(E_{jt}^1)^c\right|F_{t-1}^\star,n_j(t) \geq \nu_{(2)},A_{it}^\kappa\right)  \\
&\leq& \mathbb{P}\left((E_{jt}^2)^c|F_{t-1}^\star,n_j(t) \geq \nu_{(2)} ,A_{it}^\kappa\right)+\mathbb{P}\left((E_{jt}^1)^c, E_{jt}^2|F_{t-1}^\star,n_j(t) \geq \nu_{(2)},A_{it}^\kappa\right)\nonumber\\
&\leq& \frac{2}{T}.\label{eq:e1c}
\end{eqnarray}


Since the part of summation for $n_i(t) < \nu_{(2)}$ is bounded by $\nu_{(2)}$, it suffices to show a bound for the other part. Based on the fact that whether $E_{j\tau}^2$ is true determined by $F_{t-1}^\star$, we have

\begin{eqnarray*}
   \mathbb{E}[\mathbb{I}(a(\tau)=j,n_j(\tau) \geq \nu_{(2)},(E_{j\tau}^1)^c,E_{j\tau}^2)|A_{i\tau}^\kappa,F_{\tau-1}^\star] = \mathbb{I}(E_{j\tau}^2)\mathbb{P}(a(\tau) = j,(E_{j\tau}^1)^c|A_{i\tau}^\kappa,F_{\tau-1}^\star).
\end{eqnarray*}

Using the equation above, we have

\begin{eqnarray*}
&&\sum_{\tau:n_j(\tau)\geq\lceil\nu_{(2)}\rceil}^t\mathbb{E}[\mathbb{E}[\mathbb{I}(a(\tau)=j,n_j(\tau) \geq \nu_{(2)},(E_{j\tau}^1)^c,E_{j\tau}^2)|A_{i\tau}^\kappa,F_{\tau-1}^\star]|G_{\tau-1}^\star,A_{i\tau}^\kappa]\\
&=&\sum_{\tau:n_j(\tau)\geq\lceil\nu_{(2)}\rceil}^t\mathbb{E}[\mathbb{I}(E_{j\tau}^2)\mathbb{P}(a(\tau) = j,(E_{j\tau}^1)^c|A_{i\tau}^\kappa,F_{\tau-1}^\star)|G_{\tau-1}^\star,A_{i\tau}^\kappa].
\end{eqnarray*}

Because $\mathbb{I}(E_{j\tau}^2) = \mathbb{I}(E_{j\tau}^2,n_j(\tau)\geq \nu_{(2)})$ given $n_j(\tau)\geq \lceil\nu_{(2)}\rceil$ and $\mathbb{P}(a(\tau)=j,(E_{j\tau}^1)^c|A_{i\tau}^\kappa,F_{\tau-1}^\star) \leq \mathbb{P}((E_{j\tau}^1)^c|A_{i\tau}^\kappa,F_{\tau-1}^\star)$, we have

\begin{eqnarray*}
&&\sum_{\tau:n_j(\tau)\geq\lceil\nu_{(2)}\rceil}^t\mathbb{E}[\mathbb{I}(E_{j\tau}^2)\mathbb{P}(a(\tau) = j,(E_{j\tau}^1)^c|A_{i\tau}^\kappa,F_{\tau-1}^\star)|G_{\tau-1}^\star,A_{i\tau}^\kappa]\\
&\leq&\sum_{\tau:n_j(\tau)\geq\lceil\nu_{(2)}\rceil}^t\mathbb{E}[\mathbb{I}(E_{j\tau}^2,n_j(\tau) \geq \nu_{(2)})\mathbb{P}((E_{j\tau}^1)^c|A_{i\tau}^\kappa,F_{\tau-1}^\star)|G_{\tau-1}^\star,A_{i\tau}^\kappa].
\end{eqnarray*}

If $\tau \geq \nu_{(2)}$, by \eqref{eq:e1c}, we have $\mathbb{P}((E_{j\tau}^1)^c|A_{i\tau}^\kappa,F_{\tau-1}^\star) \leq 2/T$. Accordingly, we get

\begin{eqnarray*}
&&\sum_{\tau:n_j(\tau)\geq\lceil\nu_{(2)}\rceil}^t\mathbb{E}[\mathbb{I}(E_{j\tau}^2,n_j(\tau) \geq \nu_{(2)})\mathbb{P}((E_{j\tau}^1)^c|A_{i\tau}^\kappa,F_{\tau-1}^\star)|G_{\tau-1}^\star,A_{i\tau}^\kappa]\\
&\leq&\sum_{\tau:n_j(\tau)\geq\lceil\nu_{(2)}\rceil}^t\mathbb{E}\left[\left.\mathbb{I}(E_{j\tau}^2,n_j(\tau) \geq \nu_{(2)}) \left(\frac{2}{T}\right)\right|G_{\tau-1}^\star,A_{i\tau}^\kappa\right].
\end{eqnarray*}

Because $\mathbb{I}(E_{j\tau}^2,n_j(\tau) \geq \nu_{(2)}) \leq 1$, we have

\begin{eqnarray*}
&&\sum_{\tau:n_j(\tau)\geq\lceil\nu_{(2)}\rceil}^t\mathbb{E}\left[\left.\mathbb{I}(E_{j\tau}^2,n_j(\tau) \geq \nu_{(2)}) \left(\frac{2}{T}\right)\right|G_{\tau-1}^\star,A_{i\tau}^\kappa\right]\\
    &\leq&\frac{2}{T}\sum_{\tau:n_j(\tau)\geq\lceil\nu_{(2)}\rceil}^t 1\\
&\leq& 2.
\end{eqnarray*}

Therefore, we have
\begin{eqnarray*}
\sum_{\tau=1}^t \mathbb{P}(a(\tau)=j,(E_{j\tau}^1)^c,E_{j\tau}^2|G_{\tau-1}^\star,A_{i\tau}^\kappa) \leq \nu_{(2)} + 2.
\end{eqnarray*}
\end{proof}

Now, we show an upper bound for the sum of third term in \eqref{eq:decomp}.

\begin{lem}
For $t\in[T]$, with probability at least $1-\delta$, we have
\begin{eqnarray*}
\sum_{\tau=1}^t \mathbb{P}(a(\tau)=j,(E_{j\tau}^2)^c|G_{\tau-1}^\star,A_{i\tau}^\kappa) \leq \nu_{(2)} + 1.
\end{eqnarray*}
\label{lem:term3}
\end{lem}
\begin{proof}
We consider $\mathbb{E}[\mathbb{I}(a(t)=j,(E_{j\tau}^2)^c)|G_{t-1}^\star,A_{it}^\kappa]$.

\begin{eqnarray*}
 &&\sum_{t=1}^T \mathbb{E}[\mathbb{I}(a(t)=j,(E_{j\tau}^2)^c)|G_{t-1}^\star,A_{it}^\kappa] \\
 &=& \sum_{t=1}^T \mathbb{E}[\mathbb{I}(a(t)=j,(E_{j\tau}^2)^c,n_i(t) < \nu_{(2)}) + \mathbb{I}(a(t)=j,(E_{j\tau}^2)^c ,n_i(t)\geq \nu_{(2)})|G_{t-1}^\star,A_{it}^\kappa]\\
&\leq& \nu_{(2)} + \sum_{t=\lceil \nu_{(2)}\rceil}^T \mathbb{P}(a(\tau)=j,(E_{j\tau}^2)^c,n_i(t)\geq \nu_{(2)})|G_{\tau-1}^\star,A_{i\tau}^\kappa) 
\end{eqnarray*}

By \eqref{eq:probe2c}, we have

\begin{eqnarray*}
    \mathbb{P}(a(\tau)=j,(E_{j\tau}^2)^c,n_i(t)\geq \nu_{(2)})|G_{\tau-1}^\star,A_{i\tau}^\kappa)
    &\leq& \mathbb{P}((E_{j\tau}^2)^c,n_i(t)\geq \nu_{(2)})|G_{\tau-1}^\star,A_{i\tau}^\kappa) \\
    &\leq& \frac{1}{T}.
\end{eqnarray*}

Therefore, we have
\begin{eqnarray*}
\sum_{\tau=1}^t \mathbb{P}(a(\tau)=j,(E_{j\tau}^2)^c|G_{\tau-1}^\star,A_{i\tau}^\kappa) \leq \nu_{(2)} + 1.
\end{eqnarray*}
\end{proof}

Now, we are ready to show an upper bound for \eqref{eq:decomp} with Lemma \ref{lem:term1}, \ref{lem:term2}, and \ref{lem:term3}.

\begin{lem}
\begin{flalign*}
    \sum_{\tau=1}^t \mathbb{P}(a(\tau)=i|G_{t-1}^\star,A_{it}^\kappa)\mathbb{P}(A_{it}^\kappa)
    \geq~\frac{p_i}{2}\left(t - N\left(\nu_{(2)}\left (\frac{2\sqrt{\pi}}{v}\left( \frac{16v^4}{\kappa^3} + \frac{1}{2} c_{\eta}\sqrt{1+L^2 \nu_{(2)}} \right) + 65\right) + 7\right)\right).\nonumber
\end{flalign*}
\label{lem:piineq}
\end{lem}
\begin{proof}
    Note that $\mathbb{P}(a(t)=i|G_{t-1}^\star,A_{it}^\kappa) = 1-\sum_{j\neq i}\mathbb{P}(a(t)=j|G_{t-1}^\star,A_{it}^\kappa)$. To find an upper bound for \eqref{eq:decomp}, we showed upper bounds of the three terms in \eqref{eq:decomp} in Lemma \ref{lem:term1}, \ref{lem:term2}, and \ref{lem:term3}. Putting them together, we have
    \begin{eqnarray*}
        \sum_{t=1}^t\mathbb{P}(a(\tau)=j|G_{t-1}^\star,A_{it}^\kappa) \leq \nu_{(2)}\left (\frac{2\sqrt{\pi}}{v}\left( \frac{16v^4}{\kappa^3} + \frac{1}{2} c_{\eta}\sqrt{1+L^2 \nu_{(2)}} \right) + 63\right) + 4+ 2\nu_{(2)} + 3.
    \end{eqnarray*}

By summing the probabilities above over all arms except for $i$, we have

\begin{eqnarray*}
    \sum_{j\neq i} \mathbb{P}(a(\tau)=j|G_{t-1}^\star,A_{it}^\kappa)\leq N\left(\nu_{(2)}\left (\frac{2\sqrt{\pi}}{v}\left( \frac{16v^4}{\kappa^3} + \frac{1}{2} c_{\eta}\sqrt{1+L^2 \nu_{(2)}} \right) + 65\right) + 7 \right).
\end{eqnarray*}

Therefore, we get
\begin{eqnarray*}
&&\sum_{\tau=1}^t \mathbb{P}(a(\tau)=i|G_{t-1}^\star,A_{it}^\kappa)\mathbb{P}(A_{it}^\kappa) = \sum_{\tau=1}^t \left(1-\sum_{j\neq i} \mathbb{P}(a(\tau)=i|G_{t-1}^\star,A_{it}^\kappa)\right)\mathbb{P}(A_{it}^\kappa)\\ 
    &\geq&\frac{p_i}{2}\left(t - N\left(\nu_{(2)}\left (\frac{2\sqrt{\pi}}{v}\left( \frac{16v^4}{\kappa^3} + \frac{1}{2} c_{\eta}\sqrt{1+L^2 \nu_{(2)}} \right) + 65\right) + 7\right)\right).
\end{eqnarray*}
    
\end{proof}

\begin{lem}
   If $t > \tau_i^{(1)}$ for 
 given $i$,
\begin{eqnarray*}
    n_i(t) \geq \frac{p_i}{4}t.
\end{eqnarray*}
where the minimum time $\tau_i^{(1)}$ is $$\tau_i^{(1)} = \max(4 \ell(\delta,T,N,\kappa),(32/p_i^2 )\log \delta^{-1})$$ and
$$\ell(\delta,T,N,\kappa) = N\left(\nu_{(2)}\left (\frac{2\sqrt{\pi}}{v}\left( \frac{16v^4}{\kappa^3} + \frac{1}{2} c_{\eta}\sqrt{1+L^2 \nu_{(2)}} \right) + 65\right) + 7\right).$$
\label{lem:ni}
\end{lem}

\begin{proof}
Consider a martingale sequence $\mathbb{I}(a(t)=i,A_{it}^\kappa) - \mathbb{P}(a(t)=i,A_{it}^\kappa|G_{t-1}^\star)$ with respect to the filtration $\{G_{t-1}^\star\}_{t=1}^\infty$ defined in . By Azuma's inequality, with probability at least $1-\delta$
\begin{eqnarray*}
    \sum_{\tau=1}^t \mathbb{I}(a(t)=i,A_{it}^\kappa) \geq - \sqrt{2t\log \delta^{-1}} +  \sum_{\tau=1}^t \mathbb{P}(a(t)=i|G_{t-1}^\star,A_{it}^\kappa)\mathbb{P}(A_{it}^\kappa).
\end{eqnarray*}
By Lemma \ref{lem:piineq}, we have
\begin{eqnarray*}
    &&\sum_{\tau=1}^t \mathbb{I}(a(t)=i,A_{it}^\kappa) \\
    &&\geq - \sqrt{2t\log \delta^{-1}} +  \frac{p_i}{2}\left(t - N\left(\nu_{(2)}\left (\frac{2\sqrt{\pi}}{v}\left( \frac{16v^4}{\kappa^3} + \frac{1}{2} c_{\eta}\sqrt{1+L^2 \nu_{(2)}} \right) + 65\right) + 7\right)\right).
\end{eqnarray*}

For ease of presentation, let 

\begin{eqnarray*}
    \ell(\delta,T,N,\kappa) = N\left(\nu_{(2)}\left (\frac{2\sqrt{\pi}}{v}\left( \frac{16v^4}{\kappa^3} + \frac{1}{2} c_{\eta}\sqrt{1+L^2 \nu_{(2)}} \right) + 65\right) + 7\right).
\end{eqnarray*}

Because $\sqrt{2t\log \delta^{-1}} \leq (p_it)/8$ for $t \geq (128/p_i^2 )\log \delta^{-1}$ and 
    $(p_i/2)\left(t - \ell(\delta,T,N,\kappa)\right) \geq 3p_it/8$
for $t \geq 4 \ell(\delta,T,N,\kappa)$, we have

\begin{eqnarray*}
    - \sqrt{2t\log \delta^{-1}} +  \frac{p_i}{2}\left(t - \ell(\delta,T,N,\kappa)\right) \geq \frac{p_it}{4},
\end{eqnarray*}
if $t \geq \tau_i^{(1)} := \max(4 \ell(\delta,T,N,\kappa),(128/p_i^2 )\log \delta^{-1})=\mathcal{O}(p_i^{-2}N\kappa^{-5}L^7\log^{1.5}(TNd_y/\delta))$. Therefore, if $t \geq \tau_i^{(1)}$, we have
\begin{eqnarray}
    n_i(t) \geq \frac{p_i t}{4},\label{eq:ni}
\end{eqnarray}
with probability at least $1-\delta$.
\end{proof}

Now, we are ready to prove Theorem \ref{thm:eta2}. From Lemma \ref{lem:subg}, we have 
\begin{eqnarray}  
\| \widehat{\eta}_i(t) - \eta_{i} \|  \leq  R\sqrt{\frac{2}{\lambda_m}}\left( \sqrt{d_y \log \left(\frac{1+TL^2}{\delta}\right)}+c_\eta\right)n_i(t)^{-1/2},\label{eq:etahaterr2}
\end{eqnarray}
if $n_i(t) > \nu_{(1)}$. Since we have $n_i(t) \geq (p_i t)/4$ by \eqref{eq:ni} for $t>\tau_i^{(1)}$, we get
\begin{flalign}
\| \widehat{\eta}_i(t) - \eta_{i} \| \leq  R \sqrt{\frac{8}{\lambda_m p_i}} \left( \sqrt{d_y \log \left(\frac{1+TL^2}{\delta}\right)}+c_\eta\right)t^{-1/2}.\nonumber
\end{flalign}
if $n_i(t) > \nu_{(1)}$ and $t>\tau_i^{(1)}$. Thus, putting the two sample conditions together, if $t>\tau_i:=\max(4p_i^{-1}\nu_{(1)},\tau_i^{(1)})$, we have
\begin{flalign}
\| \widehat{\eta}_i(t) - \eta_{i} \| \leq  R \sqrt{\frac{8}{\lambda_m p_i}} \left( \sqrt{d_y \log \left(\frac{1+TL^2}{\delta}\right)}+c_\eta\right)t^{-1/2}.\nonumber
\end{flalign}

Thus, if $t > \tau_M:= \max_{i\in [N]}\tau_i = \mathcal{O}({p_{\min}^+}^{-2}N\kappa^{-5}L^7\log^{1.5}(TNd_y/\delta))$,  with probability at least $1-\delta$, we have the following estimation accuracy
\begin{flalign}
\| \widehat{\eta}_i(t) - \eta_{i} \| \leq  R \sqrt{\frac{8}{\lambda_m p_i}} \left( \sqrt{d_y \log \left(\frac{1+TL^2}{\delta}\right)}+c_\eta\right)t^{-1/2}.\nonumber
\end{flalign}

By Theorem \ref{thm:eta2}, Lemma \ref{lem:eta1}, and Lemma \ref{lem:ni}, if $t>\tau_M$, we have

\begin{flalign}
\|\widetilde{\eta}_i(t)-\eta_i\| \leq   g^{(1)}(\delta) t^{-1/2},\label{eq:g1}
\end{flalign}
where
\begin{eqnarray*}
    g^{(1)}(\delta) &=& \sqrt{\frac{8 }{p_{\min}^+ \lambda_m }} \left( v\sqrt{2d_y \log \frac{2TN}{\delta}} + R\sqrt{d_y \log \left(\frac{1+TL^2}{\delta}\right)}+c_\eta\right)\\    &=&\mathcal{O}\left({p_{\min}^+}^{-1/2}d_y^{1/2} \sqrt{ \log(TNd_y/\delta)}\right).
\end{eqnarray*}

\newpage
\section{Proof of Theorem \ref{thm:reg2}}
\label{sec:pthm:reg2}

\begin{proof}
Note that the regret can be written as
\begin{eqnarray*}
    \mathrm{Regret}(T)&=& \sum_{t=1}^T (y_{a^\star(t)}(t)^\top \eta_{a^\star(t)}(t) - y_{a(t)}(t)^\top \eta_{a(t)}(t) )\\
    &\leq& 2c_\eta L \tau_M + \sum_{t=\lceil \tau_M\rceil}^T (y_{a^\star(t)}(t)^\top \eta_{a^\star(t)}(t) - y_{a(t)}(t)^\top \eta_{a(t)}(t) )\mathbb{I}(a^\star(t)\neq a(t)),
\end{eqnarray*} 
because $\|y_i(t)\| \leq L$ for all $i\in[N]$ and $t\in [T]$ and $\|\eta_i\|\leq c_\eta$ for all $i\in[N]$. The order of first term is $\mathcal{O}({p_{\min}^+}^{-2}NL^8\kappa^{-5}\log^{1.5}(TNd_y/\delta))$. Now, we aim to show an upper bound for the second term. The second term can be written as
\begin{eqnarray*}
   &&\sum_{t=\lceil \tau_M\rceil}^T  (y_{a^\star(t)}(t)^\top \eta_{a^\star(t)}(t) - y_{a(t)}(t)^\top \eta_{a(t)}(t) )\mathbb{I}(a^\star(t)\neq a(t)) \\
    &\leq&\sum_{t=\lceil \tau_M\rceil}^T  (y_{a^\star(t)}(t)^\top (\eta_{a^\star(t)}(t) - \widetilde{\eta}_{a^\star(t)}(t) ) - y_{a(t)}(t)^\top (\eta_{a(t)}(t) - \widetilde{\eta}_{a^\star(t)}(t) ) )\mathbb{I}(a^\star(t)\neq a(t)),
\end{eqnarray*}
because $y_{a(t)}(t)^\top\widetilde{\eta}_{a(t)}-y_{a^\star(t)}(t)^\top \widetilde{\eta}_{a^\star(t)}(t) \geq 0$. Since $\|y_i(t)\| \leq L$ for all $t\in [T]$, we have 
\begin{eqnarray*}
&&\sum_{t=\lceil \tau_M\rceil}^T (y_{a^\star(t)}(t)^\top (\eta_{a^\star(t)}(t) - \widetilde{\eta}_{a^\star(t)}(t) ) - y_{a(t)}(t)^\top (\eta_{a(t)}(t) - \widetilde{\eta}_{a^\star(t)}(t) ) )\mathbb{I}(a^\star(t)\neq a(t))\\
&\leq& L \sum_{t=\lceil \tau_M\rceil}^T (\|\widetilde{\eta}_{a^\star(t)}(t)-\eta_{a^\star(t)}\| + \|\widetilde{\eta}_{a(t)}(t)-\eta_{a(t)}\|)\mathbb{I}(a^\star(t)\neq a(t)).
\end{eqnarray*}
By \eqref{eq:g1}, if $t>\tau_M$, we have

\begin{flalign}
\|\widetilde{\eta}_{a^\star(t)}(t)-\eta_{a^\star(t)}\|+ \|\widetilde{\eta}_{a(t)}(t)-\eta_{a(t)}\| \leq   g^{(1)}(\delta) t^{-1/2},\nonumber
\end{flalign}
where
\begin{eqnarray*}
    g^{(1)}(\delta) &=& 2\sqrt{\frac{8 }{p_{\min}^+ \lambda_m }} \left( v\sqrt{2d_y \log \frac{2TN}{\delta}} + R\sqrt{d_y \log \left(\frac{1+TL^2}{\delta}\right)}+c_\eta\right)\\    &=&\mathcal{O}\left(v{p_{\min}^+}^{-1/2}\sqrt{d_y \log(TNd_y/\delta)}\right).
\end{eqnarray*}
Accordingly, the regret can be written as
\begin{eqnarray*}
    \mathrm{Regret}(T) \leq  2c_\eta L \tau_M + Lg^{(1)}(\delta)\sum_{t=\lceil \tau_M\rceil}^T    t^{-1/2} \mathbb{I}(a^\star(t)\neq a(t)).
\end{eqnarray*}
Thus, it suffices to show an upper bound for $\sum_{t=\lceil \tau_M\rceil}^T    t^{-1/2} \mathbb{I}(a^\star(t)\neq a(t))$. To proceed, we apply Lemma \ref{lem:azu} to find a high probability upper-bound for the summation of the martingale difference sequence $\{t^{-0.5}\mathbb{I}(a^\star(t) \neq a(t))- t^{-0.5}\mathbb{P}(a^\star(t) \neq a(t)|G_{t-1}^\star)\}_{t=\lceil \tau_M\rceil}^T$ with respect to the filtration $\{G_{t-1}^\star\}_{t=1}^\infty$, and get the following inequality with probability at least $1-\delta$:
\begin{eqnarray}
\sum_{t=\lceil \tau_M\rceil}^T \frac{1}{\sqrt{t}} \mathbb{I}(a^\star(t) \neq a(t)) \leq \sqrt{4\log T \log \delta^{-1} }+ \sum_{t=\lceil \tau_M\rceil}^T\frac{1}{\sqrt{t}} \mathbb{P}(a^\star(t) \neq a(t))\label{eq:az}.
\end{eqnarray}

To find a bound for $\mathbb{P}(a^\star(t) \neq a(t))$, we consider $\mathbb{P}(y_j(t)^\top \widetilde{\eta}_j(t)-y_i(t)^\top\widetilde{\eta}_i(t))>0 |A_{it}^\star)$. Using the same logic as \eqref{eq:abc2}, we decompose the following probability as follows:
\begin{eqnarray}
    &&\mathbb{P}(y_j(t)^\top \widetilde{\eta}_j(t)-y_i(t)^\top\widetilde{\eta}_i(t))>0 |G_{t-1}^\star,A_{it}^\star)\nonumber\\ 
    &\leq&\mathbb{P}(y_i(t)^\top (\widetilde{\eta}_i(t)-\widehat{\eta}_i(t)) > - y_i(t)^\top(\widehat{\eta}_i(t)-\eta_i) +  0.5(y_i(t)^\top\eta_i-y_j(t)^\top\eta_j)  |G_{t-1}^\star,A_{it}^\star)\nonumber\\
    &+& \mathbb{P}(y_j(t)^\top (\widetilde{\eta}_j(t)-\widehat{\eta}_j(t)) > - y_j(t)^\top(\widehat{\eta}_j(t)-\eta_j) +  0.5(y_i(t)^\top\eta_i-y_j(t)^\top\eta_j)  |G_{t-1}^\star,A_{it}^\star).\label{eq:peq}
\end{eqnarray}

Similarly to \eqref{eq:prob1} and \eqref{eq:prob2}, we have
\begin{eqnarray*}
    &&\mathbb{P}(y_i(t)^\top (\widetilde{\eta}_i(t)-\widehat{\eta}_i(t)) > - y_i(t)^\top(\widehat{\eta}_i(t)-\eta_i) +  0.5(y_i(t)^\top \eta_i-y_j(t)^\top \eta_j) |G_{t-1}^\star,A_{it}^\star) \nonumber\\
    &\leq&C\sqrt{\frac{1}{n_i(t)}}\left(8v\sqrt{ \frac{\pi}{\lambda_m}} + 4h(\delta,T)\right),
\end{eqnarray*}
where $h$ is defined in \eqref{eq:h}. Since we have $n_i(t) \geq (p_i t)/4$ for $t > \tau_M$, we have 
\begin{eqnarray}
    &&\mathbb{P}(y_i(t)^\top (\widetilde{\eta}_i(t)-\widehat{\eta}_i(t)) > - y_i(t)^\top(\widehat{\eta}_i(t)-\eta_i) +  0.5(y_i(t)^\top \eta_i-y_j(t)^\top \eta_j) |G_{t-1}^\star,A_{it}^\star) \nonumber\\
    &\leq&C\sqrt{\frac{4}{p_i t}}\left(8v\sqrt{ \frac{\pi}{\lambda_m}} + 4h(\delta,T)\right),\label{eq:peqi}
\end{eqnarray}
where $h(\delta,T)$ is defined in \eqref{eq:h}. Similarly, we get
\begin{eqnarray}
    &&\mathbb{P}(y_j(t)^\top (\widetilde{\eta}_j(t)-\widehat{\eta}_j(t)) > - y_j(t)^\top(\widehat{\eta}_j(t)-\eta_j) +  0.5(y_i(t)^\top \eta_i-y_j(t)^\top \eta_j) |G_{t-1}^\star,A_{it}^\star) \nonumber \\
    &\leq& C\sqrt{\frac{4}{p_j t}} \left(8v\sqrt{ \frac{\pi}{\lambda_m }} + 4h(\delta,T)\right),\label{eq:peqj}
\end{eqnarray}
if $t > \tau_M$. Accordingly, based on \eqref{eq:peq}, \eqref{eq:peqi}, and \eqref{eq:peqj}, we obtain the following bounds for the probabilities 
\begin{eqnarray*}
    &&\mathbb{P}(y_j(t)^\top \widetilde{\eta}_j(t)-y_i(t)^\top\widetilde{\eta}_i(t))>0 |G_{t-1}^\star,A_{it}^\star) \\
    &\leq& \frac{2C}{\sqrt{p^+_{\min}}} \left(v\left(8\sqrt{ \frac{\pi}{\lambda_m}}+8 \sqrt{ \frac{\pi}{\lambda_m }}\right) + 4h(\delta,T)+4h(\delta,T)\right)t^{-1/2}.
\end{eqnarray*}
By simple calculations, we get
\begin{eqnarray*}
    \mathbb{P}(y_j(t)^\top \widetilde{\eta}_j(t)-y_i(t)^\top\widetilde{\eta}_i(t) >0 |G_{t-1}^\star,A_{it}^\star)     \leq \frac{16C}{\sqrt{p^+_{\min}}} \left(2v\sqrt{ \frac{\pi}{\lambda_m}} + h(\delta,T)\right)t^{-1/2}.
\end{eqnarray*}
By summing the above probability up over $i, j \in [N]$, if $t>\tau_M$, we get an upper bound for the probability of choosing a sub-optimal arm at time $t$  
\begin{eqnarray*}
    &&\mathbb{P}(a^\star(t) \neq a(t)|G_{t-1}^\star)\\ 
    &=& \sum_{i=1}^N \sum_{j\neq i} \mathbb{P}(a(t)=j|G_{t-1}^\star,A_{it}^\star)\mathbb{P}(A_{it}^\star)  \leq \sum_{i=1}^N \sum_{j\neq i}\mathbb{P}(y_j(t)^\top \widetilde{\eta}_j(t)-y_i(t)^\top\widetilde{\eta}_i(t)>0 |G_{t-1}^\star,A_{it}^\star)\mathbb{P}(A_{it}^\star) \\
    &\leq&N\left(\frac{16C}{\sqrt{p^+_{\min}}} \left(2v\sqrt{ \frac{\pi}{\lambda_m}} + h(\delta,T)\right)t^{-1/2}\right).
\end{eqnarray*}

By plugging the inequality above to \eqref{eq:az}, with probability at least $1-\delta$, we have

\begin{eqnarray*}
\sum_{t=\lceil \tau_M\rceil}^T  \frac{1}{\sqrt{t}} \mathbb{I}(a^\star(t) \neq a(t)) &\leq& \sqrt{4\log T \log \delta^{-1} }+ \sum_{t=\lceil\tau_M\rceil}^T N\left(\frac{16C}{\sqrt{p^+_{\min}}} \left(2v\sqrt{ \frac{\pi}{\lambda_m}} + h(\delta,T)\right)t^{-1}\right)\\
&\leq&\sqrt{4\log T \log \delta^{-1} } + \frac{16CN}{\sqrt{p^+_{\min}}} c_M(\delta,T) \log T,
\end{eqnarray*}
where $c_M(\delta,T)$ is defined in \eqref{eq:cm}. 
Therefore, putting together $L=\mathcal{O}(\sqrt{d_y\log(TNd_y/\delta)})$, $g^{(1)}(\delta)=\mathcal{O}\left(v\sqrt{ {p_{\min}^+}^{-1}d_y \log(TNd_y/\delta)}\right)$, $c_M(\delta,T) = \mathcal{O}(\sqrt{d_y \log (TNd_y/\delta)})$, and $\tau_M =\mathcal{O}({p_{\min}^+}^{-2}NL^7\kappa^{-5}\log^{1.5}(TNd_y/\delta))$,
\begin{eqnarray*}
    \mathrm{Regret}(T) 
    &\leq&  2c_\eta L \tau_M  + L   g^{(1)}(\delta) \left(\sqrt{4\log T\log \delta^{-1}}  +  \frac{16c_M(\delta,T)CN}{\sqrt{p^+_{\min} }} \log T\right)\\
    &=&\mathcal{O} \left((p^+_{\min})^{-2} L^8\log^{1.5} \left(\frac{TNd_y}{\delta  }\right) + vLNd_y {p_{\min}^+}^{-0.5}\log^2(TNd_y/\delta) \right)\\
    &=& \mathcal{O}\left(\frac{vNd_y^{4} }{(p^+_{\min})^2\kappa^5  } \log^{5.5}\left(\frac{TNd_y}{\delta  }\right)\right).
\end{eqnarray*}

This regret bound is inflated by $d_y^3\log^{3.5}\left(TNd_y/\delta\right)$ due to the order of maximum magnitude of observation norm $L$. If the support of observations is bounded by a positive constant so that $L$ is a positive constant unrelated to other factors ($N,~d_y,~T$, and $\delta$), the upper bound can be reduced to $\mathcal{O}\left(Nd_y(p^+_{\min})^{-2}\kappa^{-5} \log^2\left(TNd_y/\delta\right)\right)$. 

\end{proof}

\end{document}